%% file: small_random.tex
\documentclass[11pt]{article}

\usepackage[utf8]{inputenc} 
\usepackage[T1]{fontenc}    
\usepackage{hyperref}      
\usepackage{url}            
\usepackage{booktabs}      
\usepackage{amsfonts}      
\usepackage{nicefrac}       
\usepackage{microtype}    
\usepackage{xcolor} 
\usepackage{hyperref}     
\definecolor{darkred}{RGB}{150,0,0}
\definecolor{darkgreen}{RGB}{0,150,0}
\definecolor{darkblue}{RGB}{0,0,200}
\hypersetup{colorlinks=true, linkcolor=darkred, citecolor=darkgreen, urlcolor=darkblue}

\usepackage{microtype}
\usepackage{graphicx}
\usepackage{booktabs} 
\usepackage{tikz}

\usepackage{hyperref,graphicx,fullpage,amsmath,amsfonts,subcaption,amssymb,bm,url,breakurl,epsfig,epsf,color,MnSymbol,mathbbol,fmtcount,semtrans,caption,multirow,comment}

\usepackage{color}
\usepackage[utf8]{inputenc}
\usepackage[T1]{fontenc}    
\usepackage{hyperref}    
\usepackage{url}            
\usepackage{booktabs}       
\usepackage{amsfonts}       
\usepackage{nicefrac}       
\usepackage{microtype}

\usepackage{amsthm}

\newtheorem{theorem}{Theorem}[section]
\newtheorem{lemma}[theorem]{Lemma} 
\newtheorem{remark}[theorem]{Remark}
\newtheorem{definition}[theorem]{Definition}

\usepackage{etoc}

\input{notation2.tex}

\title{Small random initialization is akin to spectral learning: Optimization and generalization guarantees for overparameterized low-rank matrix reconstruction}

\begin{document}

\title{Small random initialization is akin to spectral learning: Optimization and generalization guarantees\\ for overparameterized low-rank matrix reconstruction}

\author{Dominik St\"oger and Mahdi Soltanolkotabi\\
	Ming Hsieh Department of Electrical and Computer Engineering\\
	University of Southern California}

\maketitle

\input{abstract.tex}


\input{intro.tex}

\input{setup.tex}

\input{main_results.tex}

\input{related_work.tex}

\input{proof_ideas.tex}

\input{simulations.tex}

\input{preliminaries.tex}

\input{spectral_phase_analysis.tex}

\input{convergence_analysis.tex}

\input{main_results_proof.tex}

\input{conclusion.tex}

\bibliography{Bibfiles,Bibfiles2,literature}
\bibliographystyle{unsrt} 

\newpage

\appendix

\input{spectral_phase_appendix.tex}

\input{convergence_phase_proofs.tex}

\end{document}

%% file: notation2.tex
\newcommand{\trace}{\text{Tr}\,}
\newcommand{\R}{\mathbb{R}}

\newcommand{\Id}{\text{Id}}

\newcommand{\innerproduct}[1]{ \langle #1 \rangle }

\newcommand{\twonorm}[1]{\left\|#1\right\|_{\ell_2}}
\newcommand{\fronorm}[1]{\Vert  #1 \Vert_F}
\newcommand{\specnorm}[1]{\Vert  #1 \Vert}
\newcommand{\nucnorm}[1]{\Vert  #1 \Vert_{\ast}}
\newcommand{\Vnorm}[1]{{\left\vert\kern-0.25ex\left\vert\kern-0.25ex\left\vert #1 
		\right\vert\kern-0.25ex\right\vert\kern-0.25ex\right\vert}}

\newcommand{\Lt}{L_t}
\newcommand{\Ut}{U_t}
\newcommand{\Uttilde}{\widetilde{U_t}}
\newcommand{\Utplus}{U_{t+1}}

\newcommand{\Et}{E_{t}}

\newcommand{\bracing}[2]{\underset{{#1}}{\underbrace{#2}}  }

\newcommand{\Lttilde}{L_t}
\newcommand{\Nttilde}{N_t}

\newcommand{\Zt}{Z_t}

\newcommand{\UUt}{U_t}
\newcommand{\UUtplus}{U_{t+1}}

\newcommand{\WWt}{W_t}
\newcommand{\Wtperp}{W_{t,\perp}}

\newcommand{\WWtplus}{W_{t+1}}
\newcommand{\Wtplusperp}{W_{t+1,\perp}}

\newcommand{\overeq}[1]{\overset{#1}{=}}
\newcommand{\overleq}[1]{\overset{#1}{\le}}
\newcommand{\overgeq}[1]{\overset{#1}{\ge}}

%% file: abstract.tex
\begin{abstract}
Recently there has been significant theoretical progress on understanding the convergence and generalization of gradient-based methods on nonconvex losses with overparameterized models. Nevertheless, many aspects of optimization and generalization and in particular the critical role of small random initialization are not fully understood. In this paper, we take a step towards demystifying this role by proving that small random initialization followed by a few iterations of gradient descent behaves akin to popular spectral methods. We also show that this \emph{implicit spectral bias} from small random initialization, which is provably more prominent for overparameterized models, also puts the gradient descent iterations on a particular trajectory towards solutions that are not only globally optimal but also generalize well. Concretely, we focus on the problem of reconstructing a low-rank matrix from a few measurements via a natural nonconvex formulation. In this setting, we show that the trajectory of the gradient descent iterations from small random initialization can be approximately decomposed into three phases: (I) a \emph{spectral or alignment phase} where we show that that the iterates have an implicit spectral bias akin to spectral initialization allowing us to show that at the end of this phase the column space of the iterates and the underlying low-rank matrix are sufficiently aligned, (II) a \emph{saddle avoidance/refinement phase} where we show that the trajectory of the gradient iterates moves away from certain degenerate saddle points, and (III) a \emph{local refinement phase} where we show that after avoiding the saddles the iterates converge quickly to the underlying low-rank matrix.  Underlying our analysis are insights for the analysis of overparameterized nonconvex optimization schemes that may have implications for computational problems beyond low-rank reconstruction.

\end{abstract}

%% file: intro.tex
\section{Introduction}\label{section:introduction}
Many contemporary problems in machine learning and signal estimation spanning deep learning to low-rank matrix reconstruction involve fitting nonlinear models to training data. Despite tremendous empirical progress, theoretical understanding of these problems poses two fundamental challenges. First, from an \emph{optimization} perspective, fitting these models often requires solving highly nonconvex optimization problems and except for a few special cases, it is not known how to provably find globally or approximately optimal solutions. Yet simple heuristics such as running (stochastic) gradient descent from (typically) small random initialization is surprisingly effective at finding globally optimal solutions. A second \emph{generalization} challenge is that many modern learning models including neural network architectures are trained in an overparameterized regime where the parameters of the model exceed the size of the training dataset. It is well understood that in this overparameterized regime, these large models are highly expressive and have the capacity to (over)fit arbitrary training datasets including pure noise. Mysteriously however overparameterized models trained via simple algorithms such as (stochastic) gradient descent when initialized at random  continue to predict well or generalize on yet unseen test data. 
In particular, it has been noted in a number of works that for many modern machine learning architectures, the scale of initialization is important for the generalization/test behavior \cite{kernelrichregimes,ghorbani2020neural}. 
It has been noted that stronger generalization performance is typically observed for a smaller scale initialization.
Indeed, small random initialization followed by (stochastic) gradient descent iterative updates is arguably the most widely used learning algorithm in modern machine learning and signal estimation. 

There has been a large number of exciting results aimed at demystifying both the optimization and generalization aspects over the past few years. We will elaborate on these results in detail in Section \ref{section:relatedwork}, however, we would like to briefly mention the common techniques and their existing limitations. On the optimization front a large body of work has emerged on providing guarantees for nonconvex optimization which can roughly be put into two categories: (I) smart initialization+local convergence and (II) landscape analysis+saddle escaping algorithms. Approaches in (I) focus on showing local convergence of local search techniques from carefully designed spectral initializations \cite{wirtinger_flow,truncated_wirtinger_flow,chen_implicit_regularization,tu2015low,ling_blind_deconv,ling_demixing,netrapalli_alternating,waldspurger_alternating}. Approaches in (II) focus on showing that in some cases the optimization landscape is benign in the sense that all local minima are global (no spurious local minima) and the saddle points have a direction of strict negative curvature (strict saddle) \cite{sun2015nonconvex}. Then specialized truncation or saddle escaping algorithms such as trust region, cubic regularization \cite{nesterov2006cubic, nocedal2006trust}, or noisy (stochastic) gradient-based methods \cite{jin2017escape, ge2015escaping, raginsky2017non, zhang2017hitting} are deployed to provably find a global optimum. Both approaches fail to fully explain the typical behavior of local search techniques in practice. Indeed, for many nonconvex problems local search techniques or simple variants,  when initialized at random, quickly converge to globally optimal solutions without getting stuck in local optima/saddles without the need for sophisticated initialization or saddle escaping heuristics. We note that while for differentiable losses eventual convergence to local minimizers is known from a random initialization \cite{lee2016gradient} on problems of the form (II), these results cannot rule out exponentially slow cases in the worst-case \cite{du2017gradient}. Indeed, it has been argued that in general a more granular analysis of the trajectory of gradient descent beyond the landscape may be necessary \cite{arora2019implicit}. For example, some recent advances has been made by analysing the trajectory of gradient descent using a leave-one-out analysis for the phase retrieval problem \cite{chen_global_convergence}.

Similarly, there has been a lot of exciting progress on the generalization front, especially for neural networks.  Specific to generalization capabilities of gradient-based approaches these results broadly fall into
two categories: (a) the first category is based on a linearization
principle which characterizes the performance of nonlinear models such as neural networks by
comparing it to a linearized kernel problem around the initialization
(a.k.a. Neural Tangent Kernels) \cite{jacot2018neural,oymak2019overparameterized,oymak2020towards,du2019gradient,arora2019fine,buchanan2021deep,chen2021how}. This has often been referred to as ''lazy training".
(b) the second category is based
on a continuous limit analysis in the limit of width going to infinity
and learning rate going to zero (mean-field analysis) \cite{mei2018mean,chizat_meanfield,mei2019mean,javanmard2020analysis,sirignano}. However, these existing analyses
contain many idealized and non-realistic assumptions (e.g. requiring large, random initialization in (a), which typically leads to worse generalization than what is observed in practice, or unrealistically large widths in (b)) and therefore cannot fully explain the success of overparameterized models or serve as a guiding principle for practitioners \cite{chizat2019lazy}.

Despite the aforementioned exciting recent theoretical progress many aspects of optimization and generalization and in particular the role of random initialization remains mysterious. This leads us to the main
challenge of this paper

\begin{quote}
\emph{Why is \textbf{small random initialization} combined with gradient descent updates so effective at finding globally optimal models that generalize well despite the nonconvex nature of the optimization landscape or model overparameterization?}
\end{quote}
\begin{figure}[t]
	\centering
	\begin{subfigure}{0.55\textwidth}
		\includegraphics[scale=0.20]{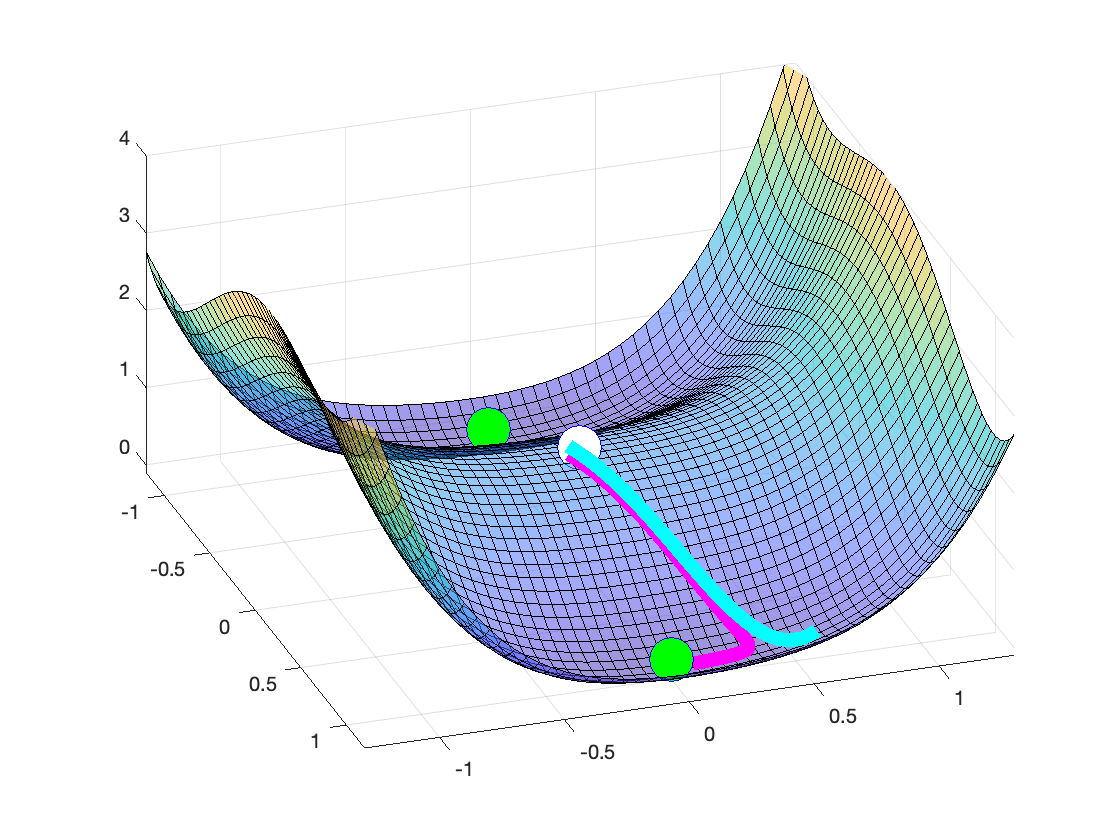}
		\label{fig:fig_introduction1}
	\end{subfigure}%
	\begin{subfigure}{0.6\textwidth}
		\includegraphics[scale=0.6]{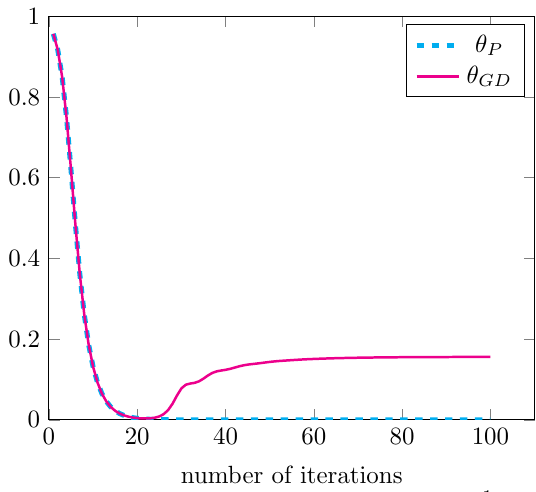}
		\label{fig:fig_introduction2}
	\end{subfigure}
	\caption{\textbf{Gradient descent from small random initialization is akin to spectral initialization.} The left figure depicts the empirical landscape of a low-rank matrix reconstruction problem with the two green circles depicting the two global minima and the white circle the saddle point at the origin. In this figure, we also depict the trajectory of the gradient descent iterations (magenta) together with the power method based on a popular spectral initialization technique (blue). Both gradient descent and power method  use the same small initialization near the origin. We see that in the early stage, the two trajectories are almost the same. The figure on the right depicts the angle between the gradient descent (magenta)/power method (blue) iterates and a popular spectral initialization technique, denoted by $\theta_{GD}$ and $\theta_P$ respectively. This figure clearly demonstrates that for the first iterations these angles are practically the same further confirming that the initial trajectory of gradient descent and power methods are similar. See Section \ref{section:numerical_experiments} for further detail on the experimental setup. (In this figure we have used $r=r_{\star}=1$.)}
	\label{fig:introduction}
\end{figure}
In this paper we wish to take a step towards addressing the above challenge by demystifying the critical role of small random initialization in gradient-based approaches. Specifically we show that
\begin{quote}
\emph{\textbf{Small random initialization} followed by a few iterations of gradient descent behaves akin to spectral initialization.}
\end{quote}

By that, we mean more precisely, that if the initialization is chosen small enough, then in the initial stage of the training, gradient descent \textit{implicitly} behaves like spectral initialization techniques such as those commonly used in techniques based on the method of moments. This \emph{implicit spectral bias} of gradient descent from random initialization puts the gradient descent iterations on a particular trajectory towards solutions that are not only globally optimal but also generalize well for overparameterized models. We also show that with small random initialization this implicit spectral bias phenomenon is more prominent for more overparameterized models in the sense that it materializes after fewer iterations. This intriguing phenomenon is depicted in Figure \ref{fig:introduction} in the context of a low-rank reconstruction problem. This figure clearly demonstrates that the first few iterations of gradient descent starting from a small random initialization are virtually identical to that of running power iterations (a popular algorithm to find the spectral initialization, see, e.g. \cite{golub_loan}).

Concretely we focus on the problem of low-rank matrix recovery, which appears in many different application areas such as recommendation systems, phase retrieval, and quantum tomography \cite{davenport2016overview}. Here, our goal is to recover a low-rank matrix of the form $XX^T$ from a few linear measurements. We consider a natural, non-convex approach based on matrix factorization, where we minimize the loss function via gradient descent. In this paper, we show that, regardless of the amount of overparameterization used, for small random initialization vanilla gradient descent will always converge towards the low-rank solution. This holds as long as the measurement operator obeys a popular restricted isometry property \cite{candesplan}.

Our analysis consists of three phases. The first phase is the aforementioned \emph{spectral} or \emph{alignment} phase where we show gradient descent from small random initialization behaves akin to spectral initialization, which is a key insight of this paper. Indeed, we show that the first few gradient descent iterates can be accurately approximated by power method iterates.  Next, we show that after this first \textit{spectral or alignment phase}, gradient descent enters a second phase, which we refer to as \textit{saddle avoidance phase}.
In this phase, we show that the trajectory of the gradient iterates moves away from degenerate saddle points, while the iterates maintain almost the same effective rank as $XX^T$.
In the third phase, the \emph{local refinement phase}, we show that the iterates approximately converge towards the underlying low-rank matrix $XX^T$ with a geometric rate up to a certain error floor which depends on the initialization scale. In particular, by decreasing the scale of initialization this error threshold can be made arbitrarily small. While in this paper our main focus is on low-rank matrix reconstruction, we believe that our analysis holds more generally for a variety of contemporary machine learning and signal estimation tasks including neural networks. 

Finally we note that while a similar setting has already been studied in \cite{LiMaCOLT18}, our analysis goes beyond it in many important ways. For example, our result holds for any amount of overparameterization and allows for arbitrarily small initialization. Maybe most importantly, we study the spectral phase phenomenon at initialization. For a detailed comparison we refer to Section \ref{section:mainresults}.

%% file: setup.tex
\section{Low-rank matrix recovery via non-convex optimization}\label{section:lowrankmatrixrecovery}

As mentioned earlier in this paper we focus on reconstructing a (possibly overparameterized) Positive Semidefinite (PSD) low rank matrix from a few measurements. In this problem, given $m$ observations of the form
\begin{equation}\label{equ:ydef}
y_i = \innerproduct{A_i, XX^T}= \trace (A_i XX^T) \quad \quad \text{ } i=1,\ldots, m,
\end{equation}
we wish to reconstruct the unknown matrix $XX^T$. Here, $X\in \mathbb{R}^{n \times r_{\star}}$ with $1 \le r_{\star} \le n$ is a factor of the unknown matrix and $\left\{ A_i \right\}_{i=1}^m $ are known symmetric measurement matrices. A common approach to solving this problem is via minimizing the loss function
\begin{equation*}
\underset{\bar{U} \in \mathbb{R}^{n \times r}}{\min} \  f(\bar{U}) :=  \underset{\bar{U} \in \mathbb{R}^{n \times r}}{\min}  \ \frac{1}{4m} \sum_{i=1}^{m} \left( y_i - \innerproduct{A_i, \bar{U}\bar{U}^T}  \right)^2,
\end{equation*} 
with $ r \ge r_{\star}$. More compactly one can rewrite the optimization problem above in the form
\begin{align}
\label{mainopt}
\underset{\bar{U} \in \mathbb{R}^{n \times r}}{\min} \  f(\bar{U}) :=  \underset{\bar{U} \in \mathbb{R}^{n \times r}}{\min} \frac{1}{4} \twonorm{\mathcal{A}\left(\bar{U}\bar{U}^T-XX^T\right)}^2,
\end{align} 
where $ \mathcal{ A}:  \R^{n\times n} \longrightarrow \mathbb{R}^m $ is the measurement operator defined by $[\mathcal{ A}\left(Z\right)]_i:=\frac{1}{\sqrt{m}} \innerproduct{A_i, Z} $.

In order to solve the minimization problem \eqref{mainopt} we run gradient descent iterations starting from (often small) random initialization. More specifically, 
\begin{align*}
U_{t+1} &= U_t - \mu \nabla f \left(U_t\right)=U_t+ \mu     \mathcal{ A}^* \left[     y - \mathcal{ A} \left( U_t U_t^T \right)    \right] U_t \\
 &=U_t+ \mu  \left[ \left(   \mathcal{ A}^*  \mathcal{ A} \right) \left( XX^T - U_t U_t^T \right)   \right] U_t  .
\end{align*}
where we have set $U_0=\alpha U$ is the initialization matrix, $\mathcal{A}^*$ denotes the adjoint operator of $\mathcal{A}$ and $y= \left(y_i\right)_{i=1}^m  \in \mathbb{R}^m$ denotes the measurement vector. Here, $U \in \mathbb{R}^{n \times r}$ is a typically random matrix which represents the form of the initialization and $\alpha>0$ is a scaling parameter. 

There are two challenges associated with analyzing such randomly initialized gradient descent updates. The first is an \emph{optimization} challenge. Since $f$ is \textit{non-convex} it is a priori not clear whether gradient descent converges to a global optimum or whether it gets stuck in a local minima and/or saddle. The second challenge is that of \emph{generalization}. This is particularly pronounced in the overparameterized scenario where the number of parameters are larger than the number of data points i.e.~$ rn \ge m$. In this case, there are infinitely many $\bar{U}$ such that $ f(\bar{U}) =0$, but $\Vert  \bar{U} \bar{U}^T -XX^T \Vert_F $ is arbitrarily large (see, e.g., \cite[Proposition 1]{slawski2015regularization}). That is, even if gradient descent converges to a global optimum, i.e. $f\left( \bar{U} \right)=0$, it is a priori not clear whether it has found the low-rank solution $XX^T $ (see also Figure \ref{figure:trajectory}).

%% file: main_results.tex
\section{Main results}\label{section:mainresults}

In this section, we present our main results. Stating these results requires a couple of simple definitions. The first definition concerns the measurement operator $\mathcal{ A}$.
\begin{definition}[Restricted Isometry Property (RIP)]\label{def:restrictedisometry}
	The measurement operator $\mathcal{ A}: \R^{n\times n} \longrightarrow \mathbb{R}^m$ satisfies RIP of rank $r$ with constant $\delta>0$, if it holds for all matrices $Z$ of rank at most $r$
	\begin{equation}\label{equation:RIP}
	\left(1-\delta \right) \fronorm{Z}^2 \le  \twonorm{ \mathcal{ A} \left(Z\right)  }^2  \le \left(1+\delta \right) \fronorm{Z}^2.
	\end{equation}
\end{definition}
We note that for a Gaussian measurement operator $\mathcal{ A}$ 
\footnote{By that, we mean that all the entries of the (symmetric) measurement matrices $\left\{ A_i \right\}_{i=1}^m$ are drawn i.i.d. with distribution $\mathcal{ N} \left(0,1\right) $ on the off-diagonal and distribution $ \mathcal{ N} \left(0,1/\sqrt{2}\right)  $ on the diagonal.},
RIP of rank $r$ and constant $\delta >0$ holds with high probability, if the number of observations satisfies $ m \gtrsim  nr/\delta^{2} $ \cite{candesplan, vershynin2018high}.

The second definition concerns the condition number of the factor $X$.
\begin{definition}[condition number] We denote the condition number of 
	$X \in \mathbb{R}^{n \times r_{\star}}$ by  $$
	\kappa:= \frac{\specnorm{X}}{\sigma_{r_{\star} } \left(X\right)},
	$$
	where $\sigma_{r_{\star} } \left(X\right)$ denotes $r_{\star}$-th largest singular value of $X$.
\end{definition}
With these definitions in place we are now ready to state our main results.

\subsection{General case: $r>r_{\star}$}
We begin by stating our first main result.
\begin{theorem}\label{mainresult:runequaln} 
	Let $X\in\R^{n\times r_*}$ and assume we have $m$ measurements of the low rank matrix $XX^T$ of the form $y = \mathcal{ A}\left(XX^T\right)$ with $\mathcal{ A}$ the measurement operator. We assume $\mathcal{ A}$ satisfies the restricted isometry property for all matrices of rank at most $2r_{\star}+1$ with constant $\delta \le  c \kappa^{-4} {r_{\star}}^{-1/2}    $. To reconstruct $XX^T$ from the measurements we fit a model of the form $\bar{U}\mapsto\mathcal{ A}\left(\bar{U}\bar{U}^T\right)$ with $\bar{U}\in\R^{n\times r}$ and $r>r_*$ via running gradient descent iterations of the form $U_{t+1} = U_t - \mu \nabla f \left(U_t\right)$ on the objective \eqref{mainopt} with a step size obeying $ \mu \le c \kappa^{-4} \specnorm{X}^{-2}  $. Here, the initialization is given by $U_0 =  \alpha U$, where $U \in \mathbb{R}^{n \times r}$ has  i.i.d. entries with distribution $ \mathcal{ N} \left(0, 1/\sqrt{r} \right) $.  With this setting and assumptions the following two statements hold.
	\begin{enumerate}
		\item Under the assumption that $r \ge 2 r_{\star} $ and that  the scale of initialization fulfills 
		\begin{equation}\label{ineq:alphabound}
		\alpha \lesssim  \min \left\{ \frac{  \left( \min \left\{ r;n  \right\} \right)^{1/4}     }{      \kappa^{1/2} n^{3/4}}  \left(     2 \kappa^2  \sqrt{\frac{n}{\min \left\{ r;n  \right\}}}   \right)^{-6\kappa^2} ; 		  \frac{  1 }{\kappa^7 n }\right\}   \specnorm{X},
		\end{equation}
		after 
		\begin{align}\label{ineq:itbound_added}
		\hat{t} \lesssim \frac{1}{\mu \sigma_{ \min }  \left(X\right)^2}  \ln \left(  \frac{ C_1 n\kappa }{\min \left\{ r;n  \right\}}  \cdot  \max \left\{   1; \frac{\kappa r_{\star}}{\min \left\{ r;n  \right\}-r_{\star}} \right\}   \cdot   \frac{ \Vert X \Vert}{\alpha} \right)
		\end{align}
		iterations we have that
		\begin{equation}\label{ineq:generalizationerror1}
		\frac{\fronorm{ U_{\hat{t}} U_{\hat{t}}^{T} - XX^T  }   }{\specnorm{X}^2} \lesssim      \frac{  n^{21/16}  \kappa^{81/16}    r_{\star}^{1/8}  }{   \left( \min \left\{ r;n \right\} \right)^{15/16}  }   \cdot   \frac{   \alpha^{21/16}  }{  \specnorm{X}^{21/16}},
		\end{equation}
		holds with probability at least $ 1 - C e^{-\tilde{c} r }$.
		\item 	Assume that $r_{\star} < r < 2r_{\star}$ and that the scale of initialization fulfills 
		\begin{equation}\label{ineq:alphabound2}
		\alpha \lesssim  \min \left\{ \frac{   \varepsilon^{1/2}   }{    n^{3/4}  \kappa^{1/2} }   \left(     \frac{  2\kappa^2\sqrt{rn} }{\varepsilon}   \right)^{-6\kappa^2}   ; \frac{ \varepsilon  }{ n   \kappa^7  }     \right\}  \specnorm{X},
		\end{equation}
		with $ 0< \varepsilon <1 $. Then, after 
		\begin{align*}
		\hat{t} \lesssim	 \frac{1}{\mu \sigma_{ \min }  \left(X\right)^2}  \ln \left(     \frac{ C_2 \kappa  n^2  }{  \varepsilon^2  \left(   r-r_{\star}  \right)   } \cdot  \frac{\specnorm{X}}{\alpha}    \right)
		\end{align*}
		iterations we have that
		\begin{equation}\label{ineq:generalizationerror2}
		\frac{\fronorm{ U_{\hat{t}} U_{\hat{t}}^{T}  -XX^T  } }{\specnorm{X}^2} \lesssim    r_{\star}^{1/8}     \left(  r-r_{\star}  \right)^{3/8} \kappa^{81/16}   \left( \frac{n}{\varepsilon }    \cdot \frac{ \alpha }{\specnorm{X}}   \right)^{21/16}
		\end{equation}
		holds with probability at least $1-  \left( \tilde{C} \varepsilon \right)^{r-r_{\star}+1} + \exp \left(-\tilde{c}r\right)$.
	\end{enumerate}	
	Here, $C_1, C_2, C, \tilde{C}, c>0$ are fixed numerical constants.
\end{theorem}

Note that the test error $ \fronorm{U_{\hat{t}} U_{\hat{t}}^{T}  -XX^T  }^2 $ can be made arbitrarily small by choosing the scale of initialization $\alpha$ small enough. In particular, the dependence of the test error on $\alpha$ is polynomial and the dependence of the  number of iterations on $\alpha$ is logarithmic, which means  that reducing the test error by scaling down $\alpha$  introduces only modest additional computational cost. Hence, as long as the rank at most $2r_{\star}+1$ RIP with constant $\delta \le  c \kappa^{-4} {r_{\star}}^{-1/2} $ holds, gradient descent converges to a point in the proximity of the low-rank solution, whenever the initialization is chosen small enough regardless of the choice of $r$. This holds even when the model is overparameterized i.e.~$ rn \gg m  $ and the optimization problem has many global optima many of which do not obey $UU^T\approx XX^T$. This result thus further demonstrates that when initialized with a small random initialization gradient descent has an implicit bias towards solutions of low-rank or small nuclear norm. This is in sharp contrast to Neural Tangent Kernel (NTK)-based theory for  low-rank matrix recovery (see \cite[Section 4.2]{oymak2019overparameterized}) which will not approximately recover the ground truth matrix $XX^T$ due to the larger scale of initialization required when using that technique.

As discussed in Section \ref{section:lowrankmatrixrecovery}, the restricted isometry property holds with high probability for a sample complexity $ m \gtrsim nr_{\star}^2 \kappa^{8}$ for Gaussian measurement matrices. Up to constants, this sample complexity 
is optimal in $n$, while it is sub-optimal in $r_{\star}$ and $\kappa$ compared to approaches based on nuclear-norm minimization (see, e.g., \cite{candesplan}). While there is numerical evidence that the true scaling of $m$ in $r_{\star}$ should also be linear in the non-convex case \cite{bhojanapalli2016global}, we note that the optimal dependence of the sample complexity on $r_{\star}$ is a major open problem in the field, as the sample complexities in all theoretical results for non-convex approaches in the literature scale at least quadratically in $r_{\star}$.

\textbf{Interpretation: } 
Recall from Section \ref{section:introduction} that our convergence analysis can be divided into three phases: the spectral phase, the saddle avoidance phase, and the local refinement phase. As it will become clear from the proofs, when $r \ge 2r_{\star} $ the bound on the needed number of iterations (see inequality \eqref{ineq:itbound_added}) can be decomposed as follows 
{\small
	\begin{equation}\label{ineq:iterationsdecomposition}
	\begin{split}
	\hat{t} \lesssim & \frac{1}{\mu \sigma_{ \min }  \left(X\right)^2}  \Bigg\lbrack   \bracing{\text{Phase I: spectral/alignment phase}}{ \ln \left(     2 \kappa^2  \sqrt{\frac{n}{\min \left\{ r;n \right\}}}    \right)} + \bracing{\text{Phase II: saddle avoidance phase}}{  \ln \left(\frac{ \sigma_{ \min } \left(X\right) }{\alpha} \right)} + \bracing{\text{Phase III:  local refinement phase}}{    \ln \left( \max \left\{   1; \frac{\kappa r_{\star}}{\min \left\{ r;n \right\} -r_{\star}} \right\}   \frac{ \Vert X \Vert}{\alpha} \right)   } \Bigg\rbrack  .
	\end{split}
	\end{equation}
}
First, we note that the duration of all three phases scales inversely with $ \sigma_{ \min  } \left(X\right)^2 $. This is due to the fact that in all three phases the dynamics associated the smallest singular value of $X$ is the slowest one and hence needs the most time to complete.

In the spectral phase,  the eigenvectors corresponding  to the  leading $r_{\star}$ eigenvalues of $\UUt \UUt^T$ become aligned with the eigenvectors corresponding  to the  leading $r_{\star}$ eigenvalues of $\mathcal{ A}^* \mathcal{ A} \left(XX^T\right)$. We observe in \eqref{ineq:iterationsdecomposition}  that in the spectral phase increasing $r$, i.e. the amount of parameters, decreases the number of iterations in this phase. As we will explain in Section \ref{sec:keyideaspec}, the reason is that increasing $r$ decreases the angle between the column space of the initialization $U_0$ and the span of the eigenvectors corresponding to the leading $r_{\star}$ eigenvalues of $\mathcal{ A}^* \mathcal{ A} \left(XX^T\right)$ used in spectral initialization. As a consequence, gradient descent needs fewer iterations to align these two subspaces. 

In the saddle avoidance phase (Phase II), $\sigma_{r_{\star} }\left( \UUt \right) $ , the $r_{\star}$th largest singular value of $\UUt$, grows geometrically until it is on the order of $ \sigma_{ \min  } \left(X\right) $. Hence, this duration depends on the ratio between the $\sigma_{ \min  } \left(X\right) $ and the the scale of initialization $\alpha$. This is clearly reflected in the upper bound on the number of needed iterations in equation \eqref{ineq:iterationsdecomposition}.

In Phase III, the local refinement phase, the matrix $ \UUt \UUt^T $ converges towards $ XX^T $. In particular, at iteration $\hat{t}$ the test error obeys \eqref{ineq:generalizationerror1}. We observe that a smaller $\alpha$ allows for a smaller test error in \eqref{ineq:generalizationerror1} but per \eqref{ineq:iterationsdecomposition} this higher accuracy is achieved with a modest increase in the required iterations.

\subsection{Special case: $r=r_{\star}$}
The following result deals with the scenario $r =r_{\star} $, that is, the iterates $  \Ut $ have as many parameters as the ground truth matrix $X$.
\begin{theorem}\label{mainresult:requalrstar}
	Let $X\in\R^{n\times r_*}$ and assume we have $m$ measurements of the low rank matrix $XX^T$ of the form $y = \mathcal{ A}\left(XX^T\right)$ with $\mathcal{ A}$ the measurement operator. We assume $\mathcal{ A}$ satisfies the restricted isometry property for all matrices of rank at most $2r_{\star}+1$ with constant $\delta \le c \kappa^{-4} {r_{\star}}^{-1/2}    $. To reconstruct $XX^T$ from the measurements we fit a model of the form $\bar{U}\mapsto\mathcal{ A}\left(\bar{U}\bar{U}^T\right)$ with $\bar{U}\in\R^{n\times r_{\star}}$ via running gradient descent iterations of the form $U_{t+1} = U_t - \mu \nabla f \left(U_t\right)$ on the objective \eqref{mainopt} with a step size obeying $ \mu \le c \kappa^{-4} \specnorm{X}^{-2}  $. Here, the initialization is given by $U_0 =  \alpha U$, where $U \in \mathbb{R}^{n \times r_{\star}}$ has  i.i.d. entries with distribution $ \mathcal{ N} \left(0, 1/\sqrt{r_{\star}} \right) $. 
	Assume that   the scale of initialization fulfills 
	\begin{equation*}
	\alpha \lesssim  \min \left\{ \frac{   \varepsilon^{1/2}   }{    n^{3/4}  \kappa^{1/2} }   \left(     \frac{  2\kappa^2\sqrt{ rn} }{\varepsilon}   \right)^{-6\kappa^2}   ; \frac{ \varepsilon^2   }{ n \sqrt{r_{\star}}  \kappa^7  }     \right\}  \specnorm{X},
	\end{equation*}
	for some $0 < \varepsilon <1$. Then with probability at least $1- C\varepsilon  +\exp \left(-cr_{\star}\right)$
	after 
	\begin{align*}
	\hat{t} \lesssim  \frac{1}{\mu \sigma_{ \min }  \left(X\right)^2}  \ln \left(      \frac{ 8 \kappa^3  n^3  }{  \varepsilon^2    } \cdot  \frac{\specnorm{X}}{\alpha}    \right)
	\end{align*}
	iterations we have that
	\begin{equation}\label{equ:requalsnproximity}
	\frac{\fronorm{ U_{\hat{t}} U_{\hat{t}}^{T}  -XX^T }  }{\specnorm{X}^2} \lesssim   \  r_{\star}^{1/8} \kappa^{81/16}  \left( \frac{n}{\varepsilon }    \cdot \frac{ \alpha }{\specnorm{X}}   \right)^{21/16} .
	\end{equation}
	Here $C,c>0$ are fixed numerical constants.
\end{theorem}
Note that by choosing $\alpha$ small enough we can make the test error in \eqref{equ:requalsnproximity} arbitrarily small. In particular, this means that then well-known local convergence results can be applied showing that $U_t U_t^T$ converges linearly to $XX^T$ (see, e.g., \cite{tu2015low}) .\\
Thus, this result implies that if the measurement operator fulfills the restricted isometry property, gradient descent with \textit{small, random initialization} will converge to the ground truth matrix $X$ in polynomial time.
It is known that under the RIP assumption the loss landscape is benign \cite{landscape_benign} in the sense that there are no local optima that are not global and all saddles have a direction of negative curvature. However, such results do not imply that vanilla gradient descent converges quickly (i.e. in polynomial time) to a global optimum, as gradient descent may take exponential time to escape from saddle points. 

To the best of our knowledge, this is the first in the non-overparameterized setting $r=r_{\star}$ result which shows the convergence of vanilla gradient descent to the ground truth from a random initialization using only the restricted isometry property in polynomial time.
The only other paperin the low-rank matrix recovery literature, which shows fast convergence of vanilla gradient descent to the ground truth from a random initialization, is \cite{chen_global_convergence}. In this work, the problem of phase retrieval has been studied, which can be formulated as a low-rank matrix recovery problem with $r=r_{\star}=1$. The paper shows that gradient descent converges from a random initialization to the ground truth with a near-optimal number of iterations. However, the proof in this paper leverages the rotation-invariance of the Gaussian measurements vectors via carefully constructed auxiliary sequences. In contrast, Theorem \ref{mainresult:requalrstar}  above relies only on the restricted isometry property and no further assumptions on $\mathcal{ A}$ are needed.

\subsection{Special case: $r=n$ with orthonormal initialization}\label{sec:comparisonthm}
In the following result, we study the scenario $r=n$, where the initialization matrix $U \in \mathbb{R}^{n \times n}$ is an orthonormal matrix, i.e. $U^T U =\Id $, instead of a Gaussian matrix as in the previous results in this paper. This is the same setting as in \cite[Theorem 1.1]{LiMaCOLT18}, and we include this special case so as to explain how our results improves upon prior work in this special case.
\begin{theorem}\label{mainresult:requalsn} 
		Let $X\in\R^{n\times r_{\star}}$ and assume we have $m$ measurements of the low rank matrix $XX^T$ of the form $y = \mathcal{ A}\left(XX^T\right)$ with $\mathcal{ A}$ the measurement operator. We assume $\mathcal{ A}$ satisfies the restricted isometry property for all matrices of rank at most $2r_{\star}+1$ with constant $\delta \le  c \kappa^{-4} {r_{\star}}^{-1/2}    $. To reconstruct $XX^T$ from the measurements we fit a model of the form $\bar{U}\mapsto\mathcal{ A}\left(\bar{U}\bar{U}^T\right)$ with $\bar{U}\in\R^{n\times n}$ via running gradient descent iterations of the form $U_{t+1} = U_t - \mu \nabla f \left(U_t\right)$ on the objective \eqref{mainopt} with a step size obeying $ \mu \le c \kappa^{-4} \specnorm{X}^{-2}  $. Here, the initialization is given by $U_0 =  \alpha U$, where $U \in \mathbb{R}^{n \times n}$ can be any orthonormal matrix. Assume that the scale of initialization satisfies $ \alpha \le  c \frac{\sigma_{\min }\left(X\right)}{\kappa^2 n} $. 
	Then, after 
	\begin{equation*}
	\hat{t} \lesssim   \frac{1}{\mu \sigma_{\min } \left(X\right)^2 }\ln \left( \max \left\{   1; \frac{\kappa r_{\star}}{n-r_{\star}} \right\}   \frac{\Vert X \Vert}{\alpha} \right) 
	\end{equation*}
	iterations we have that
	\begin{equation*}\label{ineq:generalizationerror3}
	\frac{\fronorm{ U_{\hat{t}} U_{\hat{t}}^{T}  -XX^T   } }{\specnorm{X}^2} \lesssim  \frac{ r_{\star}^{1/8}     n^{3/8} }{\kappa^{3/16}  }  \cdot   \frac{   \alpha^{21/16}  }{  \specnorm{X}^{21/16}}.
	\end{equation*}
Here $c>0$ is a fixed numerical constant.
\end{theorem}
Note that this result improves over \cite[Theorem 4.1]{LiMaCOLT18} in several aspects. First of all, in  \cite{LiMaCOLT18} it is assumed that the measurement operator $\mathcal{ A}$ has the  rank-$4r$ restricted isometry property with constant $\delta \lesssim \kappa^{-6} r_{\star}^{-1/2} \log^{-2} \frac{n}{\alpha}$. In particular, this suggests that this result cannot handle the scenario that the scale of initialization $\alpha$ becomes arbitrarily small, as this would also require that the restricted isometry constant $\delta$ becomes arbitrarily small as well. This in turn would require an arbitarily large sample size.  Moreover,  \cite{LiMaCOLT18} requires a step size of at most  $\mu \lesssim\kappa^{-6} r_{\star}^{-1/2} \log^2 n  \specnorm{X}^{-2} $, whereas the above theorem only needs the weaker assumption $  \mu \lesssim \kappa^{-2} \specnorm{X}^{-2}  $. These improvements aside the main difference between our result and this prior work is that we can handle any $r$ by formalizing an intriguing connection between small random initialization and spectral learning.

%% file: related_work.tex
\section{Related work}\label{section:relatedwork}

\textbf{Global convergence guarantees for nonconvex low-rank matrix recovery}: As mentioned earlier in Section \ref{section:introduction}, there is a large body of work on developing global convergence guarantees for nonconvex problems. In the context of low-rank matrix recovery, several papers have demonstrated that low-rank reconstruction problems in a variety of domains can be solved via nonconvex gradient descent starting from spectral initialization. More precisely, this has been shown for phase retrieval \cite{wirtinger_flow,truncated_wirtinger_flow,chen_implicit_regularization}, matrix sensing \cite{tu2016low}, blind deconvolution \cite{ling_blind_deconv,ling_demixing}, and matrix completion \cite{xiaodong_matrixcompletion}. However, in practice often random initialization is used in lieu of specialized spectral initialization techniques. To remedy this issue, more recent literature \cite{landscape_phaseretrieval,ge2016matrix_spurious,zhang2019sharp}, focusses on studying the loss landscape of such problems. These papers show that despite their non-convexity under certain assumptions these loss landscapes are benign in the sense that there are no \textit{spurious local minima}, (i.e.~all minimizers are global minima) and saddles points have a strict direction of negative curvature (a.k.a.~strict saddle) \cite{sun2015nonconvex}. Then specialized truncation or saddle escaping algorithms such as trust region, cubic regularization \cite{nesterov2006cubic, nocedal2006trust} or noisy (stochastic) gradient-based methods \cite{jin2017escape, ge2015escaping, raginsky2017non, zhang2017hitting} are deployed to provably find a global optimum. These papers however do not directly develop global convergence for gradient descent (without any additional modification) from a random initialization. For differentiable losses eventual convergence to local minimizers is known from random initialization \cite{lee2016gradient} but these results do not provide convergence rates and only guarantee eventual convergence. Indeed, gradient descent may converge exponentially slowly in the worst-case \cite{du2017gradient}. In contrast to the above literature our result in Theorem \ref{mainresult:requalrstar} (in the case of $r=r_{\star}$) shows that gradient descent from a small random initialization converges rather quickly to the global optima. As mentioned  earlier, we are able to establish this result by demonstrating that in the initial phase gradient descent iterates are intimately connected to the spectral initialization techniques discussed above. Furthermore, the above spectral initialization followed by local convergence or landscape analysis techniques cannot be directly applied in in the overparameterized case ($r>r_{\star}$) whereas our analysis works regardless of model overparameterization. 

We would like to mention that even more recently the paper \cite{chen_global_convergence} proves the convergence of gradient descent starting from a random initialization for low-rank recovery problems via an interesting leave-one-out analysis. To the best of our knowledge, this is the only existing result, which provides convergence guarantees for vanilla gradient descent from random initialization for low-rank matrix recovery problems in the non-overparameterized setting $r=r_{\star}$. However, the leave-one-out analysis heavily relies on the independence and the rotation invariance of the measurements. Also similar to the above this analysis does not seem to easily lend to generalization in the overparameterized regime. In contrast, our proof techniques rely on standard restricted isometry assumptions without requiring the independence of the measurements and does provide generalization guarantees with model overparameterization ($r>r_{\star}$).
Moreover, in \cite{hou2020fast} it has been shown that Riemannian gradient descent converges with nearly linear rate to the true solution from a random initialization in the population loss scenario.\\

\noindent\textbf{Overparameterization in low-rank matrix recovery}: In the influential work  \cite{gunasekar2017implicit} it has been conjectured and in the special case that the measurement matrices commute proven that gradient descent on overparameterized matrix factorization converges to the solution with the minimal nuclear norm. This phenomenon  is now often referred to as \textit{implicit regularization}. In \cite{arora2019implicit}, evidence is provided that adding depth even increases the tendency of gradient descent to converge to low-rank solution. In  \cite{razin2020implicit} it has been shown that there are certain scenarios where the conjecture in \cite{gunasekar2017implicit} does not hold.  In \cite{li2021towards} theoretical and empirical evidence has been provided that gradient flow with infinitesimal initialization is equivalent to a certain rank-minimization heuristic.

In this paper, we shed further light on the implicit regularization of gradient descent. In particular, we provide a precise analysis of the initial stage and relate it to the power method and our analysis explains how overparameterization is beneficial in the initial stages. Closest to our work is the paper \cite{LiMaCOLT18}, which studies a special case of  the problem analysed in this paper. More precisely, this paper considers the special case $r=n$ with orthonormal initialization. We also applied our theory to this exact same setting, see Section  \ref{sec:comparisonthm}, where we include a detailed comparison for this special scenario. Most importantly our theory is able to handel the case $\alpha \rightarrow 0$, which the result in \cite{LiMaCOLT18} seems not to be able to. Moreover, analysing the full range of possible choise of $r$ requires a careful analysis of the spectral phase, which is one key novely of this paper compared to \cite{LiMaCOLT18}.
 
In \cite{kelly_kyrillidis,kabanava_kueng} it has been shown that in certain scenarios, where the measurement matrices $A_i$ are positive semidefinite (PSD), the equation $y=\mathcal{ A} \left(UU^T\right)$ has a unique low-rank solution. This means that in these scenarios the PSD constraint by itself might lead to a low-rank matrix recovery, which makes implicit regularization by gradient descent meaningless in this setting. However, note that these results not apply to the scenario studied in this paper, as we assume the measurement matrices $A_i$ to be Gaussian, which, in particular, means that they are not positive semidefinite. In particular, in our setting it can be shown that there are infinitely many solutions to the equation $y= \mathcal{ A} \left(UU^T\right) $ with arbitrarily large test error \cite{slawski2015regularization}.\\

\noindent\textbf{Gradient-based generalization guarantees for overparameterized tensors and neural networks:} A recent line of work is concerned with connecting the analysis of neural network training with the so-called neural tangent kernel (NTK) \cite{jacot2018neural,oymak2019overparameterized,oymak2020towards,du2019gradient,arora2019fine}. The key idea is that for a large enough initialization, it suffices to consider a linearization of the neural network around the origin. This allows connecting the analysis of neural networks with the well-studied theory of kernel methods. This is also sometimes referred to as \textit{lazy training}, as with such an initialization the parameters of the neural networks stay close to the parameters at initialization.
However, there is a line of work, which suggests that NTK-analysis might not be sufficient to completely explain the success of neural networks in practice. The paper \cite{chizat2019lazy} provides  empirical evidence that by choosing a smaller initialization the test error of the neural network decreases. A similar performance gap between the performance of the  NTK and neural networks has been observed in \cite{ghorbani2020neural}, where it has been shown that the performance gap is larger if the covariance matrix is isotropic. 

There is also a line of work \cite{mei2018mean,chizat_meanfield,mei2019mean,javanmard2020analysis,sirignano}, which is concerned with the mean-field analysis of neural networks. The insight is that for sufficiently large width the training dynamics of the neural network can be coupled with the evolution of a probability distribution described by a PDE. These papers use a smaller initialization than in the NTK-regime and, hence, the parameters can move away from the initialization. However, these results do not provide explicit convergence rates and require an unrealistically large width of the neural network.

For the problem of  tensor decomposition it has also been shown that gradient descent with small initialization is able to leverage low-rank structure \cite{wang2020beyondNTK}. This is relevant to neural network analysis, since in \cite{ge2018learning_tensor} a relationship between tensor decomposition and training neural networks has been established.  In \cite{li2020learningbeyondNTK}  it has been shown that neural networks with ReLU function and trained by SGD can outperform any kernel method. One crucial element in their analysis is that the early stage of the training is connected with learning the first and second moment of the data.

While in this paper we do not study overparameterized tensor or neural network models we note that the NTK-theory can also be applied to low-rank matrix recovery (see \cite[Section 4.2]{oymak2019overparameterized}). This means that if the scale of initialization is chosen large enough and the number of parameters is larger than the number of measurements, i.e. $nr \gtrsim m$, then gradient descent will converge linearly to a global minimizer with zero loss. However, since for this approach the parameters will stay close to the initialization, this approach will not recover the ground truth matrix $XX^T$.  Hence, an NTK analysis will not yield good generalization. In contrast in this paper we have seen that choosing a small initialization is a remedy for low-rank matrix recovery. So in this sense our result can be viewed as going beyond the lazy training in NTK theory. In fact we believe that similar analysis to the one developed in this paper for low-rank recovery can be used to analyze a much broader class of overparameterized models including the analysis of neural networks. We defer this to a future paper.
\\

\noindent\textbf{Linear neural networks:} In \cite{bartlett2018gradient,arora2018optimization,arora2018convergence,bah2019learning,chou2020gradient}  the convergence of gradient flow and gradient descent is studied for (deep) linear neural networks of the form
\begin{equation*}
\underset{W_1, W_2, \ldots, W_N}{\min} \  \sum_{i=1}^m \big\Vert W_N \ldots W_2 W_1 x_i - y_i   \big\Vert^2.
\end{equation*}
However, note that this model is different from the one studied in this paper. In \cite{gunasekar2018implicit} it is shown that gradient descent for convolutional linear neural networks has a bias towards the $\ell_p$-norm, where $p$ depends on the depth of the network.

%% file: proof_ideas.tex
\section{Overview and key ideas of the proof}\label{section:key_ideas}
In this section, we briefly discuss the key ideas and techniques in our proof. We begin by discussing a simple decomposition, which is utilized throughout our proofs. Next, in Sections \ref{sec:keyideaspec} and  \ref{keypflocal} we show that the trajectory of the gradient descent iterations can be approximately decomposed into three phases: (I) a \emph{spectral or alignment phase} where we show that gradient descent from random initialization behaves akin to spectral initialization allowing us to show that at the end of this phase the column spaces of the iterates $\UUt$ and the ground truth matrix $X$ are sufficiently aligned, (II) a \emph{saddle avoidance phase}, where we show that the trajectory of the gradient iterates move away from certain degenerate saddle points , and (III) a \textit{refinement phase}, where the product of the gradient descent iterates $\UUt\UUt^T$ converges quickly to the underlying low-rank matrix $XX^T$. The latter result holds up to a small error that is commensurate with the scale of the initialization and tends to zero as the scale of the initialization goes to zero. Figure \ref{proof_sketch_figure} depicts these three phases.

Let us remark that the proof in the related work \cite{LiMaCOLT18} decomposes the convergence analysis into two phases, which roughly correspond to Phase II and Phase III in our proof. However, the proof details are quite different since we use a different decomposition into signal and noise term, see Section \ref{section:noiseterm}.

\begin{figure}
	\centering
	\includegraphics[scale=0.8]{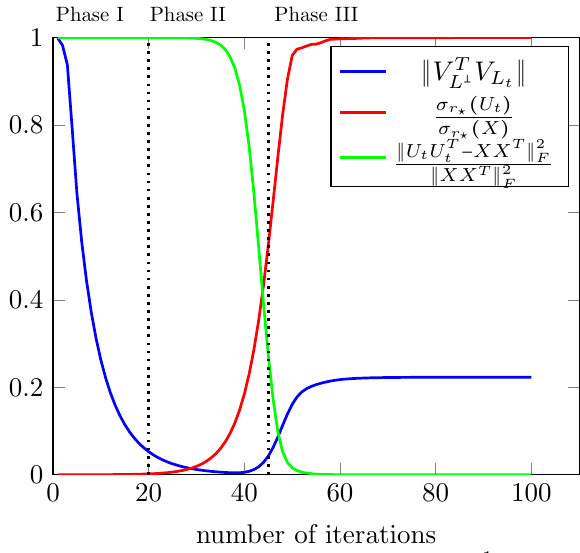}
	\caption{\textbf{Depiction of the three phases of convergence.} Let $L$ denote the subspace spanned by the eigenvectors corresponding to the $r_{\star}$ largest eigenvalues of the matrix $\mathcal{ A}^* \mathcal{ A}\left(XX^T\right)$ and $L_t$ denote the subspace spanned by the eigenvectors corresponding to the $r_{\star}$ largest eigenvalues of the matrix $ \UUt \UUt^T$. This figure demonstrates that the convergence analysis can be divided into three phases:  (I) spectral/alignment phase; (II) saddle avoidance phase and (III) the refinement phase.  We see that in the first phase the first $r_{\star}$ eigenvectors of $\Ut \Ut^T$ rapidly learn the subspace corresponding to the first $r_{\star}$ eigenvectors of  $\mathcal{ A}^* \mathcal{ A}\left(XX^T\right)$, i.e. the angle $ \specnorm{ V_{L^\perp}^T V_{L_t}  } $ becomes small. The $r_{\star}$th largest singular value of $\Ut$ is still small in this phase and the (normalized) test error $ \specnorm{\UUt \UUt - XX^T}^2/\specnorm{XX^T}^2 $  has not decreased yet. In Phase (II), however,  we see that $\sigma_{r_{\star}} \left(\Ut\right) $ is growing, whereas the loss begins to decrease in this phase and the subspaces stay aligned. In Phase (III) we see that the test error is converging towards $0$ rapidly, meaning that $\UUt \UUt^T$ converges to $XX^T$. Consequently, $ \sigma_{r_{\star}}  \left( \UUt\right) /  \sigma_{r_{\star}}  \left( X\right) $ converges to $1$ (red curve). We also see that in this phase the angle $ \specnorm{ V_{L^\perp}^T V_{L_t}  } $ grows again, until it reaches a certain threshold. This is because in this phase the top $r_{\star}$ eigenvalues of $\Ut \Ut^T$ become aligned with the eigenvectors of $XX^T$.}
	\label{proof_sketch_figure}
\end{figure}

\subsection{Decomposition of $U_t$  into ``signal'' and ``noise'' matrices}\label{section:noiseterm}
 A key idea in our proof is to decompose the matrix $U_t$ into the sum of two matrices. The first matrix, which is of rank $r_{\star}$, can be thought of as the ``signal''' term. We will show that the product of this matrix with its transpose converges towards the ground truth low-rank matrix $XX^T$. The second matrix, will have rank at most $r-r_{\star}$ and will have column span orthogonal to the column span of  the ground truth matrix $X$. We will show that the spectral norm of this matrix will remain relatively small depending on the scale of initialization $\alpha$. Hence, this term can be interpreted as the ``noise'' term.\\
We now formally introduce our decomposition. To this aim, consider the matrix $V_X^T U_t \in \mathbb{R}^{r_{\star} \times r}$ and denote its singular value decomposition by $V_X^T U_t =  V_t \Sigma_t \WWt^T $ with $\WWt \in \mathbb{R}^{r \times r_{\star}}$. Similarly, we shall use $W_{t,\perp} \in \mathbb{R}^{r \times (r-r_{\star})} $ to denote the orthogonal matrix, whose column space is orthogonal to the column space of $W_t$ (i.e.~the basis of the subspace orthogonal to the span of $\WWt$). We then can decompose $U_t$ into 
\begin{align*}
U_t = \bracing{\text{signal term}}{ U_t W_t W_t^T } +  \bracing{\text{noise term}}{U_t W_{t,\perp} W_{t,\perp}^T}.
\end{align*}
This decomposition has the following two simple properties, which will be useful throughout our proofs.
\begin{lemma}[Properties of signal-noise decomposition]\text{ }\\
\vspace{-0.5cm}
\begin{enumerate}
\item The column space of the noise term is orthogonal to the column span of $X$, i.e. $V_X^T U_t W_{t,\perp} =0$.
\item When $V_X^T U_t $ is full rank, then the signal term  has rank $r_{\star}$ and the noise term has rank at most $r - r_{\star} $.
\end{enumerate}
\end{lemma}
\begin{proof}
The first statement follows directly from the observation $V_{X}^T U_t W_{t,\perp} W_{t,\perp}^T = V_{X}^T U_t  \left(   \Id - \WWt \WWt^T  \right)= 0$. The second statement is a direct consequence of the definition of $\WWt$.
\end{proof}
We would like to note that decomposing $U_t$ into two terms has appeared in prior work such as \cite{LiMaCOLT18} as well as in earlier work in the compressive sensing literature. However, \cite{LiMaCOLT18}  uses a different decomposition. A key advantage of our decomposition is that it only depends on $U_t$ and $X$, whereas the decomposition in \cite{LiMaCOLT18} depends on all previous iterates $U_0,  U_1, \ldots, U_{t-1}$.

\subsection{The spectral/alignment phase}\label{sec:keyideaspec}
In this section we turn our attention to giving an overview of the key ideas and proofs of the spectral/alignment phase. More specifically, we will argue that in the first few iterations gradient descent implicitly performs a form of spectral initialization. By that, we mean that after the first few iterations the column span of the signal term $U_t \WWt \WWt^T$ is aligned with the column span of $X$ and that $\specnorm{U_t \Wtperp}$ is relatively small compared to $ \sigma_{\min } \left(\Ut \WWt  \right) $, meaning that the signal term dominates the noise term. 

We now provide the main intuition behind the analysis in our spectral/alignment phase. Our starting point is the observation that for the gradient at the initialization $ U_0 = \alpha U $ it holds that
\begin{align*}
\nabla f \left( U_0 \right)  &= - \left[   \mathcal{ A}^* \mathcal{ A} \left( XX^T - U_0 U_0^T \right) \right]  U_0 \\
&= - \alpha   \left[   \mathcal{ A}^* \mathcal{ A} \left( XX^T  \right) \right]  U + \alpha^3 \left[   \mathcal{ A}^* \mathcal{ A} \left(  U U^T \right) \right]  U.
\end{align*}
In particular, we observe that for $\alpha >0$ sufficiently small the second term is negligible. Hence, we have that
\begin{align*}
U_1 &=   U_0 - \mu \nabla f(U_0)              \\
&=  \left(  \Id + \mu   \left[   \mathcal{ A}^* \mathcal{ A} \left( XX^T  \right) \right]  \right) U_0 -  \alpha^2  \left[   \mathcal{ A}^* \mathcal{ A} \left(  U U^T \right) \right]  U_0 \\
&=  \left(  \Id + \mu     \mathcal{ A}^* \mathcal{ A} \left( XX^T  \right)   \right) U_0  +  O \left(  \alpha^2 \specnorm{U_0}  \right)
\end{align*}
In the first few iterations (i.e.~small $t$) we expect the matrix $U_t$ to be small and continue to scale commensurately with $\alpha$ and we expect that a similar approximation holds for the first iterations. Hence, for $\alpha$ sufficiently small we can approximate $U_t$ by
\begin{equation}\label{ineq:sketch1}
U_t  \approx \bracing{=: Z_t}{ \left(  \Id + \mu    \mathcal{ A}^* \mathcal{ A} \left( XX^T  \right)    \right)^t }  U_0:=\tilde{U}_t. 
\end{equation}
Figure \ref{power_method_figure} clearly illustrates that the first few iterations of gradient descent behave essentially identical to \eqref{ineq:sketch1} confirming our intuition and proofs.
\begin{figure}[t]
	\centering
	\begin{subfigure}{0.4\textwidth}
		\includegraphics[scale=0.7]{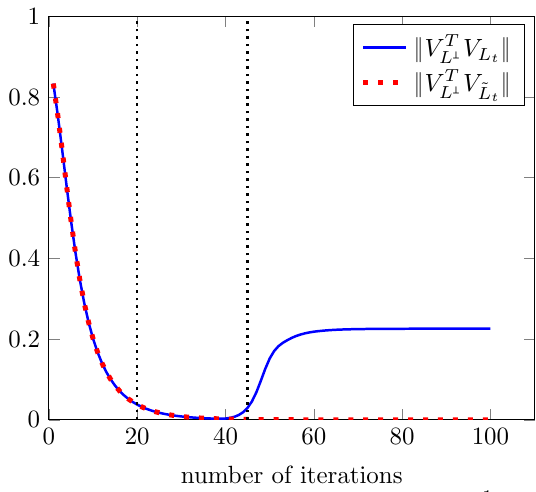}
		\caption{}
		\label{subfigure:powermethod1}
	\end{subfigure}%
	\begin{subfigure}{0.45\textwidth}
		\includegraphics[scale=0.7]{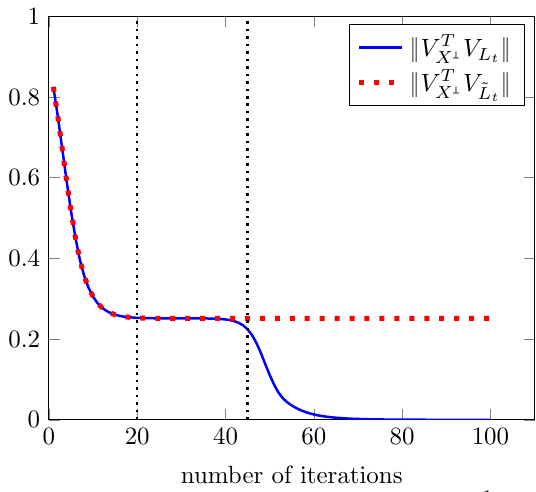}
		\caption{}
		\label{subfigure:powermethod2}
	\end{subfigure}
	\caption{\textbf{Depiction of the spectral alignment phase: in the first few iterations, gradient descent with small initialization behaves like a power method.} Here, $L$ denotes the subspace spanned by the eigenvectors corresponding to the $r_{\star}$ largest eigenvalues of the matrix $\mathcal{ A}^* \mathcal{ A}\left(XX^T\right)$. $\Lt$ denotes the subspace spanned by the eigenvectors corresponding to the $r_{\star}$ largest eigenvalues of the matrix $\UUt \UUt^T$. Moreover,  $\tilde{L}_t$ denotes the subspace spanned by the eigenvectors corresponding to the $r_{\star}$ largest eigenvalues of the matrix $ \tilde{U}_t  \tilde{U}_t^T $, where  $\tilde{U}_t = \left(  \Id + \mu    \mathcal{ A}^* \mathcal{ A} \left( XX^T  \right)    \right)^t   U_0$. In Figure \ref{subfigure:powermethod1} we see that in the first  iterations  $\UUt$ and  $\tilde{U}_t$ learn the subspace $L$ at  the same rate. In Figure \ref{subfigure:powermethod2} we observe that  also the angle between $V_X$ and $\Lt$, respectively $\tilde{L}_t$, decreases monotonically in the spectral phase and then both angles stay constant in the saddle-avoidance phase. We see that in the local convergence phase the angle between $V_X$ and $\Lt$ converges to $0$ as expected since $\UUt$ converges to $X$ up to a rotation.}
	\label{power_method_figure}
\end{figure}

We indeed formally prove that such an approximation holds in Section \ref{pfspectral}. We note that the matrix $Z_1=\Id + \mu    \mathcal{ A}^* \mathcal{ A} \left( XX^T  \right)$ is the basis for the commonly used spectral initialization, where typically a factorization of the rank $r_*$ approximation of this matrix is used as the initialization \cite{tu2015low,chen_implicit_regularization,xiaodong_matrixcompletion}. Therefore, the approximation \eqref{ineq:sketch1} suggests that gradient descent iterates modulo the normalization are akin to running power method on $Z_1$. Therefore, we expect the column space of the signal term at the end of the spectral phase to be closely aligned with those of the commonly used spectral initialization techniques and in turn the column space of $X$ as we formalize below.

To be more precise about the aforementioned alignment with $X$, let the singular value decomposition of $ \mathcal{ A}^* \mathcal{ A} \left( XX^T  \right) $ be given by $ \mathcal{ A}^* \mathcal{ A} \left( XX^T  \right) = \sum_{i=1}^{n} \lambda_i v_i v_i^T $. It follows that
\begin{equation}\label{ineq:sketch2}
U_t  \approx   \left[   \sum_{i=1}^n  \left( 1+ \mu \lambda_i  \right)^t v_i v_i^T  \right] U_0.
\end{equation}
It is well-known that when the operator $ \mathcal{ A}$ obeys the restricted isometry property we have
\begin{equation*}
 \mathcal{ A}^* \mathcal{ A} \left( XX^T  \right)  \approx XX^T .
\end{equation*}
 In particular, we have that 
 \begin{equation*}
 \lambda_{r_{\star} +1 } \left( \mathcal{ A}^* \mathcal{ A} \left( XX^T  \right) \right)  \ll  \lambda_{r_{\star} } \left(   \mathcal{ A}^* \mathcal{ A} \left( XX^T  \right)  \right)  .
 \end{equation*}
Hence, it follows from
\begin{equation*}
Z_t =  \sum_{i=1}^n  \left( 1+ \mu \lambda_i  \right)^t v_i v_i^T  
\end{equation*}
 that $\lambda_{r_{\star} } \left(Z_t\right)/\lambda_{r_{\star}+1 } \left(Z_t\right) $ grows exponentially. In particular, this means that
 \begin{equation*}
Z_t  \approx   \sum_{i=1}^{r_{\star}}  \left( 1+ \mu \lambda_i  \right)^t v_i v_i^T  
 \end{equation*}
and, by \eqref{ineq:sketch1}, 
 \begin{equation*}
U_t  \approx   \left[   \sum_{i=1}^{r_{\star}}  \left( 1+ \mu \lambda_i  \right)^t v_i v_i^T  \right] U_0.
 \end{equation*}
 Since $U_0$ is a \textit{random} Gaussian matrix, for an appropriate choice of $t$, we will be able to show that the matrix $U_t$ has the following two properties with high probability, where $L=\text{span} \left\{ v_1; \ldots; v_{r_{\star}} \right\}$  and $ \Lttilde $ is the projection of $U_t$ onto its best rank-$r_{\star}$ approximation:
 \begin{itemize}
\item There is a sufficiently large gap between $ \sigma_{r_{\star}  } \left( \Ut \right) $ and $\sigma_{r_{\star} +1 } \left( \Ut \right)  $, i.e., $  \frac{ \sigma_{r_{\star}  } \left( \Ut \right) }{\sigma_{r_{\star} +1 } \left( \Ut \right) } \ge \Delta >1  $, where $\Delta$ is  an appropriately chosen constant.
\item We have that $\specnorm{ V_{L^\perp}^T V_{\Lttilde}  } $ is small. Since the column space of $ \mathcal{ A}^* \mathcal{ A} \left( XX^T \right)$ is aligned with the column space of $X$, this also implies that $ \specnorm{ V_{X^\perp}^T V_{\Lttilde}  }  $ is small.
 \end{itemize}
This confirms that in the first few iterations, gradient descent indeed implicitly performs akin to spectral initialization with the column space of $\UUt$ aligned with the column space of $X$.
However, this does not yet fully complete our analysis for the spectral/alignment phase, since critical to the analysis of second phase we need certain properties to hold for the signal and noise terms $\Ut \WWt$ and $\Ut \Wtperp$ (see Section \ref{section:noiseterm}) rather than the singular value decomposition of $U_t$. However, using the  properties of the SVD of $U_t$, which are listed above, we will  establish the following properties of $\Ut \WWt$ and $\Ut \Wtperp$ .
\begin{itemize}
	\item  The column space of $U_t W_t$ is aligned with the column space of $X$: $ \specnorm{ V_{X^\perp}^T  V_{U_t W_t}   } \le c  \kappa^{-2} $.
	\item The spectral norm of the noise term is not too large compared to the minimum singular value of the signal term, i.e.,  $ 2 \sigma_{\min} \left( U_t W_t  \right) \ge \specnorm{ U_t  W_{t,\perp} }  $.
	\item  The spectral norm of the noise term is bounded from above in the sense that i.e., $ \specnorm{\UUt \Wtperp} \ll  \sigma_{\min } \left(X\right) $.

	\item The spectral norm of $U_{t}$ is bounded, i.e., $ \specnorm{ U_{t} } \le 3 \specnorm{X}$.
\end{itemize}

\subsection{The saddle avoidance phase and the refinement phase}\label{keypflocal}

In the next two phases, we will show that the signal term $\UUt \WWt \WWt^T \UUt^T$ converges towards $XX^T$, whereas the spectral norm of the noise term, i.e. $\specnorm{\UUt \Wtperp} $, stays small. For that, we show that throughout this process the columns of the matrices $X$ and $\UUt \WWt $ stay approximately aligned, i.e., the angle $\specnorm{ V_{X^\perp}^T  V_{U_t W_t}   }  $ stays small. This latter property also ensures that after the spectral phase the iterates are not too close to well known saddle points of the optimization landscape (it is known that this problem may have degenerate saddle points at a point $\bar{U}$ obeying rank$(\bar{U})< r_*$ \cite{bhojanapalli2016global}). See Figure \ref{fig:loss_landscape} for a depiction of the gradient flows of the landscape when $r_*=1$.

\begin{figure}
	\centering
	\includegraphics[scale=0.6]{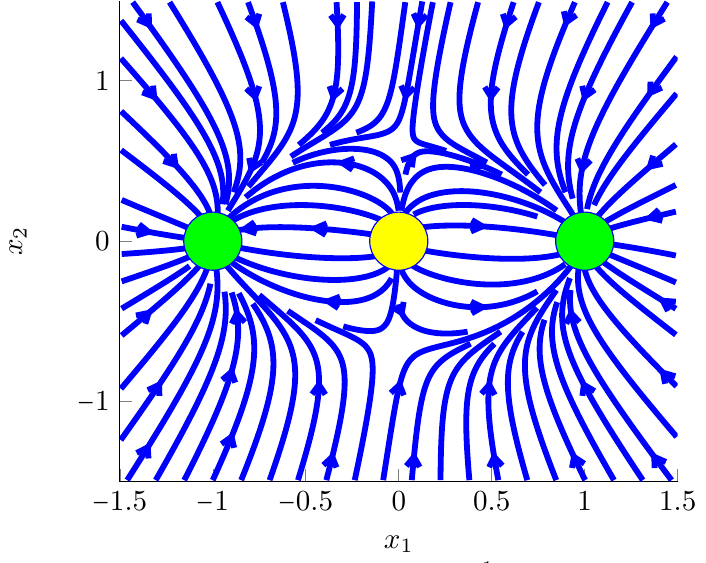}
	\caption{\textbf{Depiction of saddle avoidance and local refinement phases.} In this figure we depict the gradient field of the loss function $f$ with $n=2$, $r_{\star}=r=1$, $m=15$, and $X=\left(1 \ \ 0\right) $. The green circles depict the two generalizable global minima of $f$, namely  $\left(1 \ \ 0\right)  $ and $ \left(-1 \ \ 0\right)  $. The red circle depicts the saddle point $ \left(0 \ \ 0\right)$. As this figure demonstrates starting from small random initialization after a while the trajectory moves away from the saddle (i.e.~avoids it) and then converges to one of the two generalizable global optima (i.e.~the local refinement phase). 
}
	\label{fig:loss_landscape}
\end{figure}
\noindent Next we sketch the proofs of Phase II and Phase III in more detail.\\

\noindent\textbf{Phase II: }In this phase, we will show that the minimal singular value of the signal term, $\sigma_{ \min } \left( \UUt \WWt \right) $ grows exponentially, until it holds that $\sigma_{ \min  } \left( \UUt \WWt \right)  \ge \frac{\sigma_{ \min  } \left(X\right)}{\sqrt{10}} $. To this aim, we show that
\begin{align*}
\sigma_{\min} \left( \UUtplus \WWtplus \right) 	\ge \sigma_{\min} \left( V_X^T \UUtplus \right) \ge    \sigma_{\min} \left( V_X^T \UUt \right)  \left( 1 +  \frac{1}{4} \mu \sigma_{\min}^2 \left(X\right)   -  \mu  \sigma_{\min}^2 \left( V_X^T \UUt \right)  \right)
\end{align*}
holds under suitable assumptions (see Lemma \ref{lemma:Vxgrowth}). In order to show that the spectral norm of the noise term $ \specnorm{ \UUt \Wtperp  } $ grows much slower than $\sigma_{\min} \left( \UUtplus \WWtplus \right) 	$, we establish the inequality
\begin{equation}\label{ineq:aux20001}
\begin{split}
&\specnorm{  \UUtplus \Wtplusperp  }\\
 \le & \left(   1- \frac{\mu}{2} \specnorm{\UUt  W_{t,\perp}}^2   +   9 \mu \specnorm{V^T_{X^\perp}  V_{\UUt \WWt}}  \specnorm{X}^2  + 2\mu \specnorm{\left(   \mathcal{A}^* \mathcal{A} -\Id \right) \left(XX^T -\UUt \UUt^T\right)  }  \right)  \specnorm{\UUt W_{t,\perp}}
\end{split}
\end{equation}
(see Lemma \ref{lemma:specnormperp}). The next inequality (see Lemma \ref{lemma:anglecontrol}) shows that $\specnorm{V^T_{X^\perp}  V_{\UUt \WWt}}$ stays sufficiently small
\begin{equation*}
\begin{split}
&\specnorm{ V_{X^\perp}^T V_{\UUtplus \WWtplus} }   \\
\le &  \left(   1-    \frac{\mu}{4} \sigma_{\min}^2 \left( X \right)    \right) \specnorm{ V_{X^\perp}^T V_{\UUt \WWt} }  + 100   \mu  \specnorm{     \left( \Id -\mathcal{ A}^*  \mathcal{A}   \right) \left( XX^T - \UUt \UUt^T  \right)   }  +500   \mu^2  \specnorm{XX^T -\UUt \UUt^T  }^2.
\end{split}
\end{equation*}
As mentioned above, this implies in particular, that $U_t$ stays sufficiently far away from saddle points $ \bar{U} $, which are rank-deficient, e.g., $ \text{rank} \left(\bar{U}\right) < r_{\star}$.\\

\noindent\textbf{Phase III:} After we have shown that $\sigma_{ \min  } \left( \UUt \WWt \right)  \ge \frac{\sigma_{ \min  } \left(X\right)}{\sqrt{10}} $ holds for some $t$, we enter the \textit{local refinement phase}. We start by observing that the error $ \fronorm{XX^T -\UUt \UUt^T}  $ can be decomposed into two summands, i.e.
\begin{equation}\label{ineq:intern458}
\fronorm{\UUt \UUt^T - XX^T} \le 4 \fronorm{ V_X^T \left(  XX^T -\UUt \UUt^T  \right)} + \fronorm{  \UUt \Wtperp\Wtperp^T \UUt^T  } .
\end{equation}
(see Lemma \ref{lemma:technicallemma8}). 
We will bound the second summand by using inequality \eqref{ineq:aux20001}, which is also valid for the third phase. We will show that the first summand decreases at a linear rate. For that, we establish the inequality
\begin{equation*}
\begin{split}
&\fronorm{V_{X}^T  \left( XX^T -\UUtplus\UUtplus^T \right) } \\
\le  &   \left(  1 - \frac{\mu}{200}  \sigma_{\min} \left(  X\right)^2 \right)\fronorm{  V_{X}^T    \left( XX^T -\UUt \UUt^T  \right) } +  \mu  \frac{\sigma_{\min}^2 \left(X\right) }{100}  \fronorm{  \UUt \Wtperp\Wtperp^T \UUt^T   }.
\end{split}
\end{equation*}
Hence, by using inequality \eqref{ineq:intern458} we will be able to show that $\fronorm{\UUt \UUt^T -XX^T} $ is decreasing, as long as the spectral norm of the noise term stays sufficiently small.

%% file: simulations.tex
\section{Numerical experiments}\label{section:numerical_experiments}
\begin{figure}[t]
	\centering
	\begin{subfigure}{0.5\textwidth}
	\includegraphics[scale=0.7]{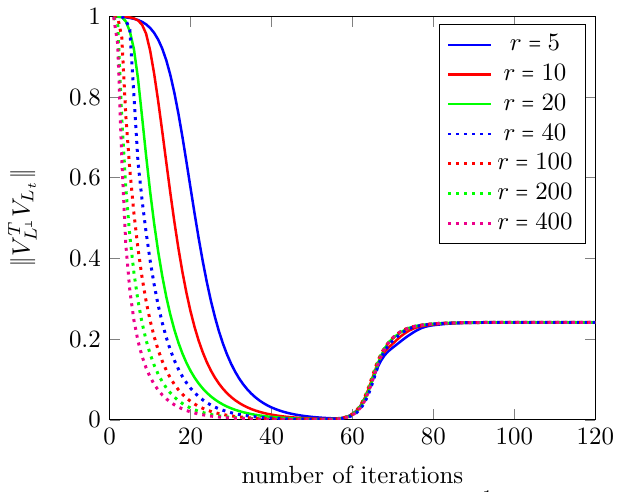}
	\caption{ }
\label{fig:power_method_different_r}
	\end{subfigure}%
	\begin{subfigure}{0.6\textwidth}
	\includegraphics[scale=0.7]{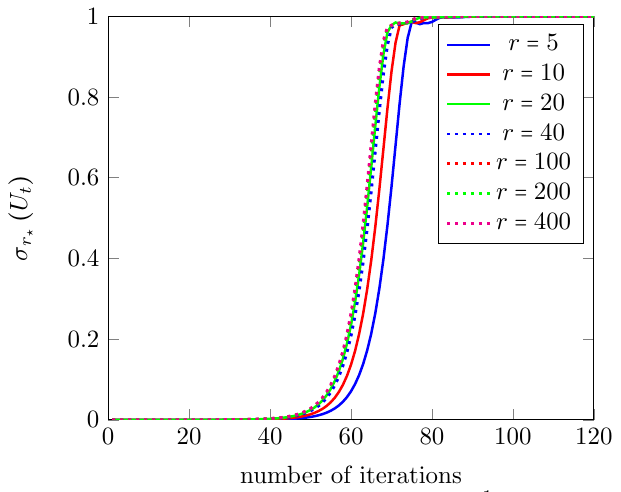}
	\caption{ }
\label{fig:power_method_different_r_sigma_min}
	\end{subfigure}

\begin{subfigure}{0.5\textwidth}
	\includegraphics[scale=0.7]{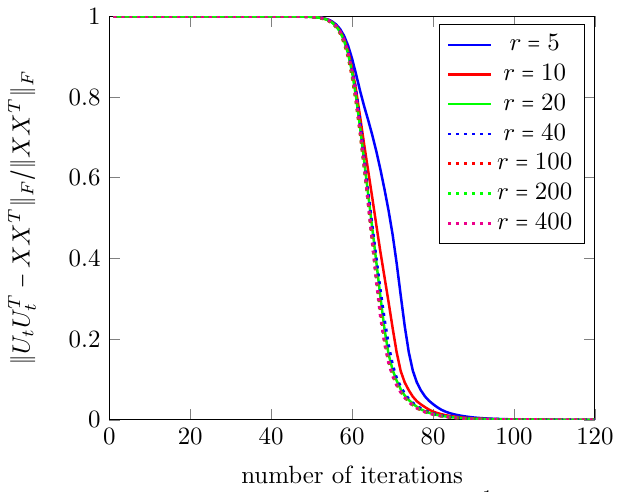}
	\caption{ }
	\label{fig:power_method_different_r_loss}
\end{subfigure}

	\caption{Impact of different levels of overparameterization on (a) the angle $\specnorm{V_{L^\perp}^T V_{L_t}} $  and  (b) the $r_{\star}$th largest singular value, (c) the trajectory of the (normalized) test error $ \Vert U_t U_t^T -XX^T \Vert_F/ \Vert XX^T \Vert_F $.}
\label{figure:different_r}
\end{figure}

In this section, we perform several numerical experiments to corroborate our theoretical results.

\noindent\textbf{Experimental setup.} For the experiments we set the ground truth matrix $X \in \mathbb{R}^{n \times r_{\star}}$ to be a random orthogonal matrix with $n=200$ and $r_{\star}=5$. Moreover, we use $ m= 10nr_{\star} =50n $ random Gaussian measurements. The initialization $U$ is chosen as in Theorem \ref{mainresult:runequaln} and we use a step size of $\mu=1/4$ which is consistent with these theorems. We note that while all experimental depictions are based on a single trial, in line with the NeurIPS guidelines we have drawn these curves multiple times (not depicted) and the behavior of the plots does not change.

\text{ }\\
\noindent\textbf{Depiction of the three phases and the role of overparameterization.} In our first experiment, we want to examine how increasing the number of parameters via  increasing the number of  the columns $r$ of the matrix $\Ut \in \mathbb{R}^{n \times r}$, affects the spectral phase.  To this aim we set the scale of initialization to $ \alpha= 1/\left(70n^2\right) $. Recall from Section \ref{section:key_ideas} that $L$ denotes the subspace spanned by the eigenvectors corresponding to the leading $r_{\star}$ singular values of $\mathcal{ A}^* \mathcal{ A} (XX^T)$ and $\Lt$ denotes the subspace spanned by the left-singular vectors corresponding to the largest $ r_{\star}$ singular values of  $\Ut$.

\underline{\emph{Spectral phase and alignment under different levels of overparameterization.}} First, we examine how the angle between these two subspaces (i.e.~$\specnorm{V_{L^\perp}^T V_{\Lt } }$) changes in the first few iterations. We depict the results for different $r$ in Figure \ref{fig:power_method_different_r}. We see that in the first few iterations, i.e. in the spectral phase, this angle converges towards zero. This confirms the main conclusion of this paper that the first few iterations of gradient descent from small random initialization indeed behaves akin to running power method for spectral initialization.
This experiment also shows that changing the number of columns $r$ of $\Ut$ has an interesting effect on the spectral phase. In particular, increasing $r$ allows the gradient descent algorithm to learn the subspace $L$ with fewer iterations, i.e. $\specnorm{V_{L^\perp}^T V_{\Lt} } $ becomes small with fewer iterations. This is in  accordance with our theory for $r_{\star} \le r \le n$ (see, for example, the first summand on the right-hand side of equation \eqref{ineq:iterationsdecomposition}), where we show that more overparameterization allows gradient descent to leave the spectral phase earlier. Interestingly, this improvement continues to hold even when increasing $r$ beyond $n$ allowing for even faster convergence of $ \specnorm{ V^T_{L^\perp} V_{\Lttilde} } $. This holds even though in this case the rank of $U_0$ is still not larger than $n$. One potential explanation for this phenomenon might be that for such a choice of  $r$ the matrix $U_0$ is better conditioned.
 
\underline{\emph{Growth of $ \sigma_{r_{\star}  } \left(\Ut\right) $ and saddle avoidance.}} In Figure  \ref{fig:power_method_different_r_sigma_min} we depict how $ \sigma_{r_{\star}  } \left(\Ut\right) $ grows during the training for different choices of $r$. We see that the curves look similar, although for smaller $r$ the growth phase sets in at a slightly later time. This is due to the fact that for smaller $r$, as we have seen in Figure \ref{fig:power_method_different_r}, Phase I, the spectral phase takes longer to complete. 

\underline{\emph{Evolution of the test error and the refinement phase.}} Similarly, in Figure  \ref{fig:power_method_different_r_loss} we depict how the (normalized) test error $\fronorm{\Ut \Ut^T-XX^T}/\fronorm{XX^T} $ evolves during the training for different choices of $r$. We observe that for smaller $r$ the third phase sets in slightly later. Again, this is due to the fact that for smaller $r$ the spectral phase takes slightly longer to complete (see inequality \eqref{ineq:iterationsdecomposition}).\\


%

\noindent\textbf{Test error under different scales of initialization.} In the next experiment, we focus on understanding how the scale of initialization $\alpha$ affects the generalization error $\fronorm{\Ut \Ut^T-XX^T}^2$. For that, we set $r=180$ and run gradient descent with for different choices of $\alpha$. 
We stop as soon as the training error becomes small ($ f \left(\Ut \right)  \le 0.5 \cdot 10^{-9} $).
We depict the results in Figure \ref{fig:generalization_scale of init}. We see that the test error decreases as $\alpha$ decreases. In particular, this figure indicates that the test error depends polynomially on the scale of initialization $\alpha$. This is in line with our theory, where we also show that the test error decreases at least with the rate $ \alpha^{21/16} $ (see inequality \eqref{ineq:generalizationerror1} in Theorem \ref{mainresult:runequaln}).\\

\begin{figure}[t]
	\centering
		\includegraphics[scale=0.7]{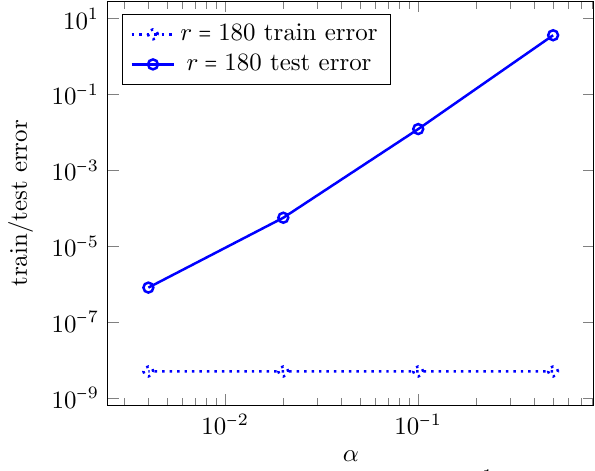}
	\caption{Relative test error $\frac{\fronorm{\UUt \UUt^T-XX^T}}{\fronorm{XX^T}}$  for different scales of initialization $\alpha$ .}
	\label{fig:generalization_scale of init}
\end{figure}

\noindent\textbf{Change of test and train error during training.} In the next experiment, we set $r=180$ and examine how the test error $\fronorm{U_tU_t^T-XX^T}^2$ and the train error $ f\left(\Ut\right) $ changes throughout training and, in particular, how this depends on the scale of initialization. To this aim, we run gradient descent with $4\cdot10^5$ iterations. We see that for a small scale of initialization, $\alpha=10^{-3}$, which is the scenario studied in this paper, both test error and train error decrease throughout training.

We observe that in the beginning, as described our theory, both test and train error decrease rapidly. After that the decrease of both test and train error slows down significantly. Moreover, the train error converges towards zero, in contrast to the test error. One reason for the slow convergence in this phase might be that $\UUt$ is ill-conditioned in the sense that $ \sigma_{r_{\star}  } \left( \UUt \WWt \right) $ is much larger than $ \specnorm{ \UUt \Wtperp } $. It is an interesting future research direction to extend our theory to this part of the training.

For large scale of initialization $\alpha=0.5$, we observe a very different behaviour. We see that the train error converges  with linear rate until machine precision is reached. However, the test error barely changes throughout the training. This scale of initialization corresponds to the \textit{lazy training regime} \cite{chizat2019lazy}, where the parameters stay close to the initialization during the training. We depict the results in Figure \ref{figure:trajectory}.

\begin{figure}[t]
	\centering
	\begin{subfigure}{0.5\textwidth}
		\includegraphics[scale=0.7]{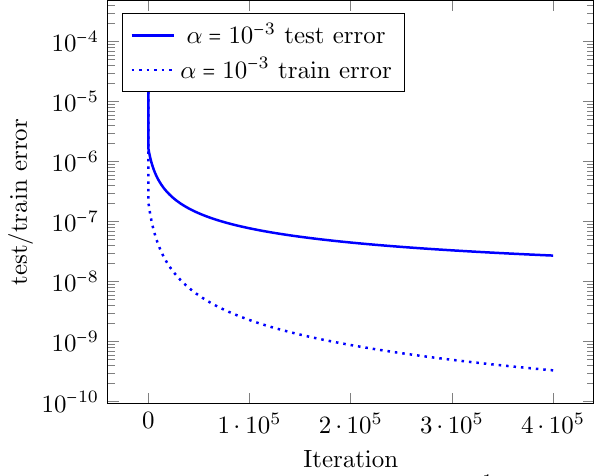}
		\caption{$\alpha=10^{-3}$ }
		\label{fig:alphasmall}
	\end{subfigure}%
	\begin{subfigure}{0.5\textwidth}
		\includegraphics[scale=0.7]{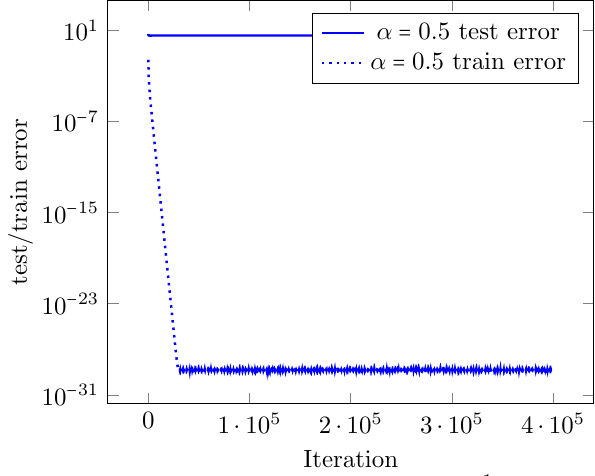}
		\caption{$\alpha=0.5$}
		\label{fig:alphalarge}
	\end{subfigure}
	\caption{Change of test error $\fronorm{U_tU_t^T-XX^T}^2 $ and train error $f \left(\Ut\right) $ for (a) small and (b) large $\alpha$ during training.}
	\label{figure:trajectory}
\end{figure}

\noindent \textbf{Number of iterations until convergence:} In the last experiment, we set $ \alpha=10^{-3} $ and examine how many iterations are needed until the test error $\fronorm{\Ut \Ut^T -XX^T}^2 $ falls below a certain threshold of $10^{-4}$ for different values of $r$ obeying $5 \le r \le 30 $. For each choice of $r$ we run the experiment ten times and then average the number of iterations for each choice of $r$. The results are depicted in Figure \ref{fig:number_of_iterations}. We observe that increasing the number of columns $r$ from $5$ to $10$, i.e., a small amount of overparameterization, decreases the number of iterations needed. After that the number of iterations needed stays roughly constant. This observation is in line with Figure \ref{figure:different_r}, where we have seen that overparameterization leads to fast decrease of the test error in the spectral phase (with diminishing speedup as $r$ becomes larger and larger) without affecting the other two phases.

\begin{figure}[t]
	\centering
	\includegraphics[scale=0.75]{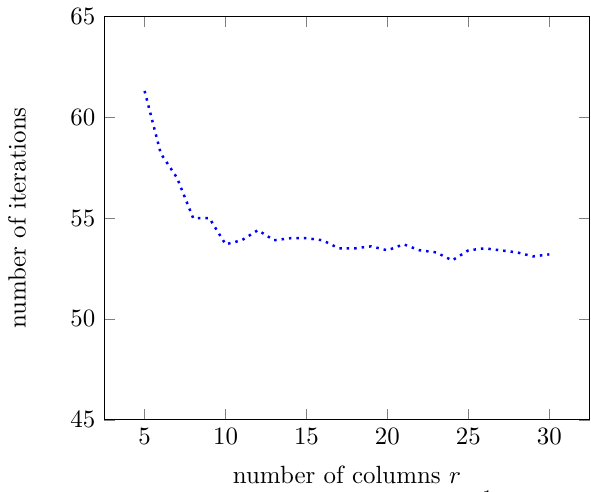}
	\caption{Number of iterations required for the test error to fall below $10^{-4} $ for different levels of overparameterization.}
	\label{fig:number_of_iterations}
\end{figure}

%% file: preliminaries.tex
\section{Preliminaries}

Before we are going into the details of the proof, we are collecting some useful definitions.
\subsection{Notation}
For  any matrix  $A \in \mathbb{R}^{n_1 \times n_2} $  we denote its spectral norm by $\specnorm{A} $ and the Frobenius norm by $\fronorm{A} = \sqrt{\trace \left(AA^T\right)} $. By $\nucnorm{A} $ we denote its nuclear norm, i.e. the sum of the singular values. Moreover, for two symmetric matrices $A,B \in S^d $ we define the Hilbert-Schmidt inner product by $\innerproduct{A,B} = \trace \left(AB\right) $. For a positive semidefinite matrix $A$ we denote its square root by $A^{1/2}$, i.e., the unique positive semidefinite matrix $B$ for which it holds that $B^2=A$. We also set $A^{-1/2} = \left(A^{1/2}\right)^{-1} $.

For any matrix  $A \in \mathbb{R}^{d_1 \times d_2} $ we will denote its singular value decomposition by $ A= V_A \Sigma_A W_A^T $ with $ V_A \in \mathbb{R}^{d_1 \times  \tilde{r}  } $, $ W_A \in \mathbb{R}^{d_2 \times \tilde{r} } $, $ {\Sigma}_A \in \mathbb{R}^{\tilde{r} \times \tilde{r}   } $, where $\tilde{r} $ denotes the rank of $A$. Moreover, by $ V_{A^{\perp}} \in  \mathbb{R}^{ (d_1 -\tilde{r}) \times d_1  }  $ we denote an orthogonal matrix, whose column span is orthogonal to the column span of the matrix $V_A $. Similarly, if $ U \subset \mathbb{R}^n $ is a subspace of dimension $\tilde{r}$, we will denote by $ V_U \in \mathbb{R}^{n \times  \tilde{r}  } $ a matrix, whose column span is the subspace $U$. Similarly as before, we will denote by $V_{U^{\perp}} \in \mathbb{R}^{ n \times   (n-\tilde{r}) } $ a matrix whose column span is orthogonal to the column span of $U$.

We will measure the angle between two subspaces $U_1, U_2 \subset \mathbb{R}^n  $ by $\specnorm{V^T_{U^{\perp}_1} V_{U_2} }  $. Moreover, we will  also several times rely on the well-known identity (see, e.g., \cite[Section 2]{chen2020spectral})
\begin{equation*}
\specnorm{V^T_{U^{\perp}_1} V_{U_2} } =  \specnorm{  V_{U_1} V_{U_1}^T -       V_{U_2} V_{U_2}^T        }.
\end{equation*}

\subsection{Restricted isometry property and related properties}
As discussed in Section \ref{section:lowrankmatrixrecovery}, we are going to assume that the measurement operator $\mathcal{ A}$ satisfies the restricted isometry property. However, as it turns out, the following two slightly weaker properties will suffice for our proof.
\begin{definition}
	The measurement operator $\mathcal{ A}: S^n \longrightarrow \mathbb{R}^m$ satisfies the spectral-to-spectral restricted isometry property of rank $r$ with constant $\delta >0$, if it holds for all symmetric matrices $Z$ of rank at most $r$ that
	\begin{equation*}
	\specnorm{\left(    \Id - \mathcal{ A}^* \mathcal{ A} \right) \left( Z \right)  } \le \delta  \specnorm{Z}.
	\end{equation*}
\end{definition}

\begin{definition}
	The measurement operator $\mathcal{ A}: S^n \longrightarrow \mathbb{R}^m$ satisfies the spectral-to-nuclear restricted isometry property with constant $\delta >0$, if it holds for all symmetric matrices $Z$ hat
	\begin{equation*}
	\specnorm{\left(    \Id - \mathcal{ A}^* \mathcal{ A} \right) \left( Z \right)  } \le \delta  \nucnorm{Z}.
	\end{equation*}
\end{definition}
The following lemma shows that these two properties are induced by the standard restricted isometry property (Definition \ref{def:restrictedisometry}).
\begin{lemma}\label{lemma:RIP1}
	Let $\mathcal{ A}: S^n \longrightarrow \mathbb{R}^m$ be a linear measurement operator.
	Then the following three statements hold.
	\begin{enumerate}
		\item Suppose that $\mathcal{ A}$ has the restricted isometry property as in \eqref{equation:RIP} for all matrices of rank $r+1$ with  constant $\delta_1 < 1$. Then $\mathcal{ A}$ has the spectral-to-spectral restricted isometry property of rank $r$ with constant $\sqrt{r} \delta_1$.
		\item Suppose that $\mathcal{ A}$ has the restricted isometry property as in \eqref{equation:RIP} for all matrices of rank $2$ with  constant $\delta_2 < 1$. Then $\mathcal{ A}$ has the spectral-to-nuclear restricted isometry property with constant $\delta_2$.
		\item Suppose that $\mathcal{ A}$ has the restricted isometry property as in \eqref{equation:RIP} for all matrices of rank $2r$ with  constant $\delta_3 < 1$. 
		Moreover, let $W \subset \R^n$ be a subspace of dimension $r$ and let $V_W \in \R^{n\times r}$ be an orthogonal matrix whose columns span the subspace $W$.
		Then it holds that for all symmetric matrices $Z \in \R^{n \times n} $ of rank at most $r$ that 
		\begin{equation}\label{ineq:RIPintern1}
			\fronorm{V_W^T \left[ \left( \Id - \mathcal{A}^* \mathcal{A} \right)  \left( Z \right) \right] } \le \delta_3 \fronorm{Z}.
		\end{equation}
	\end{enumerate}
\end{lemma}

\begin{proof}
	From \cite{candesplan} it follows that if $\mathcal{ A}$ has the restricted isometry property of  rank $r+r'$ with constant $\delta_1  <0$, then it holds for all matrices $Z$ with rank at most $r$ and all matrices $Y$ with rank at most $r'$ that
	\begin{equation}\label{equ:RIPintern}
	\big\vert \innerproduct{ \left( \Id - \mathcal{ A}^* \mathcal{ A} \right) \left(Z\right)  , Y}  \big\vert= \big\vert \innerproduct{  \mathcal{ A}^* \mathcal{ A} \left(Z\right)  , Y}  - \innerproduct{Z,Y}    \big\vert  \le  \delta_1 \fronorm{Z} \fronorm{Y}.
	\end{equation}
	In order to prove the first statement it suffices to note that there is a vector $v \in \mathbb{R}^n$ with $ \twonorm{v}=1$ such that
	\begin{equation*}
	\specnorm{\left(    \Id - \mathcal{ A}^* \mathcal{ A} \right) \left( Z \right)  }  =  \innerproduct{  \left( \Id - \mathcal{ A}^* \mathcal{ A} \right) \left(Z\right)  , vv^T}.
	\end{equation*}
	The claim follows from \eqref{equ:RIPintern}, $\twonorm{v}=1$, and from $\fronorm{Z} \le \sqrt{r} \specnorm{Z} $.
	
	In order to show the second claim consider the eigenvalue decomposition $Z= \sum_{i=1}^n \lambda_i v_i v_i^T $.  We compute that
	\begin{align*}
	\specnorm{\left(    \Id - \mathcal{ A}^* \mathcal{ A} \right) \left( Z \right)  } &\le \sum_{i=1}^n \vert \lambda_i \vert  \specnorm{\left(    \Id - \mathcal{ A}^* \mathcal{ A} \right) \left( v_i v_i^T \right)  }\\
	&\le \delta_2  \sum_{i=1}^n \vert \lambda_i \vert   \specnorm{v_i v_i^T}\\
	&= \delta_2 \nucnorm{Z},
	\end{align*}
	where in the second inequality we used the spectral-to-spectral restricted isometry property (with $r=1$), which holds due to the first part of this proof. This finishes the proof of the second statement.

	It remains to prove the third statement.
	For that, we need to prove inequality \eqref{ineq:RIPintern1}.
	For that, let $Z\in \R^{n \times n} $ be a symmetric matrix with rank at most $r$.
	Let $\tilde{Z} \in \R^{r \times n}$ be a matrix with $ \fronorm{\tilde{Z}} =1$ such that
	\begin{equation*}
	\fronorm{V_W^T \left[ \left( \Id - \mathcal{A}^* \mathcal{A} \right)  \left( Z \right) \right] }
	= \innerproduct{ V_W^T \left[ \left( \Id - \mathcal{A}^* \mathcal{A} \right)  \left( Z \right) \right], \tilde{Z} }.	
	\end{equation*}
	It follows that
	\begin{align*}
		\innerproduct{ V_W^T \left[ \left( \Id - \mathcal{A}^* \mathcal{A} \right)  \left( Z \right) \right], \tilde{Z} }	
		=& \innerproduct{  \left[ \left( \Id - \mathcal{A}^* \mathcal{A} \right)  \left( Z \right) \right], V_W \tilde{Z} }\\	
        \le & \delta_3 \fronorm{ Z} \fronorm{ V_W \tilde{Z}} 
		\le  \delta_3 \fronorm{ Z} ,
	\end{align*}
	where the last line is due to inequality \eqref{equ:RIPintern}.
	This finishes the proof of Lemma \ref{lemma:RIP1}. 
\end{proof}

%% file: spectral_phase_analysis.tex
\section{Analysis of the spectral phase}\label{pfspectral}

In the following we will provide an analysis of the spectral phase, where the proofs of the technical lemmas are deferred to Appendix \ref{appendix:spectralphase}.
Our first goal is to show that in the first few iterations $\Ut$ can be approximated by
\begin{equation*}
\Uttilde := \left(  \Id + \mu \bracing{=:M}{ \mathcal{A}^* \mathcal{A} \left(    XX^T   \right) }   \right)^t  U_0 =: Z_t U_0,
\end{equation*}
where we have set
\begin{equation*}
\Zt  = \left(\Id + \mu M\right)^t.
\end{equation*}
Next, we define
\begin{equation*}
t^{\star}:= \min \left\{  i \in \mathbb{N}: \ \specnorm{ \widetilde{U}_{i-1} - U_{i-1}}  > \specnorm{\widetilde{U}_{i-1} }    \right\}.
\end{equation*}
The next lemma shows how well $\UUt$ can be approximated by $\Uttilde $ for  $t\le t^{\star}$. To formulate it, we set $\Et = \Ut -\Uttilde $.  
\begin{lemma}\label{lemma:errorboundspectral}
	Suppose that $\mathcal{A}$ satisfies the rank-$1$ RIP with constant $\delta_1$. For all integers $t$ such that $1 \le t \le t^{\star}$ it holds that
	\begin{equation*}
	\specnorm{E_t} = \specnorm{\Ut - \Uttilde} \le    \frac{4}{\lambda_1 \left(M\right)} \alpha^3   \min \left\{ r;n \right\} \left(1+\delta_1 \right)   \left( 1+\mu \lambda_1 \left( M \right)  \right)^{3t}   \specnorm{U}^3 .
	\end{equation*}
\end{lemma}
The next lemma gives a lower bound for $ t^{\star}  $. In particular, this shows how long the approximation in Lemma \ref{lemma:errorboundspectral} is valid.
\begin{lemma}\label{lemma:init:closeness}
	Let $\Uttilde$ be as defined before and consider the eigenvalue decomposition $ \mathcal{A}^* \mathcal{A} \left(   XX^T\right) = \sum_{i=1}^{n} \lambda_i v_i v_i^T $.
	Then we have that
	\begin{equation*}
	t^{\star} \ge  \Bigg\lfloor  \frac{ \ln \left(   \frac{\lambda_1 \left(M\right)}{ 4 \alpha^2 \left( 1+\delta_1 \right) \specnorm{U}^3  } \left( \frac{\twonorm{U_0^T v_1}}{\alpha \min \left\{ r;n \right\}}  \right)  \right)}{ 2 \ln \left(  1+\mu \lambda_1 \left( M \right)  \right)} \Bigg\rfloor.
	\end{equation*}
\end{lemma}
Next, recall the relation
\begin{equation*}
U_t = \Uttilde +E_t 
= \Zt U_0 + \Et 
\end{equation*}
and denote by $L $ the subspace spanned by the eigenvectors, which correspond to the largest $r_{\star}$ eigenvalues of the matrix $M= \mathcal{ A}^* \mathcal{ A} \left(XX^T\right)$.  Note that $L$ is also the subspace spanned by the eigenvectors corresponding to the largest $r_{\star}$ eigenvalues of the matrix $\Zt$.  Denote by $ \Lttilde$ the subspace spanned by the left-singular vectors of $U_t = \Zt U_0 + E_t$, which corresponds to the largest $r_{\star}$ singular values.

Since $\Zt$ is computed via a power method we expect for $t$ large enough that $\lambda_{r_{\star} } \left(\Zt \right) \gg \lambda_{r_{\star} +1 } \left(\Zt\right) $. Moreover, if, in addition, $\specnorm{\Et}$ is sufficiently small, we expect in this case that the subspace $L$ is aligned with the subspace $\Lttilde $.
This is made precise by the following lemma.
	\begin{lemma}\label{lemma:fundam}
	Let $Z_t \in \mathbb{R}^{n \times n}$ be a symmetric matrix. Let $U \in \mathbb{R}^{n \times r}$ be a matrix and let $E_t\in \mathbb{R}^{n \times r}$. Set $U_0=\alpha U$ for some $\alpha > 0$. Moreover, assume that  
	\begin{equation}\label{ineq:initintern20}
	\sigma_{r_{\star} +1} \left( \Zt  \right)  \specnorm{U}  + \frac{\specnorm{E_t}}{\alpha} < 	\sigma_{r_{\star} } \left( \Zt  \right)  \sigma_{\min } \left(V_L^T U\right).
	\end{equation}
	\noindent Then the following three inequalities hold.
	\begin{align}
	\sigma_{r_{\star}}  \left(  \Zt U_0 +E_t \right) &\ge \alpha   \sigma_{r_{\star}} \left( \Zt  \right) \sigma_{\min} \left(  V_L^T U \right) - \specnorm{E_t},\label{ineq:initintern21}\\
	\sigma_{r_\star+1} \left(\Zt U_0+E_t\right) &\le \alpha \sigma_{r_{\star}+1} \left(\Zt\right) \specnorm{U} + \specnorm{E_t},\label{ineq:initintern22}\\
	\specnorm{V_{L^\perp}^T V_{\Lttilde} }  &\le \frac{ \alpha  \sigma_{r_{\star} +1 } \left(\Zt \right) \specnorm{U}   +  \specnorm{E_t }    }{  \alpha   \sigma_{r_{\star}} \left( \Zt  \right) \sigma_{\min} \left(  V_L^T U \right)  -  \alpha \sigma_{r_{\star}+1} \left(\Zt\right) \specnorm{U}  -\specnorm{E_t}  }.\label{ineq:initintern6}
	\end{align}
\end{lemma}
 Recall that we are interested in bounds for the quantities $ 	\sigma_{r_{\star}} \left( U_t W_t  \right)   $, $ 	\specnorm{ V_{X^\perp}^T  V_{U_tW_t} }    $, and $ \specnorm{ U_t  W_{t,\perp}  } $, i.e., properties of the signal and noise term. However, in the lemma above, we have obtained instead bounds for $ \sigma_{r_{\star}} \left( U_t \right) $, $ \specnorm{V_{X^\perp}^T  V_{\Lttilde} }  $, and $  \sigma_{r_{\star} +1} \left( U_t\right)  $, i.e. for the singular value decomposition of $\Ut$. 
 However, if $ \specnorm{V_{X^\perp}^T  V_{ \Lttilde  } } $ is small, these quantities are closely related to each other, as the next lemma shows.
\begin{lemma}\label{lemma:init:closeness2}		
   Assume that $  \specnorm{V_{X^\perp}^T  V_{ \Lttilde  }} \le \frac{1}{8}  $ for some $t\ge 1$. Then it holds that
	\begin{align}
	\sigma_{r_{\star}} \left( U_t W_t  \right)  &\ge   \frac{1}{2}  \sigma_{r_{\star}} \left( U_t \right),\label{ineq:initspec2}\\
	\specnorm{ V_{X^\perp}^T  V_{U_tW_t}   } & \le 7 \specnorm{V_{X^\perp}^T  V_{\Lttilde} }  , \label{ineq:initspec1}\\
	\specnorm{ U_t  W_{t,\perp}  } &\le  2   \sigma_{r_{\star} +1} \left( U_t\right) . \label{ineq:initspec3}
	\end{align}
\end{lemma}
By combining Lemma \ref{lemma:fundam} and Lemma \ref{lemma:init:closeness2}, we obtain the following technical result.
\begin{lemma}\label{lemma:Westimates}
	Let $XX^T $ be a low-rank matrix of rank $r_{\star}$. Assume that
	\begin{equation*}
	M:= \mathcal{A}^* \mathcal{A} \left( XX^T \right) = XX^T +\tilde{E},
	\end{equation*}
	with $\specnorm{ \tilde{E}} \le \delta  \lambda_{r_{\star}}  \left(XX^T\right) $, where $ \delta \le  \widetilde{c}_1 $ where $ \widetilde{c}_1>0$ is a sufficiently small absolute constant.
	Furthermore,  set $E_t = U_t -\tilde{U}_t $. Moreover, assume that
	\begin{equation}\label{equ:gammadefinition}
	\gamma :=    \frac{ \alpha \sigma_{r_{\star} +1 } \left( \Zt \right) \specnorm{U} +\specnorm{E_t} }{\alpha \sigma_{r_{\star} } \left(  \Zt   \right) \sigma_{\min } \left( V_L^T U \right) }  \le  \widetilde{c}_2   \kappa^{-2},
	\end{equation}
where $ \widetilde{c}_2>0$ is a sufficiently small, absolute constant. Then it holds that
	\begin{align}
	\sigma_{\min} \left( U_t W_t  \right)  &\ge    \frac{\alpha}{4}   \sigma_{r_{\star}} \left( Z_t  \right) \sigma_{\min} \left(  V_L^T U \right) ,  \\
	\specnorm{ U_t  W_{t,\perp} }  &\le  \frac{ \kappa^{-2}}{8}     \alpha   \sigma_{r_{\star}} \left( Z_t  \right) \sigma_{\min} \left(  V_L^T U\right), \\
	\specnorm{ V_{X^\perp}^T  V_{U_tW_t}   } &\le 56 \left(  \delta + \gamma   \right).
	\end{align}
\end{lemma}
In order to utilize Lemma \ref{lemma:Westimates}, we need to insert bounds for the approximation error $\specnorm{\Et}$, which we have derived in the Lemmas  \ref{lemma:errorboundspectral} and \ref{lemma:init:closeness}. This yields the following lemma. 
\begin{lemma}\label{lemma:spectralmain2}
	Fix a sufficiently small constant $c>0$. Let $U \in \mathbb{R}^{n \times r}$. Assume that $\mathcal{A}$ has the spectral-to-nuclear restricted isometry property for some constant $\delta_1 < 1$. Moreover, assume that
	\begin{equation*}
	M:= \mathcal{A}^* \mathcal{A} \left( XX^T \right) = XX^T +\tilde{E},
	\end{equation*}
	with $\specnorm{ \tilde{E}} \le \delta  \lambda_{r_{\star}}  \left(XX^T\right) $, where $ \delta \le c_1 \kappa^{-2} $. 
	Denote by $L$ the subspace spanned by the eigenvectors corresponding to the $r_{\star}$ largest eigenvalues of the matrix $\mathcal{ A}^* \mathcal{ A}\left(XX^T\right)$.
	Let $U_0=\alpha U$, where
	\begin{equation}\label{ineq:assumptiononalpha}
	\alpha^2 \le \frac{ c_2  \specnorm{X}^2    }{  32  \min \left\{ r;n \right\}  \kappa \specnorm{U}^3}   \left(    \frac{2\kappa^2 \specnorm{U}}{ c_3 \sigma_{\min } \left( V_L^T U \right)   }    \right)^{-12\kappa^2}   \min \left( \sigma_{\min } \left(V_L^T U\right);  \twonorm{U^T v_1}     \right),
	\end{equation}
	where $v_1 $ denotes the eigenvector corresponding to a leading eigenvalue of the matrix $ \mathcal{ A}^* \mathcal{ A} \left(XX^T\right)$. Assume that the step size satisfies $ \mu \le c_2 \kappa^{-2} \specnorm{X}^{-2}  $. 
	Then after
	\begin{align*}
	{t_{\star}} \asymp  \frac{1}{\mu \sigma_{r_{\star} }  \left(X\right)^2} \cdot  \ln \left(     \frac{2\kappa^2 \specnorm{U}}{ c_3 \sigma_{\min } \left( V_L^T U \right)   }    \right) 
	\end{align*}
	iterations it holds that
	\begin{align}
	\specnorm{U_{t_{\star}}} & \le 3 \specnorm{X},    \label{ineq:specmain1} \\
	\sigma_{\min} \left( U_{t_{\star}} W_{t_{\star}}  \right)  &\ge    \frac{\alpha \beta}{4}, \label{ineq:specmain2} \\
	\specnorm{ U_{t_{\star}}  W_{t_{\star},\perp} }  &\le  \frac{ \kappa^{-2}}{8} \alpha \beta, \label{ineq:specmain3}\\
	\specnorm{ V_{X^\perp}^T  V_{U_{t_{\star}}  W_{t_{\star}}}   } &\le  c \kappa^{-2},\label{ineq:specmain4}
	\end{align}
	where $\beta >0 $ satisfies
	\begin{equation}
	\sigma_{\min} \left(  V_L^T U \right)  \lesssim \beta \lesssim  \sigma_{\min} \left(  V_L^T U \right)  \left(    \frac{ \kappa^2 \specnorm{U}}{ c_3 \sigma_{\min } \left( V_L^T U \right)   }    \right)^{2}  = \frac{ \kappa^4 \specnorm{U}^2    }{c^2_3 \sigma_{ \min  } \left(  V_L^T U   \right)} \label{ineq:betabound}.
	\end{equation}
	Here $c_1, c_2, c_3 >0$ are absolute constants only depending on the choice of $c$.
\end{lemma}
Note that the result above holds for any initlialization $U$. To complete the proof we are going to utilize the fact that $U$ is a random matrix with Gaussian entries. This yields the following lemma, which is the main result of this section.
\begin{lemma}\label{lemma:spectralmain2:simplified}
	Fix a sufficiently small constant $c>0$.
	Let $U \in \mathbb{R}^{n \times r}$ be a random matrix with i.i.d. entries with distribution $ \mathcal{ N} \left(0, 1/\sqrt{r} \right) $ and let $ 0 < \varepsilon  < 1 $. Assume that $\mathcal{A}$ has the spectral-to-nuclear restricted isometry property for some constant $\delta_1 < 1$. Moreover, assume that
	\begin{equation*}
	M:= \mathcal{A}^* \mathcal{A} \left( XX^T \right) = XX^T +\tilde{E},
	\end{equation*}
	with $\specnorm{ \tilde{E}} \le \delta  \lambda_{r_{\star}}  \left(XX^T\right) $, where $ \delta \le c_1 \kappa^{-2} $. 
	Let $U_0=\alpha U$, where
	\begin{equation}\label{ineq:assumptiononalpha:simplified}
	\alpha^2   \lesssim \begin{cases}
	\frac{ \sqrt{\min \left\{ r;n \right\}   }   \specnorm{X}^2    }{      \kappa n^{3/2}}   \left(     2 \kappa^2  \sqrt{\frac{n}{\min \left\{ r;n \right\}}}   \right)^{-12\kappa^2} \quad &\text{if } r\ge 2r_{\star}  \\
	\frac{   \specnorm{X}^2    }{    n^{3/2}  \kappa }   \left(     \frac{  2\kappa^2\sqrt{rn} }{\varepsilon}   \right)^{-12\kappa^2}   \varepsilon \quad &\text{if } r < 2r_{\star}  
	\end{cases}.
	\end{equation}
	Assume that the step size satisfies $ \mu \le c_2 \kappa^{-2} \specnorm{X}^2  $. 
	Then with probability at least $1-p$, where
	\begin{align*}
	p = \begin{cases}
	O \left(  \exp \left(-\tilde{c} r  \right) \right) \quad &\text{if } r\ge 2r_{\star}  \\
	\left( \tilde{C} \varepsilon\right)^{r-r_{\star}+1}  +   \exp \left(-\tilde{c}r\right)  \quad &\text{if } r < 2r_{\star}  
	\end{cases}
	\end{align*}
	the following statement holds. After
	\begin{align*}
	{t_{\star}} \lesssim     \begin{cases}
	\frac{1}{\mu \sigma_{r_{\star} }  \left(X\right)^2} \cdot  \ln \left(     2 \kappa^2  \sqrt{\frac{n}{\min \left\{ r;n \right\}}}    \right)  \quad &\text{if } r\ge 2r_{\star}  \\
	\frac{1}{\mu \sigma_{r_{\star} }  \left(X\right)^2} \cdot  \ln \left(    \frac{  2\kappa^2\sqrt{rn} }{\varepsilon}   \right)  \quad &\text{if } r < 2r_{\star}  
	\end{cases}
	\end{align*}
	iterations it holds that
	\begin{align}
	\specnorm{U_{t_{\star}}} & \le 3 \specnorm{X},    \label{ineq:specmain1:simplified} \\
	\sigma_{\min} \left( U_{t_{\star}} W_{t_{\star}}  \right)  &\ge    \frac{\alpha \beta}{4}, \label{ineq:specmain2:simplified} \\
	\specnorm{ U_{t_{\star}}  W_{t_{\star},\perp} }  &\le  \frac{ \kappa^{-2}}{8} \alpha \beta, \label{ineq:specmain3:simplified}\\
	\specnorm{ V_{X^\perp}^T  V_{U_{t_{\star}}  W_{t_{\star}}}   } &\le  c \kappa^{-2},\label{ineq:specmain4:simplified}
	\end{align}
	where $\beta >0 $ satisfies
	\begin{equation*}
	\beta \lesssim \begin{cases}
	\frac{n \kappa^4  }{ c^2_3   \left(  \min \left\{ r;n \right\}\right)} \quad &\text{if } r\ge 2r_{\star}  \\
	\frac{ n \kappa^4  }{ c^2_3  \varepsilon }   \quad &\text{if } r < 2r_{\star}  
	\end{cases} \label{ineq:betabound:simplified}
	\end{equation*}
	as well as
	\begin{equation*}
	\beta \gtrsim \begin{cases}
	1\quad &\text{if } r\ge 2r_{\star}  \\
	\frac{\varepsilon}{r}   \quad &\text{if } r < 2r_{\star}  
	\end{cases}. \label{ineq:betabound:simplified2}
	\end{equation*}
		Here $c_1, c_2, c_3 >0$ are absolute constants only depending on the choice of $c$. Moreover, $\tilde{C}, \tilde{c} >0$ are absolute numerical constants.
\end{lemma}
The proof of Lemma \ref{lemma:spectralmain2:simplified} requires the following theorem, which gives a non-asymptotic lower bound for the smallest singular value of a Gaussian matrix.
\begin{theorem}\label{thm:rudelsonvershynin}\cite{rudelson_vershynin}
	Let $ G \in \mathbb{R}^{r_{\star}\times r} $ with $ r_{\star} \le r $ and i.i.d. Gaussian entries with distribution $ \mathcal{ N} \left(0,1/\sqrt{r} \right) $. Then for every $ \varepsilon >0 $ we have with probability at least $ 1- \left(C\varepsilon\right)^{r-r_{\star}+1} - \exp \left(-cr\right) $ that
	\begin{equation*}
	\sigma_{ \min  } \left(G\right) \ge  \varepsilon  \frac{\sqrt{r} - \sqrt{r_{\star}-1} }{\sqrt{r}}. 
	\end{equation*}
	The constants $ C,c > 0 $ are universal.
\end{theorem}
With this theorem in place we can prove Lemma \ref{lemma:spectralmain2:simplified}.
\begin{proof}[Proof of Lemma \ref{lemma:spectralmain2:simplified}]
We will deduce this statement from Lemma \ref{lemma:spectralmain2}. For that we need to estimate  $\specnorm{U}$, $ \sigma_{ \min  } \left( V_L^T U\right)  $, and $ \twonorm{U^T v_1} $. It is well-known (see, e.g. \cite[Section 4]{vershynin2018high}) that with probability at least $ 1- O \left( \exp \left(-c \max \left\{ r;n \right\}\right)  \right) $ it holds that
\begin{equation}\label{ineq:internUest1}
\specnorm{U} \lesssim  \sqrt{\max \left\{r;n \right\}/r} = \sqrt{\frac{n}{\min \left\{ r;n \right\}}}  .
\end{equation}
Next, note that again due to rotation invariance of the Gaussian measure the vector  $ U^T v_1 \in \mathbb{R}^r$ has i.i.d. entries with distribution $   \mathcal{ N} \left( 0,1/\sqrt{r} \right)   $. Hence,  with probability at least $ 1- O \left( \exp \left(-cr\right)  \right)  $ it holds that
\begin{equation}\label{ineq:internUest2}
\twonorm{U^T v_1} \asymp  1.
\end{equation}
Next, we note that due to rotation invariance of the Gaussian distribution the matrix $V_L^T U \in \mathbb{R}^{r_{\star} \times r} $ has i.i.d. entries with distribution $ \mathcal{ N} \left( 0,1/\sqrt{r} \right) $.
Moreover, note that using the elementary inequality $\sqrt{1-x} \le 1-\frac{1}{2x} $ we obtain that
\begin{equation}\label{ineq:internUesthelp}
\frac{ \sqrt{r} - \sqrt{r_{\star}-1} }{\sqrt{r}}  \ge \frac{ \sqrt{r}  -\sqrt{r_{\star}} \left(1 - \frac{1}{2r_{\star}}  \right)}{\sqrt{r}}
\gtrsim \begin{cases}
1 \quad &\text{if } r\ge 2r_{\star} \\
\frac{1}{r} \quad &\text{else}
\end{cases}.
\end{equation}
In order to proceed we are going to distinguish the following two cases.\\

\noindent \textbf{Case 1:} $ r\ge 2 r_{\star} $\\
Note that by choosing $\varepsilon>0$ appropriately, we obtain from Theorem \ref{thm:rudelsonvershynin} combined with inequality \eqref{ineq:internUesthelp} that with probability at least $ 1- O \left(\exp  \left(-cr  \right) \right) $ it holds that
\begin{equation}\label{ineq:internUest3}
\sigma_{ \min  } \left( V_L^T U\right)   \gtrsim 1.
\end{equation}
By combining the inequalities  \eqref{ineq:internUest1}, \eqref{ineq:internUest2}, and \eqref{ineq:internUest3}  with Lemma \ref{lemma:spectralmain2} the claim follows in the case that $r \ge 2 r_{\star} $.\\

\noindent \textbf{Case 2:} $r_{\star} \le  r\le  2 r_{\star} $\\
Similar to the first case, we note that by choosing $\varepsilon>0$ appropriately, we obtain by applying Theorem \ref{thm:rudelsonvershynin} combined with inequality \eqref{ineq:internUesthelp} that with probability at least $ 1- \left(C\varepsilon\right)^{r-r_{\star}+1}-\exp  \left(-cr\right)$ it holds that
\begin{equation}\label{ineq:internUest4}
\sigma_{ \min  } \left( V_L^T U\right)   \gtrsim \frac{\varepsilon}{  r }.
\end{equation}
By combining the inequalities \eqref{ineq:internUest1}, \eqref{ineq:internUest2}, and \eqref{ineq:internUest4} with Lemma \ref{lemma:spectralmain2} the claim follows.
\end{proof}

%% file: convergence_analysis.tex
\section{Analysis of the saddle avoidance and refinement phases}
Before stating and proving the main result of this section, Theorem \ref{thm:localconvergence}, we will first collect some useful lemmas. Their proofs are deferred to Appendix \ref{appendix:refinementphase}.

In Phase II we will show that $ \sigma_{ \min  } \left( \Ut \WWt\right) $ grows until it reaches $ \sigma_{ \min  } \left( \Ut \WWt\right)  \ge \frac{\sigma_{ \min  } \left(X\right)}{\sqrt{10}} $. For that, we note
\begin{align*}
\sigma_{\min} \left(  \UUt \WWt \right)  &\overeq{(a)}  \sigma_{\min} \left(  \UUt \WWt \WWt^T \right)\\
&\ge   \sigma_{\min} \left(   V_X^T \UUt \WWt \WWt^T \right) \\
&\overeq{(b)}  \sigma_{\min} \left(   V_X^T \UUt \right),
\end{align*}
where $(a)$ and $(b)$ follow from the definition of $\WWt$. Hence, in order to show that $\sigma_{\min} \left(  \UUt \WWt \right) \ge \frac{\sigma_{ \min  } \left(X\right)}{\sqrt{10}}   $ it suffices to show that $  \sigma_{\min} \left(  V_X^T \UUt \right) \ge \frac{\sigma_{ \min  } \left(X\right)}{\sqrt{10}}    $. For that, we will use the next lemma, which shows that $ \sigma_{\min} \left(  V_X^T \UUt \right) $ grows exponentially.
\begin{lemma}\label{lemma:Vxgrowth}
	Assume that $ \mu \le c \specnorm{X}^{-2} \kappa^{-2} $, $\specnorm{\UUt} \le  3 \specnorm{X} $, and that $ \specnorm{V_{X^{\perp}}^T V_{\UUt \WWt} } \le c \kappa^{-1} $.
	Moreover, suppose that
	\begin{equation}\label{ineq:intern414}
	\specnorm{ \left(  \mathcal{A}^* \mathcal{A} - \Id \right) \left(XX^T -\UUt \UUt^T \right)  }  \le c \sigma_{\min}^2 \left(X\right).
	\end{equation}
	Furthermore, assume that $V_X^T \UUt $ has full rank. Then it holds that
	\begin{equation*}
	\sigma_{\min} \left( V_X^T \UUtplus \right)   \ge  \sigma_{\min} \left( V_X^T \UUtplus \WWt \right) \ge    \sigma_{\min} \left( V_X^T \UUt \right)  \left( 1 +  \frac{1}{4} \mu \sigma_{\min}^2 \left(X\right)   -  \mu  \sigma_{\min}^2 \left( V_X^T \UUt \right)  \right).
	\end{equation*}
		Here $c>0$ is constant, which is chosen small enough. 
\end{lemma}
The next lemma will allow us to show that the noise term $ \specnorm{\UUt W_{t,\perp}} $ is growing slower than  $	\sigma_{\min} \left( V_X^T \UUtplus \right) $.
\begin{lemma}\label{lemma:specnormperp}
	Assume that  $ \mu \le c \min \left\{\specnorm{X}^{-2}; \specnorm{ \left( \mathcal{A}^* \mathcal{A} - \Id \right) \left(XX^T -\UUt \UUt^T\right)  }^{-1}   \right\}  $ and that $ \specnorm{\UUt} \le 3 \specnorm{X} $. Moreover, suppose that $V_X^T \UUtplus \WWt$  has full rank and that  $\specnorm{  V_{X^{\perp}}^T  V_{\UUt \WWt}} \le  c \kappa^{-1}$.
	Then it holds that
	\begin{equation*}
	\begin{split}
	 &\specnorm{  \UUtplus \Wtplusperp  }\\
	 \le &  \left(   1- \frac{\mu}{2} \specnorm{\UUt  W_{t,\perp}}^2   +   9 \mu \specnorm{V^T_{X^\perp}  V_{\UUt \WWt}}  \specnorm{X}^2  + 2 \mu \specnorm{\left(   \mathcal{A}^* \mathcal{A} -\Id \right) \left(XX^T -\UUt \UUt^T\right)  }  \right)  \specnorm{\UUt W_{t,\perp}}.
	\end{split}
	\end{equation*}
	Here, $c>0$ is an absolute constant chosen small enough.
\end{lemma}
The next lemma shows that the angle between the column space of the signal term $\UUt \WWt$ and column space of $X$ stays sufficiently small. 
\begin{lemma}\label{lemma:anglecontrol}
	Assume that $  \specnorm{\UUt \Wtperp } \le 2  \sigma_{\min} \left(\UUt \WWt\right) $ and $ \specnorm{\UUt} \le 3 \specnorm{X} $  holds. Moreover, assume that
	\begin{align}
	\specnorm{     \left( \Id -\mathcal{ A}^*  \mathcal{A}   \right) \left( XX^T - \UUt \UUt^T  \right)  } &\le c \sigma_{\min}^2 \left(X\right), \label{intern:eqref1}\\
	\specnorm{V_{X^\perp}^T V_{\UUt \WWt} } &\le c,   \label{intern:eqref2} \\
	\mu & \le c \kappa^{-2} \specnorm{X}^{-2},  \label{intern:eqref3}\\
	\specnorm{\UUt \Wtperp} &\le c\kappa^{-2}  \specnorm{X}, \label{intern:eqref4}
	\end{align}
	where $c>0$ is a small enough absolute constant. 
	Then it holds that
	\begin{align*}
	&\specnorm{ V_{X^\perp}^T V_{\UUtplus \WWtplus} }   \\
	\le &  \left(   1-    \frac{\mu}{4} \sigma_{\min}^2 \left( X \right)    \right) \specnorm{ V_{X^\perp}^T V_{\UUt \WWt} }  + 100   \mu  \specnorm{     \left( \Id -\mathcal{ A}^*  \mathcal{A}   \right) \left( XX^T - \UUt \UUt^T  \right)   }  +500   \mu^2  \specnorm{XX^T -\UUt \UUt^T  }^2.
	\end{align*}
\end{lemma}
The next lemma will show that we have $ \specnorm{\UUt} \le 3 \specnorm{X} $ for all $t$, a technical assumption which is needed in the above lemmas.
\begin{lemma}\label{lemma:Ucontrol}
	Assume that $ \specnorm{\UUt} \le 3 \specnorm{X} $, $  \mu \le \frac{1}{27 \specnorm{X}^2}$, and
	\begin{align*}
	\specnorm{ \left( \mathcal{ A}^* \mathcal{A} - \Id  \right) \left(  XX^T - \UUt \UUt^T \right) }    &\le   \Vert X \Vert^2  .
	\end{align*}
	Then it also holds that $ \specnorm{\UUtplus} \le  3 \specnorm{X} $.
\end{lemma}
With these lemmas in place, we will be able to show that $ \sigma_{ \min  } \left( \UUt \WWt \right) \ge \frac{\sigma_{ \min  } \left(X\right)}{\sqrt{10}} $ holds after sufficiently many iterations. Hence, we can enter Phase III, the \textit{local refinement phase}.

The next lemma is concerned with this third phase. It shows that $\UUt \WWt \WWt^T \UUt^T$ converges towards $XX^T$, when projected onto the column space of $X$. We are going to provide a somewhat more general version of the lemma than what is needed in the proofs of our main results, since it may be of independent interest. For that, let $\Vnorm{\cdot}$ be a matrix norm, which satisfies $ \Vnorm{ABC} \le \specnorm{A} \Vnorm{B} \specnorm{C} $ for all matrices $A,B,C$. Furthermore, we assume that $ \Vnorm{A} = \Vnorm{A^T} $ for all matrices $A$. For example, this property is fulfilled by all Schatten-$p$ norms.
\begin{lemma}\label{lemma:localconvergence}
	Assume that $ \specnorm{\UUt} \le 3\specnorm{X} $ and that $\sigma_{\min} \left(\UUt \WWt\right) \ge \frac{1}{\sqrt{10}} \sigma_{\min} \left(X\right)$. Moreover, assume that $\mu \le c \kappa^{-2} \specnorm{X}^{-2} $,  $ \specnorm{V_{X^\perp}^T V_{\UUt \WWt}} \le c \kappa^{-2}   $, and
	\begin{equation}\label{ineq:intern773}
	\begin{split}
	 \max \Bigl\{  \Vnorm{ V_X^T \left( \Id- \mathcal{A}^* \mathcal{A} \right) \left( \Delta_t \right)  }, 
	     \Vnorm{ V_{\Ut \WWt}^T \left( \Id- \mathcal{A}^* \mathcal{A} \right) \left(\Delta_t\right)  },
		 \specnorm{ \left( \Id- \mathcal{A}^* \mathcal{A} \right) \left(\Delta_t \right)  } \Bigr\} \le  c \kappa^{-2}    \Vnorm{ \Delta_t },
	\end{split}
	\end{equation}
	where $ \Delta_t:=  XX^T -\UUt \UUt^T $.
	Then it holds that
	\begin{equation*}
	\Vnorm{V_{X}^T  \left( XX^T -\UUtplus\UUtplus^T \right) } \le    \left(  1 - \frac{\mu}{200}  \sigma_{\min} \left(  X\right)^2 \right)\Vnorm{  V_{X}^T    \left( XX^T -\UUt \UUt^T  \right) } +  \mu  \frac{\sigma_{\min}^2 \left(X\right) }{100}  \Vnorm{  \UUt \Wtperp\Wtperp^T \UUt^T   }.
	\end{equation*}
	Here, $c>0$ is an absolute constant chosen small enough.
\end{lemma}
When applying this lemma in our proof, we are going to set $\Vnorm{\cdot} = \fronorm{\cdot} $. However, we believe that this lemma might be of independent interest, as it shows that $\Ut \Ut^T $ converges linearly towards $ XX^T $ with respect to several different norms.

Having collected all the necessary ingredients, we can state and prove the main theorem of this section.
\begin{theorem}\label{thm:localconvergence}
Let $\left\{ U_t \right\} \subset \mathbb{R}^{n \times r} $ be the sequence created by the gradient descent algorithm. Assume that $\mu \le c_1 \kappa^{-4} \specnorm{X}^{-2}$, where $c_1>0$ is a small enough absolute constant. 
Moreover, assume that $\mathcal{ A}$ satisfies the restricted isometry property for rank-$\left( 2r_{\star} +1 \right)$ matrices with constant $\delta \le  \frac{c_1}{\kappa^4 \sqrt{r_{\star} }}$. 
Let $\gamma >0$ and choose the iteration count $t_{\star}$ such that $  \sigma_{\min} \left(  U_{t_{\star}}W_{t_{\star}} \right) \ge \gamma $.
Furthermore, assume that the following conditions hold:
\begin{align}
\specnorm{U_{t_{\star}}W_{t_{\star},\perp}} &\le 2\gamma, \label{ineq:metaassumption1}\\
\specnorm{U_{t_{\star}}} &\le 3 \specnorm{X},  \label{ineq:metaassumption2}\\
\gamma &\le  c_2 \frac{\sigma_{\min } \left(X\right)}{\min \left\{ r;n \right\} \kappa^2},  \label{ineq:metaassumption3}\\
\specnorm{ V_{X^\perp}^T V_{U_{t_{\star}}W_{t_{\star}}}   }&\le c_2 \kappa^{-2},  \label{ineq:metaassumption4}
\end{align}
where $c_2>0$ is a small enough absolute constant.
Then after
\begin{equation}\label{ineq:iterationsbound}
\hat{t} - t_{\star} \lesssim   \frac{1}{\mu \sigma_{\min } \left(X\right)^2 }\ln \left( \max \left\{   1; \frac{\kappa r_{\star}}{ \min \left\{ r;n \right\}   -r_{\star}} \right\}   \frac{\Vert X \Vert}{\gamma} \right) 
\end{equation}
 iterations it holds that
\begin{equation*}
\frac{\fronorm{U_{\hat{t}} U_{\hat{t}}^{T}-XX^T }  }{\specnorm{X}^2} \lesssim \frac{ r_{\star}^{1/8}     \left(  \min \left\{ r;n \right\} -r_{\star}  \right)^{3/8} }{\kappa^{3/16}  }  \cdot   \frac{   \gamma^{21/16}  }{  \specnorm{X}^{21/16}}.
\end{equation*}
\end{theorem}
\begin{remark}\label{remark:iteration_count}
The proof of Theorem \ref{thm:localconvergence} shows that the number of iterations needed to complete Phase II is smaller than
\begin{equation*}
 t_1 - t_{\star}\lesssim   \frac{1}{\mu \sigma_{\min}^2 \left(X\right)}  \ln \left(  \frac{ \sigma_{\min } \left(X\right)}{\gamma }    \right)  
\end{equation*}
and that the number of iterations needed to complete Phase III is smaller than
\begin{equation*}
\hat{t}   - t_1 \lesssim  \frac{1}{\mu \sigma_{\min}^2 \left(X\right)} \ln \left( \max \left\{   1; \frac{\kappa r_{\star}}{ \min \left\{ r;n \right\}   -r_{\star}} \right\}   \frac{\Vert X \Vert}{\gamma} \right)   .
\end{equation*}
\end{remark}

\begin{proof}[Proof of Theorem \ref{thm:localconvergence}]
\textbf{Phase II: } In this phase, we will prove that $\sigma_{\min} \left(  V_X^T U_t  \right) $ is growing exponentially until it is at larger than  $\frac{ \sigma_{\min} \left(X\right)}{\sqrt{10}} $, while $\specnorm{U_{t}W_{t,\perp}} $ grows at a much slower rate. For that, set 
\begin{equation*}
t_1 := \min \left\{  t \ge t_{\star}:   \sigma_{\min} \left(  V_X^T  U_{t} \right)  \ge  \frac{ \sigma_{\min} \left(X\right)}{\sqrt{10}}   \right\}.
\end{equation*}
We will prove by induction that for $t_{\star} \le t \le t_1$ the following inequalities hold:
\begin{align}
\sigma_{\min} \left(  V_X^T  U_{t} \right) &\ge \frac{1}{2} \left(   1+   \frac{1}{8} \mu \sigma_{\min}^2 \left(X\right)    \right)^{t-t_{\star}}   \gamma,    \label{ineq:induction1}  \\
\specnorm{U_{t}W_{t,\perp}} &\le  2 \left(  1  +  80 \mu   c_2   \sigma_{\min }^2 \left(X\right)     \right)^{t-t_{\star}} \gamma,\label{ineq:induction2} \\
\specnorm{U_{t}} &\le 3\specnorm{X},  \label{ineq:induction3}\\
\specnorm{ V_{X^\perp}^T V_{U_{t}W_{t}}   } &\le c_2  \kappa^{-2} . \label{ineq:induction4}
\end{align}
Note that when the inequalities above hold, then from the definition of $t_1$ above and inequality \eqref{ineq:induction1} we can derive that we must have that
\begin{equation}\label{ineq:induction101}
t_1-t_{\star}  \le     \frac{16}{\mu \sigma_{\min}^2 \left(X\right)}  \ln \left(  \sqrt{\frac{5}{2}}  \cdot  \frac{ \sigma_{\min } \left(X\right)}{\gamma }    \right) .
\end{equation}
For $t=t_{\star}$, we first note that inequalities \eqref{ineq:induction2}, \eqref{ineq:induction3}, and \eqref{ineq:induction4} follow directly from our assumptions. In order to prove inequality \eqref{ineq:induction1} we note that
\begin{align*}
\sigma_{\min } \left(V_X^T U_{t_{\star}}\right) &\ge \sigma_{\min } \left( V_X^T V_{U_{t_{\star}} W_{t_{\star}}}  \right) \sigma_{\min }\left(  U_{t_{\star}} W_{t_{\star}}   \right)  \overgeq{(a)}  \frac{1}{2} \sigma_{\min }\left(  U_{t_{\star}} W_{t_{\star}}   \right) \overgeq{(b)}  \frac{\gamma}{2},
\end{align*}
where  inequality $(a)$ is a consequence of assumption \eqref{ineq:metaassumption4} and inequality $(b)$ follows from the definition of $\gamma$.
Assume now that we have shown these four inequalities for some $t < t_1$.
In order to prove them for $t+1$ we note first that
\begin{align}
&\specnorm{\left(   \mathcal{A}^* \mathcal{A} -\Id \right) \left(XX^T -U_tU_t^T\right)  } \nonumber\\
 \overleq{(a)} & \specnorm{\left(   \mathcal{A}^* \mathcal{A} -\Id \right) \left(XX^T -U_t W_t W_t^TU_t^T\right)  } + \specnorm{\left(   \mathcal{A}^* \mathcal{A} -\Id \right) \left(U_t  W_{t,\perp} W_{t,\perp}^T  U_t^T\right)  }  \nonumber \\
  \overleq{(b)}  &  \delta \sqrt{r_{\star}} \specnorm{ XX^T -U_t W_t W_t^TU_t^T}+ \delta \nucnorm{ U_t  W_{t,\perp} W_{t,\perp}^T  U_t^T  } \nonumber\\
  \le \  &  \delta \sqrt{r_{\star}} \left( \specnorm{X}^2 + \specnorm{ U_t W_t }^2   \right)+ \delta \nucnorm{ U_t  W_{t,\perp} W_{t,\perp}^T  U_t^T  } \nonumber\\
 \overleq{(c)} &  10 \delta \sqrt{r_{\star}}  \specnorm{X}^2 + \delta\left( \min \left\{ r;n \right\}    -r_{\star} \right)  \specnorm{ U_t  W_{t,\perp}}^2 \nonumber\\
   \overleq{(d)}  &  10 c_1 \kappa^{-2}  \sigma_{\min }^2 \left(X\right) +  4 \delta \left( \min \left\{ r;n \right\}  -r_{\star} \right) \left( 1  +  80 \mu   c_2   \sigma_{\min }^2 \left(X\right)     \right)^{2(t-t_{\star})} \gamma^2 \nonumber\\
   \overleq{(e)}  &   10 c_1  \kappa^{-2}   \sigma_{\min }^2 \left(X\right) +   8  \delta \left(  \min \left\{ r;n \right\}    -r_{\star} \right)   \sigma_{\min } \left(X\right)^{1/4} \gamma^{7/4}\nonumber
     \\
  \overleq{(f)}  &  40 c_1  \kappa^{-2}  \sigma_{\min }^2  \left(X\right) .\label{ineq:intern335}
\end{align}
In inequality $(a)$ we applied the triangle inequality and for  inequality $(b)$ we used the restricted isometry property as well as Lemma \ref{lemma:RIP1}.
Inequality $(c)$ is due to the induction assumption \eqref{ineq:induction3}.
In inequality $(d)$ we used the assumption $\delta \le \frac{ c_1}{ \kappa^4 \sqrt{r_{\star} }}$ as well as the induction assumption  \eqref{ineq:induction2}.
For inequality $(e)$ we used $t \le t_1$ as well as \eqref{ineq:induction101} and for  inequality $(f)$ we used \eqref{ineq:metaassumption3}.

Next, we observe that by Lemma \ref{lemma:Vxgrowth} we have that
\begin{align*}
 \sigma_{\min} \left( V_X^T U_{t+1} W_{t+1}\right) &=\sigma_{\min} \left( V_X^T U_{t+1} \right) \\ 
 &\ge \sigma_{\min} \left( V_X^T U_{t+1} \WWt \right)\\
   &\ge    \sigma_{\min} \left( V_X^T U_t \right)  \left( 1 +  \frac{1}{4} \mu \sigma_{\min}^2 \left(X\right)   -  \mu  \sigma_{\min}^2 \left( V_X^T U \right)  \right) \\
& \overgeq{(a)}    \sigma_{\min} \left( V_X^T U_t \right)  \left( 1 +  \frac{1}{8} \mu \sigma_{\min}^2 \left(X\right)   \right) .
\end{align*}
In $(a)$ we have used that $     \sigma_{\min} \left( V_X^T U_t \right)       \le \frac{ \sigma_{\min} \left(X\right)}{\sqrt{10}}  $, which follows from $t < t_1$. Using the induction assumption, this implies inequality \eqref{ineq:induction1}. Moreover, the inequality chain above shows that $ V_X^T \UUtplus W_{t+1} $ has full rank. This allows us to apply Lemma \ref{lemma:specnormperp}, which implies that
\begin{align*}
\specnorm{  U_{t+1} W_{t+1,\perp}  } &\le  \left(   1- \frac{\mu}{2} \specnorm{ \Ut \Wtperp}^2   +   9 \mu \specnorm{V^T_{X^\perp}  V_{\UUt W}}  \specnorm{X}^2  + 2\mu \specnorm{\left(   \mathcal{A}^* \mathcal{A} -\Id \right) \left(XX^T - \Ut \Ut^T  \right)  }  \right)  \specnorm{\Ut \Wtperp}\\
&\overleq{(a)}  \left(   1  +  80 \mu   c_2   \sigma_{\min }^2 \left(X\right)   \right)  \specnorm{\Ut  W_{t,\perp}}\\
&\le  2 \left(  1  +  80 \mu   c_2   \sigma_{\min }^2 \left(X\right)     \right)^{t+1-t_{\star}} \gamma,
\end{align*}
where in inequality $(a)$ we used \eqref{ineq:induction4} as well as \eqref{ineq:intern335} and that the constant $c_1$ is chosen sufficiently small. This shows inequality \eqref{ineq:induction2}. Next, due to inequality \eqref{ineq:intern335}, our induction assumptions, and Lemma  \ref{lemma:Ucontrol} we obtain that $\specnorm{U_{t+1}} \le 3 \specnorm{X} $, which shows inequality \eqref{ineq:induction3}.\\
Next, we note that by Lemma \ref{lemma:anglecontrol} we have that
\begin{align*}
	&\specnorm{ V_{X^\perp}^T V_{U_{t+1} W_{t+1}} }   \\
\le &  \left(   1-    \frac{\mu}{4} \sigma_{\min}^2 \left( X \right)    \right) \specnorm{ V_{X^\perp}^T V_{U_t W_t} }  + 100   \mu  \specnorm{     \left( \Id -\mathcal{ A}^*  \mathcal{A}   \right) \left( XX^T - U_tU_t^T \right)   }  +500   \mu^2  \specnorm{XX^T -U_tU_t^T }^2\\
\overleq{(a)} &  \left(   1-    \frac{\mu}{4} \sigma_{\min}^2 \left( X \right)    \right) \specnorm{ V_{X^\perp}^T V_{U_t W_t} }  + 4000    c_1 \mu  \kappa^{-2} \sigma_{\min }^2 \left(X\right)   +50000   \mu^2  \specnorm{X }^4\\
\overleq{(b)} &  \left(   1-    \frac{\mu}{4} \sigma_{\min}^2 \left( X \right)    \right) \specnorm{ V_{X^\perp}^T V_{U_t W_t} }  + 4000   c_1  \mu \kappa^{-2} \sigma_{\min }^2 \left(X\right)   +50000    c_1   \mu  \kappa^{-2} \sigma_{\min}^2 \left(X\right)\\
\overleq{(c)} &  \left(   1-    \frac{\mu}{4} \sigma_{\min}^2 \left( X \right)    \right) c_2 \kappa^{-2}  + 4000  c_1 \mu \kappa^{-2} \sigma_{\min }^2 \left(X\right)   +50000   c_1 \mu  \kappa^{-2} \sigma_{\min}^2 \left(X\right),
\end{align*}
where in inequality $(a)$ we used the induction hypothesis \eqref{ineq:induction3} as well as \eqref{ineq:intern335}. Inequality $(b)$ follows from  inequality \eqref{ineq:intern335} and our assumption on the step size $\mu $. In inequality $(c)$ we used the induction assumption $ \specnorm{ V_{X^\perp}^T V_{U_t W_t} } \le c_2 \kappa^{-2} $.
By choosing the constant $c_1 >0$ small enough, this implies inequality \eqref{ineq:induction4} and, hence, finishes the induction step.

Note that from the definition of $t_1$ and from inequality \eqref{ineq:induction1} we obtain inequality \eqref{ineq:induction101}.
Hence, we obtain that
\begin{align}
\specnorm{U_{t_1}W_{t_1  ,\perp}} &\overleq{(a)}  2 \left(  1  +  80 \mu   c_2   \sigma_{\min }^2 \left(X\right)     \right)^{t_1-t_{\star}} \gamma  \nonumber \\
 &\overleq{(b)}  2 \left(  \sqrt{\frac{5}{2}}  \cdot \frac{ \sigma_{\min} \left( X \right) }{\gamma }  \right)^{ 1280 c_2  }     \gamma  \nonumber\\
&\overleq{(c)}  2 \left(  \sqrt{\frac{5}{2}}  \cdot  \frac{\sigma_{\min } \left(X\right) }{\gamma}\right)^{ 1/64    }   \gamma  \nonumber\\
 &\le  3   \sigma_{\min } \left(X\right)^{1/64} \gamma^{63/64} \label{ineq:intern2222} \\
 &\overleq{(d)}  3   \sigma_{\min } \left(X\right)^{1/8} \gamma^{7/8}, \label{ineq:intern2223}
\end{align}
where inequality $(a)$ follows from inequality \eqref{ineq:induction2} and inequality $(b)$ follows from \eqref{ineq:induction101}.
Inequality $(c)$ follows from choosing the absolute constant $c_2>0$ small enough.
Inequality $(d)$ follows from the fact that $ \gamma \le \sigma_{\min} \left(X\right)  $.
This finishes the proof of the second phase.\\

\noindent \textbf{Phase III: }
In the third phase, we analyse the refinement of the signal $\UUt$. For that, we define
\begin{align}
	\hat{t}_1 := & t_1 + \Big\lfloor  \frac{300}{\mu \sigma_{\min } \left(X\right)^2}  \ln \left(  \frac{5}{18}  \kappa^{1/4}  \sqrt{\frac{r_{\star}}{\min \left\{ r;n \right\}      -r_{\star}}}  \frac{  \specnorm{X}^{7/4}}{\gamma^{7/4}}                 \right)   \Big\rfloor, \label{ineq:tdefinition1}\\
	\hat{t}_2 :=&  \min \left\{   t: \  \left(  \sqrt{ \min \left\{ r;n\right\} - r_{\star}  } +1 \right) \fronorm{   \UUt \Wtperp \Wtperp^T \UUt^T  }  \ge \fronorm{ XX^T - \UUt  \UUt^T } \text{ and }  t\ge t_1      \right\},\label{ineq:tdefinition2}  \\
	\hat{t} \ :=& \min \left\{ \hat{t}_1; \hat{t}_2 \right\}. \label{ineq:tdefinition}
\end{align}
Similar as in Phase II, we are going to show inductively that the following inequalities are fulfilled for $ t_1 \le  t \le \hat{t}  $:
\begin{align}
\sigma_{\min} \left(  U_t W_t\right) &\ge \sigma_{\min} \left(  V_X^T U_t\right)   \ge   \frac{\sigma_{\min } \left(X\right)}{\sqrt{10}}  ,    \label{ineq:induction5}  \\
\specnorm{U_{t}W_{t,\perp}} &\le \left(1+ 80 \mu   c_2   \sigma_{\min }^2 \left(X\right)    \right)^{t- t_1} \specnorm{U_{t_1} W_{t_1,\perp}},\label{ineq:induction6} \\
\specnorm{U_{t}} &\le 3\specnorm{X},  \label{ineq:induction7}\\
\specnorm{ V_{X^\perp}^T V_{U_{t}W_{t}}   } &\le c_2\kappa^{-2} , \label{ineq:induction8} \\
\fronorm{V_{X}^T  \left( XX^T -U_{t}U_{t}^T \right) } &\le  10 \sqrt{r_{\star}} \left(  1 - \frac{\mu}{400}  \sigma_{\min}^2 \left(  X\right) \right)^{t- t_1}   \specnorm{X}^2 . \label{ineq:induction9}
\end{align}
For $t=t_1$ we note  the inequalities \eqref{ineq:induction5}, \eqref{ineq:induction7}, and \eqref{ineq:induction8} follow from the results in Phase 1.
Inequality \eqref{ineq:induction6} follows directly from setting $t=t_1$.
Moreover, for $t=t_1$,  inequality \eqref{ineq:induction9} follows from the observation that
\begin{align*}
\fronorm{  V_{X}^T    \left( XX^T -U_{t_1}U_{t_1}^T \right) } &= \fronorm{  V_{X}^T    \left( XX^T -U_{t_1} W_{t_1} W_{t_1}^T U_{t_1}^T \right) } \\
&\le \fronorm{XX^T} + \fronorm{U_{t_1} W_{t_1} W_{t_1}^T U_{t_1}^T}  \\
&\le  \sqrt{r_{\star}}  \left(   \specnorm{XX^T} + \specnorm{U_{t_1} W_{t_1} W_{t_1}^T U_{t_1}^T}     \right)\\
&\le 10 \sqrt{r_{\star}} \specnorm{X}^2,
\end{align*}
where we have used that $\specnorm{U_{t_1} W_{t_1}} \le \specnorm{U_{t_1} } \le  3\specnorm{X} $ by induction assumption \eqref{ineq:induction7}.

For the induction step from $t$ to $t+1$ (with $ t < \hat{t} $), we note first that with similar arguments as in Phase 1 we can show that
\begin{align*}
&\specnorm{\left(   \mathcal{A}^* \mathcal{A} -\Id \right) \left(XX^T -U_tU_t^T\right)  }  \\
\le &  10\delta \sqrt{r_{\star}}  \specnorm{X}^2 + \delta  \left(  \min \left\{ r;n \right\}  -r_{\star} \right)  \specnorm{ U_t  W_{t,\perp}}^2 \\
\overleq{(a)} &  10 c_1  \kappa^{-2}    \sigma_{\min } \left(X\right)^2 +  6 \delta \left( \min \left\{ r;n \right\}   -r_{\star} \right) \left( 1  +  80 \mu   c_2   \sigma_{\min } \left(X\right)^2     \right)^{2(t-t_1)}   \sigma_{\min } \left(X\right)^{1/4} \gamma^{7/4} \\
\overleq{(b)} &  10c_1    \kappa^{-2}   \sigma_{\min } \left(X\right)^2 +  6  \delta  \left( \min \left\{ r;n \right\}       -r_{\star} \right) \left(  \frac{5}{18}  \kappa^{1/4}  \sqrt{\frac{r_{\star}}{\min \left\{ r;n \right\}     -r_{\star}}}     \frac{  \specnorm{X}^{7/4}}{\gamma^{7/4}}      \right)^{O \left(c_2\right)}   \sigma_{\min } \left(X\right)^{1/4} \gamma^{7/4} \\
 \overleq{(c)}&  40 c_1 \kappa^{-2}   \sigma_{\min } \left(X\right)^2 ,
\end{align*}
where inequality $(a)$ follows from \eqref{ineq:induction6}.
Inequality $(b)$ follows from \eqref{ineq:tdefinition} as well as the elementary inequality $ \ln \left(1+x\right) \le x $.
Inequality $(c)$ follows from the assumptions $ \gamma \le c_2 \frac{\sigma_{ \min  } \left(X\right)}{ \min \left\{r;n \right\}  \kappa^2 } $ and $\delta\le \frac{c_1}{\kappa^4 \sqrt{r_{\star}}}$.
This puts us in a position to apply our technical lemmas. 
We note that by Lemma \ref{lemma:Vxgrowth} we have that
\begin{align*}
  \sigma_{\min} \left( U_{t+1} W_{t+1} \right) &\ge \sigma_{\min} \left( V_X^T U_{t+1} \right) \ge   \sigma_{\min} \left( V_X^T U_{t+1} \WWt \right)  \\ 
  &\ge    \sigma_{\min} \left( V_X^T U_t \right) \bracing{= (\ast)}{  \left( 1 +  \frac{1}{4} \mu \sigma_{\min} \left(X\right)^2   -  \mu  \sigma_{\min} \left( V_X^T U_t \right)^2  \right) }. 
\end{align*}
Note that for $ \sigma_{\min} \left(  V_X^T U_t \right) \le \frac{1}{2} \sigma_{\min } \left(X\right) $ it holds that $ (\ast)  \ge 1 $ and thus it follows that \eqref{ineq:induction5} holds for $t+1$ in this case. In the case of  $ \frac{1}{2} \sigma_{\min } \left(X\right)  \le  \sigma_{\min} \left(  V_X^T U_t \right) $ we obtain that 
\begin{equation*}
\left( \ast  \right) \ge 1  -  \mu  \sigma_{\min} \left( V_X^T U_t \right)^2  \overgeq{(a)} 1- 9\mu \specnorm{X}^2 \overgeq{(b)} 4/5,
\end{equation*}
where in inequality $(a)$ we used the induction hypothesis \eqref{ineq:induction3} and in inequality $(b)$  we used the assumption $\mu \le c_1  \kappa^{-2} \specnorm{X}^{-2}$. Hence, we have shown that also in this case  the inequality \eqref{ineq:induction5} holds for $t+1$.
Note that the previous inequality chain also implies that $ V_X^T U_{t+1} W_t $ is invertible. Hence, in a similar way as in Phase II for inequality \eqref{ineq:induction2} we can verify that \eqref{ineq:induction6} holds for $t+1$.
Moreover, note that from Lemma \ref{lemma:Ucontrol}, induction assumption \eqref{ineq:induction3}, and the assumption on the step size $\mu$ it follows that $\specnorm{U_{t+1}} \le 3 \specnorm{X}$. Moreover, inequality \eqref{ineq:induction8} can be shown analogously as in Phase $1$.

Next, we want to apply Lemma \ref{lemma:localconvergence}.
For that, we compute that
\begin{align*}
&\fronorm{ V_X^T \left( \Id- \mathcal{A}^* \mathcal{A} \right) \left(XX^T -\UUt \UUt^T \right)  } \\
\le &  \fronorm{  V_X^T \left( \Id- \mathcal{A}^* \mathcal{A} \right) \left(XX^T -\UUt \WWt \WWt^T \UUt^T \right)   } + \fronorm{  V_X^T \left( \Id- \mathcal{A}^* \mathcal{A} \right) \left( \UUt \Wtperp \Wtperp^T \UUt^T \right)   }.
\end{align*}
Now choose symmetric matrices $Z_1, Z_2$ of rank at most $r_{\star}$ such that $XX^T -\UUt \WWt \WWt^T=Z_1+Z_2$ and $\innerproduct{Z_1,Z_2} =0 $.
It follows that 
\begin{align*}
	\fronorm{  V_X^T \left( \Id- \mathcal{A}^* \mathcal{A} \right) \left(XX^T -\UUt \WWt \WWt^T \UUt^T \right)   }
	&\le  \fronorm{  V_X^T \left( \Id- \mathcal{A}^* \mathcal{A} \right) \left( Z_1 \right)   } + \fronorm{  V_X^T \left( \Id- \mathcal{A}^* \mathcal{A} \right) \left( Z_2 \right)   }	\\
	&\overleq{(a)} \delta \left( \fronorm{Z_1} + \fronorm{Z_2} \right) \overleq{(b)} \sqrt{2} \delta \fronorm{ XX^T -\UUt \WWt \WWt^T \UUt^T },
\end{align*}
where inequality $(a)$ is a consequence of the restricted isometry property of order $2r_{\star}+1$ (see inequality \eqref{ineq:RIPintern1} in Lemma \ref{lemma:RIP1}).\footnote{Note that if we would have assumed that $\mathcal{A}$ satisfies the restricted isometry property of order $3r_{\star}$ the decomposition of $XX^T -\UUt \WWt \WWt^T \UUt^T= Z_1+Z_2$ would not have been necessary.
Instead, we could have directly applied inequality \eqref{ineq:RIPintern1} in Lemma \ref{lemma:RIP1} with $Z= XX^T -\UUt \WWt \WWt^T \UUt^T$.}
Inequality $(b)$ follows from $ \innerproduct{Z_1,Z_2} =0 $.
Next, denote by $ \UUt \Wtperp \Wtperp^T \UUt^T = \sum_{i=1}^n \lambda_i v_i v_i^T$ the eigendecomposition of $\UUt \Wtperp \Wtperp^T \UUt^T$. Again, using inequality \eqref{ineq:RIPintern1} in Lemma \ref{lemma:RIP1} it follows that
\begin{align*}
	\fronorm{  V_X^T \left( \Id- \mathcal{A}^* \mathcal{A} \right) \left( \UUt \Wtperp \Wtperp^T \UUt^T \right)   } 
	\le & \sum_{i=1}^n \lambda_i \fronorm{  V_X^T \left( \Id- \mathcal{A}^* \mathcal{A} \right) \left( v_i v_i^T \right)   }\\
	\le & \delta \sum_{i=1}^n \lambda_i \fronorm{v_i v_i^T}
	= \delta \nucnorm{ \UUt \Wtperp \Wtperp^T \UUt^T}.	
\end{align*}
Combining the last three inequality chains, we obtain that
\begin{align*}
\fronorm{ V_X^T \left( \Id- \mathcal{A}^* \mathcal{A} \right) \left(XX^T -\UUt \UUt^T \right)  } 
\le  & \delta \fronorm{ XX^T -\UUt \WWt \WWt^T \UUt^T } + \delta \nucnorm{   \UUt \Wtperp \Wtperp^T \UUt^T  }\\
\overleq{(b)} &  \delta \fronorm{ XX^T -\UUt \WWt \WWt^T \UUt^T } + \delta \sqrt{ \min \left\{ r;n\right\} - r_{\star}  } \fronorm{   \UUt \Wtperp \Wtperp^T \UUt^T  } \\
\le &  \delta \fronorm{ XX^T -\UUt   \UUt^T } + \delta \left( \sqrt{ \min \left\{ r;n\right\} - r_{\star}  } +1 \right) \fronorm{   \UUt \Wtperp \Wtperp^T \UUt^T  }\\
\overleq{(c)} &  2\delta \fronorm{ XX^T -\UUt   \UUt^T } \overleq{(d)} \frac{c}{\kappa^2} \fronorm{ XX^T -\UUt   \UUt^T }. 
\end{align*}
Inequality $(a)$ follows from the restricted isometry property.
Inequality $(b)$ follows from the fact that $  \UUt \Wtperp  $ has rank at most $r-r_{\star}$ and inequality $(c)$ follows from $t < \hat{t}_1$ (see \eqref{ineq:tdefinition2}).
In inequality $(d)$ we used the assumption that $ \delta \le \frac{c_1}{ \sqrt{r_{\star}} \kappa^4} \le  \frac{c}{ 2 \kappa^2}  $ with $c>0$ being the constant in Lemma \ref{lemma:localconvergence}.
In an analogous way we can establish the inequalities
\begin{align*}
	\fronorm{ V_{\Ut \WWt}^T \left( \Id- \mathcal{A}^* \mathcal{A} \right) \left(XX^T -\UUt \UUt^T \right)  } &\le  \frac{c}{\kappa^2} \fronorm{ XX^T -\UUt   \UUt^T },\\	
	\specnorm{  \left( \Id- \mathcal{A}^* \mathcal{A} \right) \left(XX^T -\UUt \UUt^T \right)  } &\le  \frac{c}{\kappa^2} \fronorm{ XX^T -\UUt   \UUt^T }.
\end{align*}
This shows that inequality \eqref{ineq:intern773} is fulfilled (with $\Vnorm{\cdot} $ being the Frobenius norm $\fronorm{\cdot}$). 
 Hence, we can apply Lemma \ref{lemma:localconvergence} and we obtain that 
\begin{align*}
&\fronorm{V_{X}^T  \left( XX^T -U_{t+1}U_{t+1}^T \right) }\\
\le &   \left(  1 - \frac{\mu}{200}  \sigma_{\min} \left(  X\right)^2 \right)\fronorm{  V_{X}^T    \left( XX^T -U_tU_t^T \right) } +  \mu \frac{\sigma_{\min} \left(X\right)^2 }{100}  \fronorm{  U_tW_{t, \perp}W_{t, \perp}^T U_t^T  } \\
\le &  10 \sqrt{r_{\star}}  \left(  1 - \frac{\mu}{200}  \sigma_{\min} \left(  X\right)^2 \right)  \left(  1 - \frac{\mu}{400}  \sigma_{\min} \left(  X\right)^2 \right)^{t-t_1} \specnorm{X}^2  +   \mu \frac{\sigma_{\min} \left(X\right)^2 }{100}  \fronorm{  U_tW_{t, \perp}W_{t, \perp}^T U_t^T  },
\end{align*}
where in the last inequality we used the induction assumption \eqref{ineq:induction9}.
We note that this shows that \eqref{ineq:induction9} also holds for $t+1$, if we can show that
\begin{equation}\label{intern:ineq6789}
\fronorm{  U_tW_{t, \perp}W_{t, \perp}^T U_t^T  } \le  \frac{5}{2} \sqrt{r_{\star}} \left(  1 - \frac{\mu}{400}  \sigma_{\min} \left(  X\right)^2 \right)^{t-t_1} \specnorm{X}^2 .
\end{equation}
In order to prove this, we note that
\begin{align*}
\fronorm{  U_tW_{t, \perp}W_{t, \perp}^T U_t^T  } &\le  \sqrt{  \min \left\{ r;n \right\}    -r_{\star}} \specnorm{  U_{t}W_{t,\perp}}^2 \\
&\le  9 \sqrt{ \min \left\{ r;n \right\}       -r_{\star}} \left(1+ 80  \mu c_2  \sigma_{\min } \left(X\right)^2 \right)^{2(t-t_1)}    \sigma_{\min } \left(X\right)^{1/4} \gamma^{7/4},
\end{align*}
where in the last inequality we used \eqref{ineq:intern2222} and \eqref{ineq:induction6}. Hence, for $c_2>0 $ small enough, inequality \eqref{intern:ineq6789} is implied by
\begin{equation*}
\frac{18}{5} \sqrt{\frac{\min \left\{ r;n \right\}   -r_{\star}}{r_{\star}}} \sigma_{\min } \left(X\right)^{1/4} \gamma^{7/4} \le    \left(  1 - \frac{\mu}{350}  \sigma_{\min} \left(  X\right)^2 \right)^{t- t_1}   \specnorm{X}^2.
\end{equation*}
By rearranging terms and using the elementary inequality $    \ln \left(1+x\right) \ge  \frac{x}{1-x}$, we see that this in turn is implied by
\begin{align*}
t-t_1 &\le \frac{300}{\mu \sigma_{\min } \left(X\right)^2}  \ln \left(  \frac{5}{18}   \sqrt{\frac{r_{\star}}{ \min \left\{ r;n \right\}    -r_{\star}}}  \cdot  \frac{\specnorm{X}^2}{\gamma^{7/4}  \sigma_{\min } \left(X\right)^{1/4}   }                \right) .
\end{align*}
Hence,  \eqref{ineq:tdefinition} shows \eqref{intern:ineq6789}, which shows  inequality \eqref{ineq:induction9} for $t+1$. This finishes the induction step.\\

\noindent \textbf{Conclusion: }
In order to finish the proof we distinguish two cases. Namely, in the first case we stop at $\hat{t}=\hat{t}_1$ and in the second case we stop at $ \hat{t} =\hat{t}_2 $.
First, we consider the case $ \hat{t} = \hat{t}_1$. Then we obtain that
\begin{align*}
\fronorm{U_{\hat{t}}U_{\hat{t}}^T - XX^T} &\overleq{(a)} 4 \fronorm{ V_X^T \left(  XX^T -U_{\hat{t}}U_{\hat{t}}^T  \right)} + \fronorm{  U_{\hat{t}}W_{\hat{t},\perp}W_{\hat{t},\perp}^T U_{\hat{t}}^T  } \\
& \overset{(b)}{ \lesssim } \sqrt{r_{\star}}  \left(  1 - \frac{\mu}{400}  \sigma_{\min} \left(  X\right)^2 \right)^{\hat{t}- t_1}   \specnorm{X}^2\\
& \overset{(c)}{ \lesssim }     \sqrt{r_{\star}} \left(  \frac{5}{18}  \kappa^{1/4}  \sqrt{\frac{r_{\star}}{\min \left\{ r;n \right\}    -r_{\star}}}  \frac{  \specnorm{X}^{7/4}}{\gamma^{7/4}}                 \right)^{-3/4}    \specnorm{X}^2       \\
& \lesssim  r_{\star}^{1/8}  \kappa^{-3/16}  \left( \min \left\{ r;n \right\}  -r_{\star}\right)^{3/8}     \gamma^{21/16}    \specnorm{X}^{11/16},
\end{align*}
where inequality $(a)$ follows from Lemma \ref{lemma:technicallemma8}.
Inequality $(b)$ follows from \eqref{ineq:induction9} and inequality \eqref{intern:ineq6789}.
Inequality $(c)$ is due to $\hat{t}=\hat{t}_1$ and the definition of $\hat{t}_1$.
This proves inequality \eqref{ineq:iterationsbound} in the case that $ \hat{t}= \hat{t}_1$.
Next, we consider the second case, $\hat{t}=\hat{t}_2$. We calculate that
\begin{align*}
& \fronorm{U_{\hat{t}}U_{\hat{t}}^T - XX^T}\\
\overleq{(a)} & 4 \fronorm{ V_X^T \left(  XX^T -U_{\hat{t}}U_{\hat{t}}^T  \right)} + \fronorm{  U_{\hat{t}}W_{\hat{t},\perp}W_{\hat{t},\perp}^T U_{\hat{t}}^T  } \\
\le & 4 \fronorm{  XX^T -U_{\hat{t}}U_{\hat{t}}^T } + \fronorm{  U_{\hat{t}}W_{\hat{t},\perp}W_{\hat{t},\perp}^T U_{\hat{t}}^T  } \\
\overleq{(b)} & \left( 4 \sqrt{ \min \left\{ r;n \right\} - r_{\star} }  +5 \right)  \fronorm{  U_{\hat{t}}W_{\hat{t},\perp}W_{\hat{t},\perp}^T U_{\hat{t}}^T  }\\
\overleq{(c)} &  \sqrt{  \min \left\{ r;n \right\} -r_{\star} }  \left( 4 \sqrt{ \min \left\{ r;n \right\} - r_{\star} }  +5 \right)    \specnorm{  U_{\hat{t}}W_{\hat{t},\perp}W_{\hat{t},\perp}^T U_{\hat{t}}^T  }  \\
\overleq{(d)} & 9  \sqrt{ \min \left\{ r;n \right\} - r_{\star} }  \left( 4 \sqrt{ \min \left\{ r;n \right\} - r_{\star} }  +5 \right)   \left(1+ 80 \mu   c_2   \sigma_{\min }^2 \left(X\right)    \right)^{2(\hat{t}- t_1)}   \sigma_{\min } \left(X\right)^{1/64} \gamma^{127/64}\\
\overleq{(e)} & 9 \sqrt{ \min \left\{ r;n \right\} - r_{\star} }  \left( 4 \sqrt{ \min \left\{ r;n \right\} - r_{\star} }  +5 \right)   \left(  \frac{5}{18}  \kappa^{1/4}  \sqrt{\frac{r_{\star}}{\min \left\{ r;n \right\}      -r_{\star}}}  \frac{  \specnorm{X}^{7/4}}{\gamma^{7/4}}\right)^{O(c_2)}  \sigma_{\min } \left(X\right)^{1/64} \gamma^{127/64}\\
=   & 9  \frac{ \sqrt{ \min \left\{ r;n \right\} - r_{\star} }  \left( 4 \sqrt{ \min \left\{ r;n \right\} - r_{\star} }  +5 \right)  \gamma^{42/64}}{\kappa^{1/64}}  \left(  \frac{5}{18}  \kappa^{1/4}  \sqrt{\frac{r_{\star}}{\min \left\{ r;n \right\}      -r_{\star}}}  \frac{  \specnorm{X}^{7/4}}{\gamma^{7/4}}\right)^{O(c_2)}  \specnorm{X}^{1/64} \gamma^{85/64}\\
\overleq{(f)}   & 9 c_2^{21/32} \frac{ \sqrt{ \min \left\{ r;n \right\} - r_{\star} }  \left( 4 \sqrt{ \min \left\{ r;n \right\} - r_{\star} }  +5 \right)  }{ \min \left\{ r;n \right\}^{42/64} \kappa^{85/64}}  \left(  \frac{5}{18}  \kappa^{1/4}  \sqrt{\frac{r_{\star}}{\min \left\{ r;n \right\}      -r_{\star}}}  \frac{  \specnorm{X}^{7/4}}{\gamma^{7/4}}\right)^{O(c_2)}  \specnorm{X}^{21/32} \gamma^{85/64}\\
\lesssim   &  \frac{ \left( \min \left\{ r;n \right\} - r_{\star} \right)^{22/64} }{  \kappa^{85/64}}  \left(  \frac{5}{18}  \kappa^{1/4}  \sqrt{\frac{r_{\star}}{\min \left\{ r;n \right\}      -r_{\star}}}  \frac{  \specnorm{X}^{7/4}}{\gamma^{7/4}}\right)^{O(c_2)}  \specnorm{X}^{21/32} \gamma^{85/64}\\
 \overset{(g)}{\lesssim} &  r_{\star}^{1/8}  \kappa^{-3/16}  \left( \min \left\{ r;n \right\}  -r_{\star}\right)^{3/8}     \gamma^{21/16}    \specnorm{X}^{11/16},
\end{align*}
Inequality $(a)$ follows from Lemma \ref{lemma:technicallemma8} and inequality $(b)$ is due to the fact $\hat{t}=\hat{t}_2$ (see \eqref{ineq:tdefinition}).
Inequality $(c)$ is due to the fact that $  U_{\hat{t}} W_{\hat{t}} $ has rank at most $ r- r_{\star} $.
For inequality $(d)$ we used inequality \eqref{ineq:induction6} combined with inequality \eqref{ineq:intern2222} and inequality $(e)$ follows from the fact that $\hat{t}\le \hat{t}_1$ (see \eqref{ineq:tdefinition1}).
Inequality $(f)$ is due to the assumption $ \gamma \le c_2 \frac{\sigma_{\min} \left(X\right)}{ \min \left\{ r;n \right\}  \kappa^2} $.
Inequality $(g)$ follows from the fact that $ \gamma \le \specnorm{X} $.
This proves inequality \eqref{ineq:iterationsbound} in the case that $ \hat{t}= \hat{t}_2$.

In order to finish the proof we need to show \eqref{ineq:iterationsbound} follows from our definition of $\hat{t}$.
For that we note that we obtain from $\hat{t} \le \hat{t}_1 $ and the definition of $\hat{t}_1$ (see equation \eqref{ineq:tdefinition}) that
\begin{align*}
\hat{t}-t_1 &\le \frac{300}{\mu \sigma_{\min } \left(X\right)^2}  \ln \left(  \frac{5}{18}   \sqrt{\frac{r_{\star}}{ \min \left\{ r;n \right\}    -r_{\star}}}  \cdot  \frac{\specnorm{X}^2}{\gamma^{7/4}  \sigma_{\min } \left(X\right)^{1/4}   }                 \right) \\
&= \frac{300}{\mu \sigma_{\min } \left(X\right)^2}  \ln \left(  \frac{5}{18}  \kappa^{1/4}  \sqrt{\frac{r_{\star}}{ \min \left\{ r;n \right\} -r_{\star}}}  \frac{  \specnorm{X}^{7/4}}{\gamma^{7/4}}                 \right)\\
&\le  \frac{300}{\mu \sigma_{\min } \left(X\right)^2}  \ln \left(      \min \left\{ 1;\frac{ \kappa  r_{\star}}{ \min \left\{ r;n \right\} -r_{\star}}\right\}   \frac{  \specnorm{X}^{7/4}}{\gamma^{7/4}}                 \right)\\
&\lesssim  \frac{1}{\mu \sigma_{\min } \left(X\right)^2}  \ln \left(      \min \left\{ 1;\frac{ \kappa  r_{\star}}{ \min \left\{ r;n \right\} -r_{\star}}\right\}   \frac{  \specnorm{X}}{\gamma}                 \right).
\end{align*}
Combining this estimate with inequality \eqref{ineq:induction101} shows \eqref{ineq:iterationsbound}.
\end{proof}

%% file: main_results_proof.tex
\section{Proof of the main results}\label{sec:proofmainresults}

\subsection{Proof of Theorem \ref{mainresult:runequaln}}

\begin{proof}[Proof of Theorem \ref{mainresult:runequaln}]
	Set $\tilde{E}= \mathcal{ A}^* \mathcal{A} \left(XX^T\right)   -XX^T  $. From the spectral-to-spectral restricted isometry property, which follows from Lemma \ref{lemma:RIP1} as well as from our assumption on the restricted isometry property, it follows that
	\begin{equation}\label{ineq:intern1000}
	\specnorm{\tilde{E}} = \specnorm{  \left( \Id - \mathcal{ A}^* \mathcal{A}     \right)   \left(XX^T \right)     } \le c \kappa^{-4}  \specnorm{X}^2    =      c \kappa^{-2} \sigma_{\min } \left(X\right)^2             .
	\end{equation}
	In order to finish the proof we will distinguish two cases:\\
	
	\noindent \textbf{Case  $ r\ge 2r_{\star} $: }
	Due to \eqref{ineq:alphabound} and \eqref{ineq:intern1000}  we can apply Lemma \ref{lemma:spectralmain2:simplified}. Hence, with probability at least $1- O \left( \exp \left( -c r  \right) \right) $ after
	\begin{align*}
	{t_{\star}} \lesssim 	\frac{1}{\mu \sigma_{ \min }  \left(X\right)^2} \cdot  \ln \left(     2 \kappa^2  \sqrt{\frac{n}{\min \left\{ r;n \right\}}}    \right) 
	\end{align*}
	iterations we have that
	\begin{align}
	\specnorm{U_{t_{\star}}} & \le 3 \specnorm{X},    \label{ineq:specmain5:simplified} \\
	\sigma_{\min} \left( U_{t_{\star}} W_{t_{\star}}  \right)  &\ge    \frac{\alpha \beta}{4}, \label{ineq:specmain6:simplified} \\
	\specnorm{ U_{t_{\star}}  W_{t_{\star},\perp} }  &\le  \frac{ \kappa^{-2}}{8} \alpha \beta, \label{ineq:specmain7:simplified}\\
	\specnorm{ V_{X^\perp}^T  V_{U_{t_{\star}}  W_{t_{\star}}}   } &\le  c \kappa^{-2}\label{ineq:specmain8:simplified}
	\end{align}
	with $1 \lesssim  \beta \lesssim \frac{n \kappa^4  }{   \min \left\{ r;n \right\} } $.
	Our goal is to apply Theorem \ref{thm:localconvergence} with $ \gamma =\frac{\alpha \beta}{4} $. For that we need to check that
	\begin{equation}\label{ineq:inter3334}
  \frac{c_2 \sigma_{\min } \left(X\right)}{\min \left\{ r;n \right\} \kappa^2} \ge \gamma = \frac{\alpha \beta}{4} 
	\end{equation}	 
holds. Note that since
\begin{equation*}
\frac{ 4 c_2  \specnorm{X}  }{\min \left\{ r;n \right\} \kappa^3 \beta}  \gtrsim  \frac{  \specnorm{X} }{\kappa^7 n }   \gtrsim \alpha
\end{equation*}
condition \eqref{ineq:inter3334} is fulfilled, when the constant in \eqref{ineq:alphabound} is chosen sufficiently small. Hence, by  Theorem \ref{thm:localconvergence}  after 
	\begin{align*}
	\hat{t} - t_{\star} \lesssim    \frac{1}{\mu \sigma_{\min } \left(X\right)^2 }\ln \left( \max \left\{   1; \frac{\kappa r_{\star}}{\min \left\{ r;n \right\}-r_{\star}} \right\}   \frac{4 \Vert X \Vert}{\alpha \beta} \right)  
	\end{align*}
	iterations it holds that
	\begin{align*}
	\frac{\fronorm{U_{\hat{t}} U_{\hat{t}}^{T} -XX^T } }{\specnorm{X}^2} &\lesssim    \frac{ r_{\star}^{1/8}     \left(  \min \left\{ r;n \right\}-r_{\star}  \right)^{3/8} }{\kappa^{3/16}  }  \cdot   \frac{   \gamma^{21/16}  }{  \specnorm{X}^{21/16}}             \\
	 &\lesssim    \frac{ r_{\star}^{1/8}     \left(  \min \left\{ r;n \right\}-r_{\star}  \right)^{3/8} }{\kappa^{3/16}  }   \frac{   \left( \alpha \beta \right)^{21/16}  }{  \specnorm{X}^{21/16}}\\
	&\lesssim  \frac{  \kappa^{81/16}  n^{21/16}    r_{\star}^{1/8}     \left(  \min \left\{ r;n \right\}-r_{\star}  \right)^{3/8} }{\left(  \min \left\{ r;n \right\}\right)^{21/16}} \cdot        \frac{   \alpha^{21/16}  }{  \specnorm{X}^{21/16}}\\
	&\le   \frac{  n^{21/16}  \kappa^{81/16}    r_{\star}^{1/8}  }{   \left( \min \left\{ r;n \right\} \right)^{15/16}  }   \cdot   \frac{   \alpha^{21/16}  }{  \specnorm{X}^{21/16}}.
	\end{align*}
	Note that for the total amount of iterations we have that
	\begin{align*}
	\hat{t} &\lesssim  	\frac{1}{\mu \sigma_{ \min }  \left(X\right)^2} \left(   \ln \left(     2 \kappa^2  \sqrt{\frac{n}{\min \left\{ r;n  \right\}}}    \right)  +    \ln \left( \max \left\{   1; \frac{\kappa r_{\star}}{\min \left\{ r;n  \right\}-r_{\star}} \right\}   \frac{4 \Vert X \Vert}{\alpha \beta} \right)  \right)  \\
	&=  \frac{1}{\mu \sigma_{ \min }  \left(X\right)^2}   \ln \left(  8 \kappa^3 \sqrt{\frac{n}{\min \left\{ r;n  \right\}}} \cdot \max \left\{   1; \frac{\kappa r_{\star}}{\min \left\{ r;n  \right\}-r_{\star}} \right\} \cdot \frac{ \Vert X \Vert}{\alpha \beta} \right) \\
	&\overleq{(b)} \frac{1}{\mu \sigma_{ \min }  \left(X\right)^2}   \ln \left( C_1 \kappa^3 \sqrt{\frac{n}{\min \left\{ r;n  \right\}}} \cdot \max \left\{   1; \frac{\kappa r_{\star}}{\min \left\{ r;n  \right\}-r_{\star}} \right\}  \cdot \frac{ \Vert X \Vert}{\alpha } \right) \\
	&\lesssim \frac{1}{\mu \sigma_{ \min }  \left(X\right)^2}   \ln \left(  \frac{C_1 \kappa n}{\min \left\{ r;n  \right\}} \cdot \max \left\{   1; \frac{\kappa r_{\star}}{\min \left\{ r;n  \right\}-r_{\star}} \right\}  \cdot \frac{ \Vert X \Vert}{\alpha } \right),
	\end{align*}
	where in inequality $(b)$ we have used $ \beta \gtrsim 1 $ and chosen the constant $C_1 >0$ large enough. This finishes the proof of the first part.\\
	
	\noindent \textbf{Case  $r_{\star} <  r <  2r_{\star} $: }	
	As in the first case, we can apply Lemma \ref{lemma:spectralmain2:simplified}. Hence, with probability at least $1- \left(C\varepsilon\right)^{r-r_{\star}+1} + O \left(  \exp \left(-cr\right) \right)$ after
	\begin{align*}
	{t_{\star}} \lesssim \frac{1}{\mu \sigma_{ \min  }  \left(X\right)^2} \cdot  \ln \left(    \frac{  2\kappa^2\sqrt{rn} }{\varepsilon}   \right) 
	\end{align*}
	iterations the inequalities \eqref{ineq:specmain1:simplified}, \eqref{ineq:specmain2:simplified}, \eqref{ineq:specmain3:simplified}, and \eqref{ineq:specmain4:simplified}  hold with  $   \frac{\varepsilon}{r} \lesssim   \beta \lesssim  \frac{ \kappa^4 n }{  \varepsilon} $.
	Again, we want to apply Theorem \ref{thm:localconvergence} with $ \gamma =\frac{\alpha \beta}{4} $. For that we need to check that
	\begin{equation}\label{ineq:inter3335} 
	 \frac{c_2 \sigma_{\min } \left(X\right)}{\min \left\{ r;n \right\} \kappa^2} \ge \gamma = \frac{\alpha \beta}{4} 
	\end{equation}	 
	holds. Note that since
	\begin{equation*}
	\frac{ 4 c_2  \specnorm{X}  }{\min \left\{ r;n \right\} \kappa^3 \beta}  \gtrsim  \frac{ \varepsilon  r  \specnorm{X}}{ \min \left\{ r;n \right\} n  \kappa^7  } =     \frac{ \varepsilon    \specnorm{X}}{ n  \kappa^7  }     \gtrsim \alpha
	\end{equation*}
	condition \eqref{ineq:inter3335} holds true, when the constant in \eqref{ineq:alphabound2} is chosen sufficiently small. Hence, by Theorem \ref{thm:localconvergence}  after 
	\begin{align*}
	\hat{t} - t_{\star}  \lesssim    \frac{1}{\mu \sigma_{\min } \left(X\right)^2 }\ln \left( \max \left\{   1; \frac{\kappa r_{\star}}{r-r_{\star}} \right\}   \frac{4 \Vert X \Vert}{\alpha \beta} \right)  
	\end{align*}
	iterations it holds that
	\begin{align*}
	\frac{\fronorm{ U_{\hat{t}} U_{\hat{t}}^{T} -XX^T } }{\specnorm{X}^2} &\lesssim \frac{ r_{\star}^{1/8}     \left(  r-r_{\star}  \right)^{3/8} }{\kappa^{3/16}  }  \cdot   \frac{   \gamma^{21/16}  }{  \specnorm{X}^{21/16}} \\
	 &\lesssim  \frac{ r_{\star}^{1/8}     \left(  r-r_{\star}  \right)^{3/8} }{\kappa^{3/16}  }  \cdot   \frac{   \left( \alpha \beta  \right)^{21/16}  }{  \specnorm{X}^{21/16}}       \\
	&\le     r_{\star}^{1/8}     \left(  r-r_{\star}  \right)^{3/8} \kappa^{81/16}   \left( \frac{n}{\varepsilon }    \cdot \frac{ \alpha }{\specnorm{X}}   \right)^{21/16} .
	\end{align*}
	Note that for the total amount of iterations we have that
	\begin{align*}
	\hat{t} &\lesssim  	\frac{1}{\mu \sigma_{ \min }  \left(X\right)^2} \left(  \ln \left(    \frac{  2\kappa^2\sqrt{rn} }{\varepsilon}   \right) +    \ln \left( \max \left\{   1; \frac{\kappa r_{\star}}{r-r_{\star}} \right\}   \frac{4 \Vert X \Vert}{\alpha \beta} \right)  \right)  \\
	&\overset{(a)}{=}  	\frac{1}{\mu \sigma_{ \min }  \left(X\right)^2} \ln \left(           \frac{8 \kappa^3 r_{\star} \sqrt{rn} }{\varepsilon \left( r-r_{\star} \right)}   \cdot \frac{\specnorm{X}}{\alpha \beta}            \right)\\
	& \overleq{(b)}  \frac{1}{\mu \sigma_{ \min }  \left(X\right)^2}  \ln \left(     \frac{ C_2 \kappa^3 r_{\star} r \sqrt{rn}  }{  \varepsilon^2  \left(   r-r_{\star}  \right)   } \cdot  \frac{\specnorm{X}}{\alpha}    \right)\\
		&\lesssim  \frac{1}{\mu \sigma_{ \min }  \left(X\right)^2}  \ln \left(     \frac{ C_2 \kappa  n^2  }{  \varepsilon^2  \left(   r-r_{\star}  \right)   } \cdot  \frac{\specnorm{X}}{\alpha}    \right),
	\end{align*}
	where in equality $(a)$ we have used $  \frac{r_{\star}}{r-r_{\star}}  \ge  1   $, which follows from $ r_{\star} < r  \le 2r_{\star} $. Inequality $(b)$ follows from $  \frac{\varepsilon}{r} \lesssim   \beta  $ as well as from choosing $C_2>0$ large enough. This finishes the proof.

\end{proof}

\subsection{Proof of Theorem \ref{mainresult:requalrstar}}

\begin{proof}[Proof of Theorem \ref{mainresult:requalrstar}]
	As in the proof of the second part of Theorem \ref{mainresult:runequaln} we can show that with probability at least $1- C\varepsilon + O \left(  \exp \left(-cr_{\star}\right) \right)$ after
	\begin{align*}
	{t_{\star}} \lesssim \frac{1}{\mu \sigma_{\min  }  \left(X\right)^2} \cdot  \ln \left(    \frac{  2\kappa^2\sqrt{n} }{\varepsilon}   \right) 
	\end{align*}
	iterations the inequalities \eqref{ineq:specmain1:simplified}, \eqref{ineq:specmain2:simplified}, \eqref{ineq:specmain3:simplified}, and \eqref{ineq:specmain4:simplified}  hold with  $   \frac{\varepsilon}{ r_{\star}    } \lesssim   \beta \lesssim  \frac{ \kappa^4 n }{  \varepsilon   } $.
	Now define the matrix $ \widehat{U_{t_{\star}}} $by adding a zero column to $U_{t_{\star}}$, i.e.,  
	\begin{equation*}
	\widehat{U}_{t_{\star}} =   \begin{pmatrix}
	U_{t_{\star}}&  0 
	\end{pmatrix}   \in \mathbb{R}^{n \times (r_{\star}+1)}. 
	\end{equation*}
	Clearly, we can run gradient descent on $\widehat{U}_{t_{\star}} $ instead of $U_{t_{\star}}$ with the same step size, which gives us a sequence $ \widehat{U}_{t_{\star}}, \widehat{U}_{t_{\star}+1}, \widehat{U}_{t_{\star}+2},\ldots   $ to which we can apply Theorem \ref{thm:localconvergence} with $ \gamma =\frac{\alpha \beta}{4} $ and $r=r_{\star}+1$. However, note that the last column always stays zero, which means that the results of this theorem also apply to $U_{t_{\star}},U_{t_{\star}+1},U_{t_{\star}+2},\ldots$. 	 Hence, after
	\begin{align*}
	\hat{t} - t_{\star}  \lesssim    \frac{1}{\mu \sigma_{\min } \left(X\right)^2 }\ln \left(  4 \kappa r_{\star}  \frac{\Vert X \Vert}{\alpha \beta} \right)   \lesssim   \frac{1}{\mu \sigma_{\min } \left(X\right)^2 }\ln \left(  4 \kappa r^{2}_{\star}  \frac{\Vert X \Vert}{\alpha \varepsilon} \right)   
	\end{align*}
	iterations it holds that
	\begin{align*}
	\frac{\fronorm{ U_{\hat{t}} U_{\hat{t}}^{T} -XX^T } }{\specnorm{X}^2} & \lesssim \frac{ r_{\star}^{1/8}  }{\kappa^{3/16}  }  \cdot   \frac{   \left( \alpha \beta  \right)^{21/16}  }{  \specnorm{X}^{21/16}} \\
	&\lesssim   \frac{ r_{\star}^{1/8}  }{\kappa^{3/16}  }   \left(  \frac{\kappa^4 n}{\varepsilon }    \cdot \frac{ \alpha }{\specnorm{X}}   \right)^{21/16}\\
	&=    r_{\star}^{1/8} \kappa^{81/16}  \left( \frac{n}{\varepsilon }    \cdot \frac{ \alpha }{\specnorm{X}}   \right)^{21/16} .
	\end{align*}
	Note that for the total amount of iterations we have that
	\begin{align*}
	\hat{t} &\lesssim  	\frac{1}{\mu \sigma_{ \min }  \left(X\right)^2} \left(  \ln \left(    \frac{  2\kappa^2\sqrt{n} }{\varepsilon}   \right) +    \ln \left(  4 \kappa r^{2}_{\star}  \frac{\Vert X \Vert}{\alpha \varepsilon} \right)  \right)  \\
	& =	\frac{1}{\mu \sigma_{ \min }  \left(X\right)^2}  \ln \left(      \frac{8 \kappa^3 r_{\star}^{2} \sqrt{n}  \Vert X \Vert}{\alpha  \varepsilon^2  }   \right) \\
	&\le  \frac{1}{\mu \sigma_{ \min }  \left(X\right)^2}  \ln \left(      \frac{ 8 \kappa^3  n^3  }{  \varepsilon^2    } \cdot  \frac{\specnorm{X}}{\alpha}    \right).
	\end{align*}
	This finishes the proof.

\end{proof}

\subsection{Proof of Theorem \ref{mainresult:requalsn}}

We start by noting that in the special case $r=n$, the required assumptions for Theorem \ref{thm:localconvergence} are already fulfilled at the initialization $t_0=0$. This means that in this special case we do not need to analyze the spectral phase. This is shown by the following lemma.
\begin{lemma}\label{lemma:spectralmain1}
	Assume that $r=n$ and let $U_0 = \alpha U$, where $U \in \mathbb{R}^{n \times n}$ is an orthonormal matrix. Then it holds that 
	\begin{align*}
	\specnorm{ V_{X^\perp}^T  V_{U_0 W_0}   } &=0,\\
	\sigma_{\min} \left( U_0 W_0  \right)  &= \alpha,\\
	\specnorm{ U_0   } &=\alpha.
	\end{align*}
\end{lemma}

\begin{proof}
Note that $V_X^T U \in \mathbb{R}^{r_{\star} \times n}$ is an isometric embedding. Hence, a feasible choice for $W_0$ is given by $W_0= U^T V_X$, which implies that 
\begin{align*}
U_0W_0 = \alpha U U^T V_X = \alpha V_X   . 
\end{align*}
It follows that $\specnorm{ V_{X^\perp}^T  V_{U_0 W_0}   } =0 $, which verifies the first equality. In order to see that the second equality holds we note that
	\begin{equation*}
	\sigma_{\min} \left( U_0 W_0  \right)  = \sigma_{\min} \left( \alpha V_X \right) = \alpha.
	\end{equation*}
	The third equality follows directly from the definition of $U_0$. This finishes the proof.
\end{proof}
Now we are in a position to give a proof of  Theorem \ref{mainresult:requalsn}.

\begin{proof}[Proof of Theorem \ref{mainresult:requalsn}]
	By Lemma  \ref{lemma:spectralmain1} we have that $\specnorm{ V_{X^\perp}^T  V_{U_0 W_0}   } =0$, $\sigma_{\min} \left( U_0 W_0  \right)  = \alpha$, and $ \specnorm{ U_0   } =\alpha $.
	This allows us to apply Theorem \ref{thm:localconvergence} with $t_0=0$ and $\gamma = \alpha$, which yields that after 
	\begin{align*}
	\hat{t} \lesssim    \frac{1}{\mu \sigma_{\min } \left(X\right)^2 }\ln \left( \max \left\{   1; \frac{\kappa r_{\star}}{n-r_{\star}} \right\}   \frac{\Vert X \Vert}{\alpha} \right)  
	\end{align*}
	iterations we have that
	\begin{align*}
	\frac{\fronorm{U_{\hat{t}} U_{\hat{t}}^{T} -XX^T } }{\specnorm{X}^2}  &\lesssim \frac{ r_{\star}^{1/8}     \left(  n-r_{\star}  \right)^{3/8} }{\kappa^{3/16}  }  \cdot   \frac{   \alpha^{21/16}  }{  \specnorm{X}^{21/16}} \\
	&\le  \frac{ r_{\star}^{1/8}    n^{3/8} }{\kappa^{3/16}  }  \cdot   \frac{   \alpha^{21/16}  }{  \specnorm{X}^{21/16}} .
	\end{align*}
	This finishes the proof.
\end{proof}

%% file: conclusion.tex
\section{Conclusion}
In this paper we focused on demystifying the role of initialization when training overparameterized models by showing that small random initialization followed by a few iterations of gradient descent behaves akin to popular spectral methods. We also show that this \emph{implicit spectral bias} from small random initialization, which is provably more prominent for overparameterized models, also puts the gradient descent iterations on a particular trajectory towards solutions that are not only globally optimal but also generalize well. 

We think that our results give rise to a number of interesting future research directions. For example, one could extend our results to scenarios where the measurement matrices are more structured such as in matrix completion \cite{candes_recht_MC} or in blind deconvolution \cite{blind_deconvolution}. Moreover, while our main results, e.g. Theorem \ref{mainresult:runequaln} do require early stopping, our simulations (e.g. Figure \ref{fig:alphasmall}) indicate that early stopping is not needed. It would be interesting to examine whether we can remove the early stopping requirement. It is also an interesting future avenue to examine whether the quadratic dependence of the sample complexity $m$ on $r_{\star}$ in our results  is really needed.

Moreover, while in this paper our main focus was on low-rank matrix reconstruction, we believe that our analysis holds more generally for a variety of contemporary overparameterized machine learning and signal estimation tasks including neural network training. This is a tantalizing future research direction.

\section*{Acknowledgements}
M.S. is supported by the Packard Fellowship in Science and Engineering, a Sloan Research Fellowship in Mathematics, an NSF-CAREER under award \#1846369, the Air Force Office of Scientific Research Young Investigator Program (AFOSR-YIP) under award \#FA9550-18-1-0078, DARPA Learning with Less Labels (LwLL) and FastNICS programs, and NSF-CIF awards \#1813877 and \#2008443.

%% file: spectral_phase_appendix.tex
\section{Proofs for the spectral phase}\label{appendix:spectralphase}

\subsection{Proof of Lemma \ref{lemma:errorboundspectral}}

\begin{proof}[Proof of Lemma \ref{lemma:errorboundspectral}]
	We are first  going to derive a formula for  $\tilde{U}_t -U_t$.\\
	
	\noindent \textbf{Claim: } Set $ 	\hat{E}_i := \mu  \mathcal{A}^* \mathcal{A} \left( U_{i-1} U_{i-1}^T  \right) U_{i-1}$. Then, for $t\ge 1$ it holds that
	\begin{equation}\label{equ:difference_char}
	\tilde{U}_t - \Ut = \sum_{i=1}^{t}  \left(  \Id + \mu \mathcal{A}^* \mathcal{A} \left(   XX^T  \right) \right)^{t-i}   \hat{E}_i.
	\end{equation}
	\noindent \textbf{Proof of the claim:} We will prove the claim by induction. For $t=1$ we note that
	\begin{align*}
	U_1 &= \left(  \Id + \mu \mathcal{A}^* \mathcal{A} \left(    XX^T  - U_0 U_0^T  \right) \right) U_0\\
	&= \left(  \Id + \mu \mathcal{A}^* \mathcal{A} \left(    XX^T    \right) \right) U_0 - \mu \mathcal{A}^* \mathcal{A} \left(   U_0 U_0^T\right) U_0\\
	&=  \tilde{U_1} - \hat{E}_1,
	\end{align*}
	which proves the claim for $t=1$. Now suppose that the claim holds for some $t$. We obtain that
	\begin{align*}
	\Utplus  &= \left(  \Id + \mu \mathcal{A}^* \mathcal{A} \left(   XX^T  - \Ut \Ut^T  \right) \right) \Ut\\ 
	&=\left(  \Id + \mu \mathcal{A}^* \mathcal{A} \left(    XX^T    \right) \right) \Ut - \mu \mathcal{A}^* \mathcal{A} \left(   \Ut \Ut^T  \right) \Ut\\
	&=\left(  \Id + \mu \mathcal{A}^* \mathcal{A} \left(   XX^T   \right) \right) \Ut - \hat{E}_{t+1},
	\end{align*}
	where the last line follows from the definition of $\hat{E}_{t+1}$. By using the induction hypthesis we obtain that
	\begin{align*}
	\Utplus &= \left(  \Id + \mu \mathcal{A}^* \mathcal{A} \left(    XX^T     \right) \right) \left(   \tilde{U}_t  -  \sum_{i=1}^{t}  \left(  \Id + \mu \mathcal{A}^* \mathcal{A} \left(    XX^T    \right) \right)^{t-i}   \hat{E}_i  \right) - \hat{E}_{t+1}\\
	&=\tilde{U}_{t+1}    -\sum_{i=1}^{t} \left(  \Id + \mu \mathcal{A}^* \mathcal{A} \left(   XX^T  \right) \right)^{t+1-i}   \hat{E}_i - \hat{E}_{t+1}\\
	&= \tilde{U}_{t+1}-\sum_{i=1}^{t+1} \left(  \Id + \mu \mathcal{A}^* \mathcal{A} \left(    XX^T   \right) \right)^{t+1-i}   \hat{E}_i,
	\end{align*}
	which shows the claimed equation \eqref{equ:difference_char}.
	
  In order to estimate $ \specnorm{U_t - \tilde{U}_t} $ we note first that
	\begin{align*}
	\specnorm{\hat{E}_i} &= \mu \specnorm{  \mathcal{A}^* \mathcal{A} \left( U_{i-1} U_{i-1}^T  \right) U_{i-1}}\\
	&\le \mu \specnorm{ \mathcal{A}^* \mathcal{A} \left( U_{i-1} U_{i-1}^T  \right) } \specnorm{U_{i-1}}\\
	&\le \left(1+\delta_1 \right) \mu \nucnorm{ U_{i-1} U_{i-1}^T  } \specnorm{U_{i-1}}\\
	&= \left(1+\delta_1 \right) \mu \fronorm{ U_{i-1}  }^2 \specnorm{U_{i-1}}.
	\end{align*}
	Moreover, we observe that
	\begin{align*}
	\big\Vert \left(  \Id + \mu \mathcal{A}^* \mathcal{A} \left(  XX^T   \right) \right)^{t-i}   \hat{E}_i \big\Vert  &\le \big\Vert  \Id + \mu \mathcal{A}^* \mathcal{A} \left(    XX^T   \right)  \big\Vert^{t-i} \big\Vert   \hat{E}_i \big\Vert \\
	&\le  \left( 1+\mu \specnorm{ \mathcal{A}^* \mathcal{A} \left(   XX^T    \right)  }  \right)^{t-i} \big\Vert   \hat{E}_i \big\Vert \\
	&\le  \left( 1+ \mu \lambda_1 \left(M\right) \right)^{t-i} \big\Vert   \hat{E}_i \big\Vert, 
	\end{align*}
	where in the first line we used the submultiplicativity of the operator norm and in the second line we used the triangle inequality. In the third line we used that $ \specnorm{ \mathcal{A}^* \mathcal{A} \left(   XX^T   \right)  }  = \lambda_1 \left(M\right) $.  Hence, we have shown that
	\begin{equation}\label{ineq:initintern19}
	\specnorm{\Ut - \Uttilde} \le   \sum_{i=1}^{t}  \left( 1+ \mu \lambda_1 \left(M\right) \right)^{t-i}  \left(1+\delta_1\right) \mu \fronorm{ U_{i-1}  }^2 \specnorm{U_{i-1}}.
	\end{equation}
	Note that for all $ 1 \le  i \le  t^{\star}$ we have that$ \specnorm{ \tilde{U}_{i-1} - U_{i-1}}  \le \specnorm{\tilde{U}_{i-1} } $, which implies that
	\begin{align*}
	\fronorm{U_{i-1}}^2 \specnorm{U_{i-1}} &\le \min \left\{ r;n \right\} \specnorm{U_{i-1}}^3\\
	&\le \min \left\{ r;n \right\} \left(  \specnorm{\tilde{U}_{i-1} } + \specnorm{ \tilde{U}_{i-1} - U_{i-1}}   \right)^3 \\
	&\le 8 \min \left\{ r;n \right\} \specnorm{\tilde{U}_{i-1} }^3\\
	&\le 8 \min \left\{ r;n \right\} \specnorm{ \Id + \mu   \mathcal{A}^* \mathcal{A} \left(XX^T\right)      }^{3 \left(i-1\right)} \specnorm{ U_0}^3\\
	&\le  8 \min \left\{ r;n \right\} \left( 1+\mu \lambda _1 \left(M\right) \right)^{3i-3} \alpha^3  \specnorm{U}^3 .
	\end{align*}
	In order  to proceed assume now $t\le t^{\star} $. Then by inequality \eqref{ineq:initintern19} and the previous inequality  we obtain that
	\begin{equation}\label{intern:201}
	\begin{split}
	\specnorm{\Ut - \Uttilde} &\le   \sum_{i=1}^{t}  \left( 1+ \mu \lambda_1 \right)^{t-i}  \left(1+\delta_1\right) \mu \fronorm{ U_{i-1}  }^2 \specnorm{U_{i-1}} \\
	&\le  8 \sum_{i=1}^{t}  \left( 1+ \mu \lambda_1 \left(M\right) \right)^{t-i}  \left(1+\delta_1\right) \mu  r \left( 1+\mu \lambda _1\right)^{3i-3} \alpha^3 \specnorm{U}^3  \\
	&=  8 \alpha^3  \mu \min \left\{ r;n \right\} \left(1+\delta_1 \right)   \left( 1+\mu \lambda_1 \left(M\right)  \right)^{t-1} \sum_{i=1}^{t} \left(1 + \mu \lambda_1 \right)^{2\left(i-1\right)}  \specnorm{U}^3  \\
	&=  8 \alpha^3  \mu \min \left\{ r;n \right\} \left(1+\delta_1 \right)   \left( 1+\mu \lambda_1 \left(M\right)   \right)^{t-1}  \frac{  \left( 1+\mu \lambda_1 \left(M\right)  \right)^{2t} -1  }{\left( 1+\mu \lambda_1 \left(M\right)  \right)^2 -1} \specnorm{U}^3   \\
	&\le   \frac{4}{\lambda_1 \left(M\right)} \alpha^3 \min \left\{ r;n \right\} \left(1+\delta_1 \right)   \left( 1+\mu \lambda_1 \left(M\right)  \right)^{3t}  \specnorm{U}^3  .
	\end{split}
	\end{equation}
	This shows the claim.
\end{proof}

\subsection{Proof of Lemma \ref{lemma:init:closeness}}

\begin{proof}[Proof of Lemma \ref{lemma:init:closeness}]
	First, we note that $ \specnorm{\Uttilde} \ge \twonorm{ \Uttilde^T v_1 } $. Then, we observe that
	\begin{align*}
	\Uttilde^T v_1 &= U_0^T \left(  \Id + \mu \mathcal{A}^* \mathcal{A} \left(   XX^T  \right) \right)^t v_1\\
	&= U_0^T \left(   \sum_{i=1}^{n} \left( 1+ \mu \lambda_i \right) v_i v_i^T   \right)^t v_1\\
	&= U_0^T   \sum_{i=1}^{n} \left( 1+ \mu \lambda_i \right)^t v_i v_i^T   v_1\\
	&= \left( 1+\mu \lambda_1 \right)^t U_0^T  v_1.
	\end{align*}
	This proves that $\specnorm{\Uttilde} \ge   \left( 1+\mu \lambda_1 \left(M\right)  \right)^t  \twonorm{ U_0^T  v_1 }  $. From this observation together with Lemma \ref{lemma:errorboundspectral} it follows for all $t < t^{\star}$ that
	\begin{equation*}
	\frac{ \specnorm{\Ut - \Uttilde}  }{\specnorm{\Uttilde} }  \le  \frac{4}{\lambda_1 \left(M\right) } \alpha^2  \left( \frac{\alpha \min \left\{ r;n \right\}  }{\twonorm{U_0^T v_1}} \right)  \left(1+\delta_1 \right)   \left( 1+\mu \lambda_1 \left(M\right)  \right)^{2t} \specnorm{U}^3  .
	\end{equation*}
	In order to finsh the proof, we are going to derive a lower bound for $t^{\star}$. First we note that by the definition of $t^*$ followed by elementary algebraic manipulations for $t< t^*$ we have
	\begin{align*}
	& \frac{4}{\lambda_1 \left(M\right) } \alpha^2  \left( \frac{\alpha \min \left\{ r;n \right\}}{\twonorm{U_0^T v_1}} \right)  \left(1+\delta_1 \right)   \left( 1+\mu \lambda_1 \left(M\right)  \right)^{2t} \specnorm{U}^3   < 1\\
	\Longleftrightarrow &  \left( 1+\mu \lambda_1 \left(M\right)  \right)^{2t} \specnorm{U}^3   < \frac{\lambda_1 \left(M\right) }{4 \alpha^2 \left( 1+\delta_1 \right) } \left( \frac{\twonorm{U_0^T v_1}}{\alpha \min \left\{ r;n \right\}}  \right)\\
	\Longleftrightarrow & t< \frac{ \ln \left(   \frac{\lambda_1 \left(M\right) }{4 \alpha^2 \left( 1+\delta_1 \right)  \specnorm{U}^3    } \left( \frac{\twonorm{U_0^T v_1}}{\alpha \min \left\{ r;n \right\}}  \right)  \right)}{ 2 \ln \left(  1+\mu \lambda_1 \left(M\right)  \right)}.
	\end{align*}
	Therefore, we must have
	\begin{equation*}
	t^{\star} \ge  \Bigg\lfloor  \frac{ \ln \left(   \frac{\lambda_1 \left(M\right)}{ 4 \alpha^2 \left( 1+\delta_1 \right) \specnorm{U}^3  } \left( \frac{\twonorm{U_0^T v_1}}{\alpha \min \left\{ r;n \right\}}  \right)  \right)}{ 2 \ln \left(  1+\mu \lambda_1 \left( M \right)  \right)} \Bigg\rfloor.
	\end{equation*}
\end{proof}

\subsection{Proof of Lemma \ref{lemma:fundam}}

\begin{proof}[Proof of Lemma \ref{lemma:fundam}]
	\noindent \textbf{Proof of inequality \eqref{ineq:initintern21}:} Due to Weyl's inequality we have that
	\begin{align*}
	\sigma_{r_{\star}} \left(\Zt U_0 + E_t\right)  &\ge \sigma_{r_{\star}} \left(\Zt U_0 \right)  - \specnorm{E_t} \ge  \sigma_{r_{\star}} \left(V_L^T \Zt U_0 \right) - \specnorm{E_t} ,
	\end{align*}
	where the second inequality follows from the Courant-Fisher minimax theorem (see, e.g., \cite[Appendix A]{foucartrauhut}). Now we note that
	\begin{align*}
	\sigma_{r_{\star}} \left( V_L^T \Zt U_0 \right)  & =   \sigma_{\min } \left( V_L^T \Zt V_L V_L^T U_0 \right)  \\
	& \ge \sigma_{\min} \left( V_L^T \Zt V_L \right) \sigma_{\min} \left(  V_L^T U_0  \right) \\
	& =  \sigma_{r_{\star} } \left(   Z_t  \right)   \sigma_{\min} \left(  V_L^T U_0  \right) \\
	& = \alpha  \sigma_{r_{\star}} \left( \Zt  \right) \sigma_{\min} \left(  V_L^T U \right) .
	\end{align*}
	This shows the second statement.\\
	
	\noindent \textbf{Proof of  inequality \eqref{ineq:initintern22}:}  From Weyl's inequality it follows that
	\begin{equation}\label{ineq:initinter4}
	\sigma_{r_\star+1} \left(\Zt U_0+E_t\right) \le  \sigma_{r_\star+1} \left(\Zt U_0  \right)  + \specnorm{E_t}.
	\end{equation}
	Denote by $U=V_U \Sigma_U W_U^T $ the singular value decomposition of $U$. Then we can compute that
	\begin{align*}
	\sigma_{r_\star+1} \left(\Zt U_0 \right)&= \alpha   \underset{ \mathcal{V}, \text{dim} \mathcal{V} = r_{\star}+1 }{\max} \  \underset{x \in \mathcal{V}, \twonorm{x}  =1}{\min} \twonorm{\Zt V_U \Sigma_U  x}  \\
	&=  \alpha   \underset{ \mathcal{V}, \text{dim} \mathcal{V} = r_{\star}+1 }{\max} \  \underset{x \in \mathcal{V}, \twonorm{x} =1}{\min} \twonorm{\Zt V_U \frac{ \Sigma_U  x}{\twonorm{ \Sigma_U x }} } \twonorm{ \Sigma_U x } \\
	&\le   \alpha   \underset{ \mathcal{V}, \text{dim} \mathcal{V} = r_{\star}+1 }{\max} \  \underset{x \in \mathcal{V}, \twonorm{x} =1}{\min} \twonorm{\Zt V_U \frac{ \Sigma_U  x}{\twonorm{ \Sigma_U x }} } \specnorm{ \Sigma_U } \\
	&= \alpha  \underset{ \mathcal{V}, \text{dim} \mathcal{V} = r_{\star}+1 }{\max} \  \underset{x \in  \mathcal{V}, \twonorm{x}  =1}{\min} \twonorm{\Zt V_U x}  \specnorm{U} \\
	&\le  \alpha \underset{ \mathcal{V}, \text{dim} \mathcal{V} = r_{\star}+1 }{\max} \  \underset{x \in \mathcal{V}, \twonorm{x}  =1}{\min} \twonorm{\Zt x} \specnorm{U} \\
	&= \alpha \sigma_{r_{\star}+1} \left( \Zt  \right) \specnorm{U}.
	\end{align*}
	The first line is due to the Courant-Fisher minimax theorem and $U_0 = \alpha U $. 	The last line follows again from the Courant-Fisher minimax theorem. Together with inequality \eqref{ineq:initinter4} this implies the third claim.\\

	\noindent \textbf{Proof of inequality \eqref{ineq:initintern6}:} 	
	First, we note that
	\begin{equation*}
	\Zt U_0+E_t= \Zt V_{L} V_{L}^T U_0 + \bracing{=:H}{ \Zt V_{L^\perp} V_{L^\perp}^T U_0+E_t}.
	\end{equation*}
	Note that since $V_{L}^T V_U$ has rank $r_{\star}$, the matrix $\Zt V_{L} V_{L}^TU $ must have rank $r_{\star}$ as well. In particular, since $\Zt V_{L} V_{L}^T U=V_{L} V_{L}^T \Zt V_{L} V_{L}^T U$ this means that $L$ is the subspace spanned by the left-singular vectors of $\Zt V_{L} V_{L}^T U$ corresponding to the largest $r_{\star}$ singular values. Due to Wedin's sin $\theta$ theorem \cite{wedinbound} we obtain that
	\begin{align*}
	\specnorm{V_{L^\perp}^T V_{\Lttilde} } &\le \frac{\specnorm{H}}{\sigma_{r_{\star} } \left( \Zt V_{L} V_{L}^T  U_0 \right)  - \sigma_{r_{\star} +1 } \left( \Zt U_0+E_t \right)  }\\
	&\le \frac{\specnorm{H}}{  \alpha \sigma_{r_{\star} } \left( Z_t   \right) \sigma_{ \min  } \left(  V_L^T U_0  \right)     - \sigma_{r_{\star} +1 } \left( \Zt U_0+E_t \right)  }\\
	&\le \frac{\specnorm{H}}{  \alpha   \sigma_{r_{\star}} \left( \Zt  \right) \sigma_{\min} \left(  V_L^T U \right)  -  \alpha \sigma_{r_{\star}+1} \left(\Zt\right) \specnorm{U}  -\specnorm{E_t}  },
	\end{align*}
	where in the last line we also used  \eqref{ineq:initintern22}.
	(Note that the assumption \eqref{ineq:initintern20} guarantees that the denominator is positive, which is a necessary condition for an application of Wedin's sin $\theta$ theorem.) Now we observe that
	\begin{align*}
	\specnorm{H} \le \specnorm{\Zt V_{L^\perp} V_{L^\perp}^T U } + \specnorm{ E_t} \le  \alpha \specnorm{\Zt V_{L^\perp} } \specnorm{U} + \specnorm{E_t }  = \alpha  \sigma_{r_{\star} +1 } \left(\Zt \right)  \specnorm{U}  +  \specnorm{E_t } .
	\end{align*}
	Together with the previous inequality chain, this shows inequality \eqref{ineq:initintern6}.
\end{proof}

\subsection{Proof of Lemma \ref{lemma:init:closeness2}}

Before proving Lemma \ref{lemma:init:closeness2}, we are going to introduce some notation. Let $ U_t = \sum_{i=1}^{r} \sigma_i u_i v^T_i $ be the singular value decomposition of $U_t$.  Define $ \Lttilde:=    \sum_{i=1}^{r_{\star}} \sigma_i u_i v^T_i $ and $ \Nttilde:=   \sum_{i=r_{\star}+1}^{r} \sigma_i u_i v^T_i  $. Denote by $ \Lttilde=V_{\Lttilde}  \Sigma_{\Lttilde} W_{\Lttilde}^T $ and $ \Nttilde=V_{\Nttilde}  \Sigma_{\Nttilde} W_{\Nttilde}^T $ the singular value decomposition of those  two matrices.

We start by proving the following technical lemma. It says that if the subpace spanned by the columns of $X$ and $\Lttilde$ are aligned, then also the subspaces given by $\WWt$ and $W_{\Lttilde^\perp}$ will be closely aligned.
\begin{lemma}\label{lemma:init123}
	Assume that  $  \specnorm{V_{X^\perp}^T  V_{\Lttilde}} \le 1/2  $. Then it holds that
	\begin{equation*}
	\specnorm{ W_{\Lttilde^\perp}^T  W_t } \le  \frac{ 2 \sigma_{r_{\star} +1} \left( U_t\right)   \specnorm{ V_{X^\perp}^T V_{\Lttilde}   } }{ \sigma_{r_{\star} } \left( U_t\right) }.
	\end{equation*}
\end{lemma}

\begin{proof}
	We note that
	\begin{align*}
	\specnorm{ W_{\Lttilde^\perp}^T  W_t } &= \sqrt{  \specnorm{ W_{\Lttilde^\perp}^T  W_t W_t^T W_{\Lttilde^\perp}   }  } \\
	&=\sqrt{  \specnorm{ W_{\Lttilde^\perp}^T   U_t^T V_X \left(  V_X^T U_t U_t^T V_X  \right)^{-1} V_X^T  U_t   W_{\Lttilde^\perp}   }  } \\
	&=\sqrt{  \specnorm{ W_{\Lttilde^\perp}^T   U_t^T V_X \left(  V_X^T U_t U_t^T V_X  \right)^{-1} V_X^T  U_t   W_{\Lttilde^\perp}   }  }\\
	&=\sqrt{  \specnorm{ W_{\Lttilde^\perp}^T    \Nttilde^T V_X \left(  V_X^T U_t U_t^T V_X  \right)^{-1} V_X^T  \Nttilde   W_{\Lttilde^\perp}   }  } \\
	&=\sqrt{  \specnorm{ W_{\Lttilde^\perp}^T W_{\Nttilde}    \Sigma_{\Nttilde} V_{\Nttilde}^T  V_X \left(  V_X^T U_t U_t^T V_X  \right)^{-1} V_X^T  V_{\Nttilde} \Sigma_{\Nttilde}     W_{\Nttilde}^T   W_{\Lttilde^\perp} } }\\
	&= \sqrt{  \specnorm{  \Sigma_{\Nttilde} V_{\Nttilde}^T  V_X \left(  V_X^T U_t U_t^T V_X  \right)^{-1} V_X^T  V_{\Nttilde} \Sigma_{\Nttilde}   }  }\\
	& \le  \frac{\specnorm{  \Sigma_{\Nttilde}  } \specnorm{ V_{\Nttilde}^T  V_X  }  }{\sigma_{\min}  \left( V_X^T U_t \right) }.
	\end{align*}
	In order to control the denominator we note that
	\begin{align*}
	\sigma_{\min}  \left( V_X^T U_t \right)  &= \sqrt{   \sigma_{\min}  \left( V_X^T U_t U_t^T V_X \right)     }\\
	&= \sqrt{   \sigma_{\min}  \left( V_X^T  \left( \Lttilde \Lttilde^T  + \Nttilde \Nttilde^T   \right)  V_X \right)     }\\
	&\ge  \sqrt{   \sigma_{\min}  \left( V_X^T  \Lttilde \Lttilde^T    V_X \right)     } \\
	& = \sigma_{\min} \left( V_X^T \Lttilde   \right) \\
	&\ge \sigma_{\min}\left( V_X^T V_{\Lttilde}   \right)  \sigma_{\min} \left(\Lttilde\right)  \\
	&\ge  \frac{ \sigma_{\min} \left(\Lttilde\right)  }{2}.
	\end{align*}
	In the last line we have used the assumption $  \specnorm{V_{X^\perp}^T  V_{\Lttilde}} \le 1/2  $. Hence, we have shown that
	\begin{align*}
	\specnorm{ W_{\Lttilde^\perp}^T  W_t } & \le  \frac{ 2\specnorm{   \Sigma_{\Nttilde}  } \specnorm{ V_{\Nttilde}^T  V_X  }  }{ \sigma_{\min}  \left( \tilde{L }  \right) }\\
	&=  \frac{ 2 \sigma_{r_{\star} +1} \left( U_t\right)   \specnorm{ V_{\Nttilde}^T  V_X   } }{ \sigma_{r_{\star} } \left( U_t\right) } \\
	&\le   \frac{ 2 \sigma_{r_{\star} +1} \left( U_t\right)   \specnorm{ V_{\Lttilde^\perp}^T  V_X   } }{ \sigma_{r_{\star} } \left( U_t\right) } \\
	&=  \frac{ 2 \sigma_{r_{\star} +1} \left( U_t\right)   \specnorm{ V_{X^\perp}^T V_{\Lttilde}   } }{ \sigma_{r_{\star} } \left( U_t\right) },
	\end{align*}
	which finishes the proof.
\end{proof}
Now we are in a position to prove Lemma \ref{lemma:init:closeness2}.
\begin{proof}[Proof of Lemma \ref{lemma:init:closeness2}]
	\noindent \textbf{Proof of inequality $\eqref{ineq:initspec2}$:} First, we observe that due to Lemma \ref{lemma:init123} and the assumption $\specnorm{V_{X^\perp}^T  V_{ \Lttilde  }} \le \frac{1}{8} $ we have that
	\begin{equation}\label{ineq:intern7654}
	\specnorm{ W_{\Lttilde^\perp}^T  W_t } \le  \frac{ 2 \sigma_{r_{\star} +1} \left( U_t\right)   \specnorm{ V_{X^\perp}^T V_{\Lttilde}   } }{ \sigma_{r_{\star} } \left( U_t\right) } \le 1/4.
	\end{equation}
	Then, we note that 
	\begin{align*}
	\sigma_{r_{\star}} \left( U_t W_t \right)^2   &=  \sigma_{r_{\star}} \left( W_t^T   U_t^T   U_t W_t  \right) \\
	&=  \sigma_{r_{\star}} \left( W_t^T   \left(  \Lttilde^T \Lttilde + \Nttilde^T \Nttilde   \right) W_t  \right)\\
	&\ge  \sigma_{r_{\star}} \left( W_t^T   \Lttilde^T \Lttilde W_t  \right) \\
	&\ge  \sigma_{r_{\star}} \left( W_t^T   W_{\Lttilde} \right)^2   \sigma_{r_{\star}} \left( \Lttilde \right)^2\\
	&=   \left(1- \specnorm{ W_{\Lttilde^{\perp}}^T W_t }^2\right)    \sigma_{r_{\star}} \left( U_t \right)^2.
	\end{align*}
	Using inequality \eqref{ineq:intern7654} we obtain inequality \eqref{ineq:initspec2}. \\
	
	\noindent \textbf{Proof of inequality $\eqref{ineq:initspec1}$:} Note that
	\begin{align*}
	V_{X^\perp}^T    V_{U_t W_t}  &= V_{X^\perp}^T   V_{U_t W_t}   V_{U_t W_t}^T U_t W_t \left( V_{U_t W_t}^T U_t W_t     \right)^{-1} \\
	&=  V_{X^\perp}^T   U_t W_t \left( V_{U_t W_t}^T U_t W_t     \right)^{-1} .
	\end{align*}
	By the triangle inequality it follows that
	\begin{align*}
	\specnorm{	V_{X^\perp}^T    V_{U_t W_t}} \le \specnorm{     V_{X^\perp}^T   \Lttilde W_t \left( V_{U_t W_t}^T U_t W_t     \right)^{-1}   }+ \specnorm{   V_{X^\perp}^T   \Nttilde W_t \left( V_{U_t W_t}^T U_t W_t     \right)^{-1}  }.
	\end{align*}
	The second term can be bounded as follows.
	\begin{equation}\label{ineq:intern:spectral4}
	\begin{split}
	\specnorm{   V_{X^\perp}^T   \Nttilde W_t \left( V_{U_t W_t}^T U_t W_t     \right)^{-1}  } &\le  \frac{\specnorm{ \Nttilde W_t }}{  \sigma_{r_{\star}} \left( U_t W_t  \right)  }  \\
	&\le   \frac{\specnorm{\Nttilde W_{\Nttilde}} \specnorm{  W^T_{\Nttilde} W_t }}{  \sigma_{r_{\star}} \left( U_t W_t  \right)  } \\
	&=  \frac{ \sigma_{r_{\star}+1} \left(U_t\right)   \specnorm{ W_{\Nttilde}^T W_t   }    }{  \sigma_{r_{\star}} \left( U_t W_t  \right)  } \\
	&\le  \frac{ \sigma_{r_{\star}+1} \left(U_t\right)   \specnorm{ W_{\Lttilde^\perp}^T W_t   }    }{  \sigma_{r_{\star}} \left( U_t W_t  \right)  }\\
	&\le  2  \frac{ \sigma_{r_{\star}+1} \left(U_t\right)   \specnorm{ W_{\Lttilde^\perp}^T W_t   }    }{  \sigma_{r_{\star}} \left( U_t   \right)  }.
	\end{split}
	\end{equation}
	In order to bound the first term, we note that
	\begin{align*}
	\specnorm{     V_{X^\perp}^T   \Lttilde W_t \left( V_{U_t W_t}^T U_t W_t     \right)^{-1}  } &\le  \specnorm{V_{X^\perp}^T  V_{\Lttilde} }  \specnorm{    \Lttilde W_t \left( V_{U_t W_t}^T U_t W_t     \right)^{-1}    } \\
	&\overleq{(a)} \specnorm{V_{X^\perp}^T  V_{\Lttilde} }  \left(    \specnorm{    U_t W_t \left( V_{U_t W_t}^T U_t W_t     \right)^{-1}   }  +   \specnorm{    \Nttilde  W_t \left( V_{U_t W_t}^T U_t W_t     \right)^{-1}   }    \right) \\
	& =  \specnorm{V_{X^\perp}^T  V_{\Lttilde} }  \left(   1 +   \specnorm{    \Nttilde  W_t \left( V_{U_t W_t} U_t W_t     \right)^{-1}   }    \right) \\
	& \overleq{(c)}   \specnorm{V_{X^\perp}^T  V_{ \Lttilde  } }  \left(   1 +   \frac{\specnorm{\Nttilde  W_t }}{\sigma_{r_{\star} }  \left(    U_t W_t \right)}   \right)  \\
	& \overleq{(c)}   \specnorm{V_{X^\perp}^T  V_{\Lttilde} }  \left(   1 +      \frac{ \sigma_{r_{\star}+1} \left(U_t\right)   \specnorm{ W_{\Lttilde^\perp}^T W_t   }    }{  \sigma_{r_{\star}} \left( U_t W_t    \right)  }  \right) \\
	& \le   \specnorm{V_{X^\perp}^T  V_{\Lttilde } }  \left(   1 +    2  \frac{ \sigma_{r_{\star}+1} \left(U_t\right)   \specnorm{ W_{\Lttilde^\perp}^T W_t   }    }{  \sigma_{r_{\star}} \left( U_t   \right)  }  \right) \\
	&\overleq{(d)} 3  \specnorm{V_{X^\perp}^T  \Lttilde }  .
	\end{align*}
	In $(a)$ we have used the triangle inequality and inequality $(b)$ follows from inspecting the inequality chain \eqref{ineq:intern:spectral4}. In $(c)$ we used inequality \eqref{ineq:initspec2} and $(d)$   follows from \eqref{ineq:intern7654}. Combining our results we obtain that
	\begin{align*}
	\specnorm{	V_{X^\perp}^T    V_{U_t W_t}} & \le  3 \specnorm{V_{X^\perp}^T  V_{L_t} }  +    2  \frac{ \sigma_{r_{\star}+1} \left(U_t\right)   \specnorm{ W_{\Lttilde^\perp}^T W_t   }    }{  \sigma_{r_{\star}} \left( U_t   \right)  }\\
	&\le 3 \specnorm{V_{X^\perp}^T  V_{L_t} }  +    4  \frac{ \sigma_{r_{\star}+1}^2 \left(U_t\right)   \specnorm{ V_{X^\perp}^T V_{\Lttilde}   }    }{  \sigma_{r_{\star}}^2 \left( U_t   \right)  }  \\
	&\le 7 \specnorm{V_{X^\perp}^T  V_{L_t} },
	\end{align*}
	where in the second line we used Lemma \ref{lemma:init123}. This shows \eqref{ineq:initspec1}.\\
	
	\noindent \textbf{Proof of inequality \eqref{ineq:initspec3}:}  We note that
	\begin{equation}\label{ineq:interninit1}
	\begin{split}
	\specnorm{ U_t W_{t,\perp }  }  &\le \specnorm{ \Lttilde W_{t,\perp }  } + \specnorm{ \Nttilde W_{t,\perp }  } \\
	& \le \specnorm{\Lttilde W_{t,\perp }  } + \specnorm{ \Nttilde   } \\
	&= \specnorm{ \Lttilde W_{t,\perp }  } + \sigma_{r_{\star} +1 } \left(U_t\right)  .
	\end{split}
	\end{equation}
	Observe that $ \specnorm{ \Lttilde W_{t,\perp } } =  \specnorm{ \Lttilde  W_{t,\perp } W_{t,\perp }^T }$. Then we compute that
	\begin{align*}
	\Lttilde W_{t,\perp } W_{t,\perp }^T &= \Lttilde \left( \Id - U_t ^T V_X\left( V_X^T U_tU_t^T V_X \right)^{-1} V_X^T U_t \right) \\
	&=  \Lttilde\left( \Id - \Lttilde^T V_X\left( V_X^T U_tU_t^T V_X \right)^{-1} V_X^T U_t \right) \\
	&=\Lttilde\left(  W_{\Lttilde} W_{\Lttilde}^T  - \Lttilde^T V_X\left( V_X^T U_tU_t^T V_X \right)^{-1} V_X^T U_t \right).
	\end{align*}
	Next, we note that
	\begin{align*}
	V_X^T \UUt \UUt^T V_X &=   V_X^T  \Lttilde \Lttilde^T V_X  +  V_X^T   \Nttilde \Nttilde^T  V_X  \\
	&=   V_X^T  \Lttilde W_{\Lttilde} W_{\Lttilde}^T  \Lttilde^T V_X  +  V_X^T   \Nttilde \Nttilde^T  V_X \\
	&=    V_X^T  \Lttilde W_{\Lttilde}    \left(    \Id + \bracing{=:A}{ \left(V_X^T  \Lttilde W_{ \Lttilde  }  \right)^{-1}  V_X^T   \Nttilde \Nttilde^T  V_X   \left(W_{ \Lttilde }^T  \Lttilde^T V_X \right)^{-1}  } \right)  W_{ \Lttilde }^T  \Lttilde^T V_X   .
	\end{align*}
	Now observe that
	\begin{align*}
	\specnorm{A} &\le  \frac{\specnorm{   V_X^T   \Nttilde \Nttilde^T  V_X  }  }{\sigma_{\min } \left( V_X^T\Lttilde W_{\Lttilde}   \right)^2}\\
	&\le \frac{ \specnorm{V_X^T V_{\Nttilde} }^2  \specnorm{   \Nttilde }^2  }{ \sigma_{\min }\left( V_X^T V_{\Lttilde} \right)^2 \sigma_{\min } \left( \Lttilde   \right)^2   } \\
	&\le \frac{ \specnorm{V_X^T V_{\Lttilde^\perp} }^2   \sigma_{r_{\star} +1 }\left(U_t\right)^2  }{ \sigma_{\min }\left( V_X^T V_{\Lttilde} \right)^2 \sigma_{r_{\star} } \left( U_t\right)^2   }\\
	&= \frac{ \specnorm{V_{X^\perp}^T V_{\Lttilde} }^2   \sigma_{r_{\star} +1 }\left(U_t\right)^2  }{ \sigma_{\min }\left( V_X^T V_{\Lttilde} \right)^2 \sigma_{r_{\star} } \left( U_t\right)^2   }\\
	&\le \frac{ \specnorm{V_{X^\perp}^T V_{\Lttilde} }^2   \sigma_{r_{\star} +1 }\left(U_t\right)^2  }{\left(1- \specnorm{V_{X^\perp}^T V_{\Lttilde} }^2\right) \sigma_{r_{\star} } \left( U_t\right)^2   }\\
	&\le 1/2.
	\end{align*}
	In the last line we have used the assumption $ \specnorm{V_{X^\perp}^T V_{\Lttilde} } \le \frac{1}{8}$. Furthermore, note that we have
	\begin{align*}
	&\Lttilde^T V_X\left( V_X^T U_tU_t^T V_X \right)^{-1} V_X^T U_t \\
	 =& \Lttilde^T V_X\left( W_{  \Lttilde }^T  \Lttilde^T V_X   \right)^{-1}   \left( \Id +A  \right)^{-1} \left(    V_X^T  \Lttilde W_{\Lttilde}  \right)^{-1}     V_X^T U_t  \\
	=& W_{\Lttilde} \left(\Id +A\right)^{-1} \left(    V_X^T  \Lttilde W_{\Lttilde  }  \right)^{-1}     V_X^T U_t \\
	=&  W_{\Lttilde} \left(\Id +A\right)^{-1}   W_{\Lttilde}^T      + W_{\Lttilde} \left(\Id +A\right)^{-1} \left(    V_X^T  \Lttilde W_{\Lttilde}  \right)^{-1}     V_X^T \Nttilde \\
	=&  W_{\Lttilde} W_{\Lttilde}^T   -      W_{\Lttilde } A \left(\Id +A\right)^{-1}   W_{\Lttilde }^T        + W_{\Lttilde} \left(\Id +A\right)^{-1} \left(    V_X^T \Lttilde W_{ \Lttilde   }  \right)^{-1}     V_X^T \Nttilde.
	\end{align*}
	Note that in the last line we used that $\specnorm{A} \le 1/2 $, which we have shown above. It follows that
	\begin{equation*}
	\Lttilde W_{t,\perp } W_{t,\perp }^T  =  \Lttilde    W_{\Lttilde} A \left(\Id +A\right)^{-1}   W_{\Lttilde}^T        - \Lttilde W_{  \Lttilde } \left(\Id +A\right)^{-1} \left(    V_X^T  \Lttilde W_{\Lttilde}  \right)^{-1}     V_X^T \Nttilde.
	\end{equation*}
	In particular, by the triangle inequality it follows that
	\begin{equation}
	\specnorm{\Lttilde  W_{t,\perp }  } \le \bracing{=:(I)}{ \specnorm{ \Lttilde   W_{\Lttilde} A \left(\Id +A\right)^{-1}   W_{\Lttilde }^T   } }+ \bracing{=:(II)}{  \specnorm{\Lttilde  W_{\Lttilde } \left(\Id +A\right)^{-1} \left(    V_X^T  \Lttilde W_{\Lttilde }  \right)^{-1}     V_X^T \Nttilde}}. \label{ineq:interninit2}
	\end{equation}
	\textbf{Bounding $(I)$:} In order to bound the first term, we note that
	\begin{align*}
	&\Lttilde   W_{\Lttilde} A \left(\Id +A\right)^{-1}   W_{\Lttilde}^T \\
	=& \Lttilde   W_{\Lttilde }  \left(V_{X }^T    \Lttilde W_{\Lttilde}  \right)^{-1} V_X^T   \Nttilde \Nttilde^T  V_X   \left(W_{\Lttilde}^T  \Lttilde^T V_X \right)^{-1}   \left(\Id +A\right)^{-1}   W_{\Lttilde}^T  \\
	=& \Lttilde   W_{\Lttilde }  \left(V_{\Lttilde }^T    \Lttilde W_{\Lttilde}  \right)^{-1}\left( V_X^T V_{\Lttilde}  \right)^{-1}  V_X^T   \Nttilde \Nttilde^T  V_X   \left(W_{\Lttilde}^T  \Lttilde^T V_X \right)^{-1}   \left(\Id +A\right)^{-1}   W_{\Lttilde}^T  \\
	=&  V_{\Lttilde}  \left( V_X^T V_{\Lttilde  }  \right)^{-1}  V_X^T   \Nttilde \Nttilde^T  V_X \left( V_{\Lttilde}^T  V_X  \right)^{-1}   \left(W_{\Lttilde}^T  \Lttilde^T  V_{\Lttilde} \right)^{-1}   \left(\Id +A\right)^{-1}   W_{\Lttilde}^T.
	\end{align*}
	It follows that
	\begin{align*}
	\specnorm{\Lttilde    W_{\Lttilde} A \left(\Id +A\right)^{-1}   W_{\Lttilde}^T} &\le \frac{\specnorm{    V_X^T   \Nttilde \Nttilde^T  V_X}}{\sigma_{\min } \left( V_X^T V_{\Lttilde  }  \right)^2  \sigma_{\min } \left( \Id  +A\right) \sigma_{\min } \left(   W_{\tilde{L }}^T  \Lttilde^T  V_{\Lttilde} \right)  } \\
	&\le  \frac{\specnorm{    V_X^T   V_{\Nttilde} }^2  \specnorm{\Nttilde}^2  }{\sigma_{\min } \left( V_X^T V_{\Lttilde  }  \right)^2   \left(1-\specnorm{A}\right)  \sigma_{r_{\star} } \left(   U_t \right)  }  \\
	&\le  \frac{\specnorm{    V_X^T   V_{\Lttilde^\perp} }^2   \sigma_{r_{\star} +1 } \left(U_t\right)^2  }{\sigma_{\min } \left( V_X^T V_{\Lttilde  }  \right)^2   \left(1-\specnorm{A}\right)  \sigma_{r_{\star} } \left(   U_t \right)  }\\
	&=  \frac{\specnorm{    V_{X^\perp}^T   V_{\Lttilde} }^2   \sigma_{r_{\star} +1 } \left(U_t\right)^2  }{\sigma_{\min } \left( V_X^T V_{\Lttilde  }  \right)^2   \left(1-\specnorm{A}\right)  \sigma_{r_{\star} } \left(   U_t \right)  }\\
	&\le \frac{\sigma_{r_{\star} +1 } \left(U_t\right)}{2} 
	\end{align*}
	In the last line we have used the assumption $ \specnorm{V_{X^\perp}^T V_{\Lttilde} } \le \frac{1}{8}$ as well as $\specnorm{A} \le 1/2$.\\
	
	\noindent \textbf{Bounding $(II)$:} We observe that
	\begin{align*}
	&\Lttilde W_{\Lttilde} \left(\Id +A\right)^{-1} \left(    V_X^T  \Lttilde W_{\Lttilde}  \right)^{-1}     V_X^T \Nttilde \\
	=&\Lttilde W_{\Lttilde} \left(    V_X^T  \Lttilde W_{\tilde{L }}  \right)^{-1}     V_X^T \Nttilde -  \Lttilde W_{\Lttilde}A \left(\Id +A\right)^{-1} \left(    V_X^T  \Lttilde W_{\Lttilde}  \right)^{-1}     V_X^T \Nttilde \\
	=&V_{\Lttilde} \left( V_X^T  V_{\Lttilde}  \right)^{-1} V_X^T \Nttilde \\
	&-  \Lttilde W_{\Lttilde}  \left(V_X^T  \Lttilde W_{\Lttilde}  \right)^{-1}  V_X^T   \Nttilde \Nttilde^T  V_X   \left(W_{\Lttilde}^T  \Lttilde^T V_X \right)^{-1}  \left(\Id +A\right)^{-1} \left(    V_X^T  \Lttilde W_{\Lttilde}  \right)^{-1}     V_X^T \Nttilde \\
	=&V_{\Lttilde} \left( V_X^T  V_{\Lttilde}  \right)^{-1} V_X^T \Nttilde\\
	&- V_{\Lttilde} \left( V_X^T V_{\Lttilde} \right)^{-1}  V_X^T   \Nttilde \Nttilde^T  V_X   \left(W_{\Lttilde}^T  \Lttilde^T  V_X \right)^{-1}  \left(\Id +A\right)^{-1} \left(    V_X^T  \Lttilde W_{\Lttilde}  \right)^{-1}     V_X^T \Nttilde \\
	=& V_{\Lttilde} \left( V_X^T  V_{\Lttilde}  \right)^{-1} V_X^T \Nttilde \left(  \Id -   \Nttilde^T  V_X   \left(W_{\Lttilde}^T  \Lttilde^T V_X \right)^{-1}  \left(\Id +A\right)^{-1} \left(    V_X^T  \Lttilde W_{\Lttilde}  \right)^{-1}     V_X^T \Nttilde    \right).
	\end{align*}
	It follows that
	\begin{align*}
	&\specnorm{ \Lttilde W_{\Lttilde} \left(\Id +A\right)^{-1} \left(    V_X^T  \Lttilde W_{\Lttilde}  \right)^{-1}     V_X^T \Nttilde   }\\ 
	\le&\frac{ \specnorm{V_X^T V_{\Nttilde}}  \sigma_{r_{\star} +1 } \left(U_t\right)  }{\sigma_{\min } \left(    V_X^T  V_{\Lttilde} \right) } \left( 1+ \frac{ \sigma_{r_{\star} +1 } \left(U_t\right)^2   \specnorm{V_X^T V_{\Nttilde}  }^2     }{\left( 1-\specnorm{A} \right)  \sigma_{r_{\star} } \left(U_t\right)^2 \sigma_{\min } \left( V_X^T V_{\Lttilde}  \right)^2 }  \right) \\
	\le&\frac{ \specnorm{V_X^T V_{\Lttilde^\perp}}  \sigma_{r_{\star} +1 } \left(U_t\right)  }{\sigma_{\min } \left(    V_X^T  V_{\Lttilde} \right) } \left( 1+ \frac{ \sigma_{r_{\star} +1 } \left(U_t\right)^2   \specnorm{V_X^T V_{\Lttilde}  }^2     }{\left( 1-\specnorm{A} \right)  \sigma_{r_{\star} } \left(U_t\right)^2 \sigma_{\min } \left( V_X^T V_{\Lttilde}  \right)^2 }  \right) \\
	\le & \frac{\sigma_{r_{\star} +1 } \left(U_t\right)}{2} .
	\end{align*}
	In the last line we have used the assumption $ \specnorm{V_{X^\perp}^T V_{\Lttilde} } \le \frac{1}{8}$ as well as $\specnorm{A} \le 1/2$. Hence, from inequality \eqref{ineq:interninit2} it follows that $\specnorm{\Lttilde W_{t,\perp }  }  \le \sigma_{r_{\star} +1 } \left(U_t\right) $. Inserting this result into inequality \eqref{ineq:interninit1} we obtain inequality \eqref{ineq:initspec3}, which finishes the proof.
\end{proof}

\subsection{Proof of Lemma \ref{lemma:Westimates}}
Before we can prove Lemma \ref{lemma:Westimates} we will need a technical lemma. In order to state it, recall that $L$ denotes the subspace spanned by the eigenvectors corresponding to the $r_{\star}$ largest eigenvalues of the matrix $M:=\mathcal{ A}^* \mathcal{ A}\left(XX^T\right)$ and that $V_L \in \mathbb{R}^{n \times r_{\star}}$ is an orthogonal matrix, whose column span is the subpace $L$. The following lemma, which follows from standard matrix perturbation theory arguments, shows that for if  $ \mathcal{ A}^* \mathcal{ A}\left(XX^T\right) $ is sufficiently close to $XX^T $ in spectral norm, then $L$ is aligned with the column space of $X$. Moreover, it says that the eigenvalues of $XX^T$ are close to the ones of $M$.
\begin{lemma}\label{lemma:weylconsequence}
	Suppose that $	M:= \mathcal{A}^* \mathcal{A} \left( XX^T \right) = XX^T +\tilde{E} $ with $\specnorm{ \tilde{E}} \le \delta  \lambda_{r_{\star}}  \left(XX^T\right) $ and $\delta < 1/2$. Then it holds that
	\begin{align*}
	\left(1-\delta \right) \lambda_1 \left(XX^T\right) &\le \lambda_1 \left(M\right)  \le  \left(1+\delta\right)  \lambda_1\left(XX^T\right),   \\
	\lambda_{r_{\star} +1} \left(M\right)  &\le \delta \lambda_{r_{\star}} \left(XX^T\right),  \\
	\lambda_{r_{\star}} \left(M\right)  &\ge \left( 1-\delta \right)   \lambda_{r_{\star}} \left(XX^T\right),\\
	\specnorm{ V_{X^\perp}^T V_{L}  }  &\le  2\delta.
	\end{align*}
\end{lemma}	
\begin{proof}
	The first three inequalities are a direct consequence of Weyl's inequality. In order to prove the fourth inequality, we denote by $L$ the subspace spanned by the eigenvectors corresponding to the $r_{\star}$ largest eigenvalues of $M$. From the Davis-Kahan $\sin \Theta$ theorem \cite{kahanbound} it follows that
	\begin{equation*}
	\begin{split}
	\specnorm{ V_{X^\perp}^T V_{L}  }\le \frac{\specnorm{\tilde{E}}}{\lambda_{r_{\star}} \left(XX^T\right) - \specnorm{\tilde{E}}  } \overleq{(a)} \frac{ \delta  }{ 1-\delta   } \overleq{(b)}  2\delta.
	\end{split}
	\end{equation*}
	Inequality $(a)$ follows from the assumption $ \specnorm{\tilde{E}} \le \delta \lambda _{r_{\star}} \left( XX^T \right)  $. In $(b)$ we used that $\delta \le \frac{1}{2}$.\\
\end{proof}
This allows us to prove Lemma \ref{lemma:Westimates}.
\begin{proof}[Proof of Lemma \ref{lemma:Westimates}]
	Due to the assumption \eqref{equ:gammadefinition} we have that $ \gamma <1/2 $, if $ \widetilde{c}_2  $ is chosen small enough, and hence we can apply Lemma \ref{lemma:fundam}. Hence, we obtain that
	\begin{align}
	\sigma_{r_{\star}}  \left(  U_t \right) &\ge \alpha   \sigma_{r_{\star}} \left( \Zt   \right) \sigma_{\min} \left(  V_L^T U \right) - \specnorm{E_t} \overset{(a)}{\ge}   \frac{\alpha}{2}   \sigma_{r_{\star}} \left(  \Zt  \right) \sigma_{\min} \left(  V_L^T U \right),      \label{ineq:initi5}  \\
	\sigma_{r_\star+1} \left( U_t  \right) &\le  \gamma      \alpha   \sigma_{r_{\star}} \left( \Zt  \right) \sigma_{\min} \left(  V_L^T U \right),\label{ineq:initi6}
	\end{align}
	where in $(a)$ we used that $\gamma \le 1/2$. Moreover, we also have that
	\begin{align*}
	\specnorm{V_{L^\perp}^T V_{\Lttilde} }  &\le \frac{ \alpha  \sigma_{r_{\star} +1 } \left(\Zt  \right) \specnorm{U} +  \specnorm{E_t }    }{  \alpha   \sigma_{r_{\star}} \left( \Zt  \right) \sigma_{\min} \left(  V_L^T U \right)  -  \alpha \sigma_{r_{\star}+1} \left( \Zt  \right)\specnorm{U}  -\specnorm{E_t}  } \le  \frac{\gamma}{1- \gamma}  .
	\end{align*}
	Now note that
	\begin{align*}
	\specnorm{V_{X^\perp}^T V_{\Lttilde} } &=  \specnorm{V_X^T V_X^T - V_{\Lttilde} V_{\Lttilde}^T  }\\
	&\le   \specnorm{V_X^T V_X^T - V_{L} V_{L}^T  }  +  \specnorm{V_{L} V_{L}^T  - V_{\Lttilde} V_{\Lttilde}^T  }\\
	& = \specnorm{V_{X^\perp}^T V_L } + \specnorm{V_{L^\perp}^T V_{\Lttilde} }\\
	&\le 2\delta + \frac{\gamma}{1-\gamma},
	\end{align*}
	where in the last inequality we applied Lemma \ref{lemma:fundam} and Lemma \ref{lemma:weylconsequence}. 
	Hence, by our assumptions on $\delta$ and $\gamma$ 
	we can apply Lemma \ref{lemma:init:closeness2}. Together with the inequality \eqref{ineq:initi5} we obtain
	\begin{align*}
	\sigma_{\min} \left( U_t W_t  \right)  &\ge   \frac{1}{2}  \sigma_{r_{\star}} \left( U_t \right) \ge     \frac{\alpha}{4}   \sigma_{r_{\star}} \left( \Zt  \right) \sigma_{\min} \left(  V_L^T U \right)  
	\end{align*}
	as well as
	\begin{align*}
	\specnorm{ V_{X^\perp}^T  V_{U_tW_t}   } & \le 7 \specnorm{V_{X^\perp}^T  V_{\Lttilde} }  \\ 
	& \le 7 \left( 8 \delta + \frac{\gamma}{1-\gamma}     \right)\\
	&\le 56 \left(  \delta + \gamma   \right).
	\end{align*}
	Moreover, it also follows from Lemma \ref{lemma:init:closeness2}, inequality \eqref{ineq:initi6} and our assumption on $\gamma$ that
	\begin{align*}
	\specnorm{ U_t  W_{t,\perp}  } &\le   2 \sigma_{r_{\star} +1} \left( U_t\right)   \\
	&\le  2 \gamma      \alpha   \sigma_{r_{\star}} \left( \Zt   \right) \sigma_{\min} \left(  V_L^T U \right)\\
	&\le  \frac{ \kappa^{-2}}{8}     \alpha   \sigma_{r_{\star}} \left(  \Zt  \right) \sigma_{\min} \left(  V_L^T U \right).
	\end{align*}
	This finishes the proof.
\end{proof}

\subsection{Proof of Lemma \ref{lemma:spectralmain2}}
\begin{proof}[Proof of Lemma \ref{lemma:spectralmain2}]
	In order to apply Lemma \ref{lemma:Westimates}, we need to show that $ \gamma \le  \widetilde{c}_2 \kappa^{-2} $ for an appropriately chosen $t_{\star}=t$. We are going to show the stronger statement $ \gamma \le c_3 \kappa^{-2} $, where $c_3$ is a sufficiently small constant depending only on $c$, which will be specified later. Note that by the definition of $\gamma$ it suffices to check the following two conditions.
	\begin{align}
	\sigma_{r_{\star} +1 } \left(\Zt \right)  \specnorm{U}&\le \frac{c_3}{2} \sigma_{r_{\star} } \left( \Zt \right) \sigma_{\min } \left( V_L^T U \right)     \kappa^{-2}, \label{ineq:log1}\\
	\specnorm{E_t}&\le \frac{c_3}{2} \alpha \sigma_{r_{\star} } \left(\Zt\right) \sigma_{\min } \left( V_L^T U \right)     \kappa^{-2}. \label{ineq:log2}
	\end{align}
	By using the identity $\Zt = \left(\Id + \mu M\right)^t$ and by rearranging terms we see that the first inequality is equivalent to the inequality
	\begin{align*}
	\frac{2 \kappa^2 \specnorm{U}}{ c_3 \sigma_{\min } \left( V_L^T U \right)   } \le \left( \frac{1+\mu \lambda_{r_{\star}} \left(M\right)  }{1+ \mu \lambda_{r_{\star} +1} \left(M\right)  }   \right)^t .
	\end{align*}	
	Hence, if we set 
	\begin{equation*}
	t_{\star} =  \Big\lceil \bracing{=:\sigma}{ \ln \left(    \frac{2\kappa^2 \specnorm{U}}{ c_3 \sigma_{\min } \left( V_L^T U \right)   }    \right)  } \left(  \ln \left( \frac{1+\mu \lambda_{r_{\star}} \left(M\right)  }{1+ \mu \lambda_{r_{\star} +1} \left(M\right)  }   \right)  \right)^{-1}  \Big\rceil ,
	\end{equation*}
	we see that condition  \eqref{ineq:log1} is satisfied. Let us check that this choice is feasible, i.e. $t_{\star} \le t^{\star}$. By Lemma \ref{lemma:init:closeness} and the definition of $t_{\star}$ it suffices to show that
	\begin{equation*}
	\ln \left(    \frac{2\kappa^2 \specnorm{U}}{ c_3 \sigma_{\min } \left( V_L^T U \right)   }    \right)   \left(  \ln \left( \frac{1+\mu \lambda_{r_{\star}} \left(M\right)  }{1+ \mu \lambda_{r_{\star} +1} \left(M\right)  }   \right)  \right)^{-1}  \le  \frac{ \ln \left(   \frac{\lambda_1 \left(M\right) }{ 4 \alpha^2 \left( 1+\delta_1 \right) \specnorm{U}^3  } \left( \frac{\twonorm{U_0^T v_1}}{\alpha r}  \right)  \right)}{ 8 \ln \left(  1+\mu \lambda_1 \left( M \right)  \right)} .
	\end{equation*}
	Next, we note that
	\begin{equation}\label{ineq:logestimate2}
	\begin{split}
	\frac{\ln  \left(   1+\mu  \lambda_1 \left(M\right)    \right)  }{\ln \left( \frac{1+\mu \lambda_{r_{\star}} \left(M\right)  }{1+ \mu \lambda_{r_{\star} +1} \left(M\right)  }   \right)  }  &= \frac{\ln  \left(   1+\mu  \lambda_1 \left(M\right)    \right)  }{\ln \left(  1+ \frac{\mu \left(  \lambda_{r_{\star} } \left(M \right) - \lambda_{r_{\star} +1 } \left(M \right)    \right) }{1+ \mu \lambda_{r_{\star} +1} \left(M\right)  }  \right)  } \\
	&\le \frac{    \lambda_1 \left(M\right)  \cdot \frac{1+\mu \lambda_{r_{\star}} \left(M\right)  }{1+ \mu \lambda_{r_{\star} +1} \left(M\right)  }    }{ \frac{   \lambda_{r_{\star} } \left(M \right) - \lambda_{r_{\star} +1 } \left(M \right)     }{1+ \mu \lambda_{r_{\star} +1} \left(M\right)  }  }\\
	& =    \frac{\lambda_1 \left(M\right)  \left(    1+ \mu \lambda_{r_{\star} } \left(M\right)     \right)  }{   \lambda_{r_{\star} } \left(M \right) - \lambda_{r_{\star} +1 } \left(M \right)            }    \le  2 \kappa^2 ,
	\end{split}
	\end{equation}
	where in the first inequality we have used the elementary inequality $  \frac{x}{1+x}  \le \ln \left(1+x\right) \le  x  $. in the last inequality we used our assumption on the step size $\mu$, Lemma \ref{lemma:weylconsequence} as well as our assumption on $\delta>0$ with a sufficiently small constant $c_1$. 
	Hence, $t_{\star} \le t^{\star} $ is implied by
	\begin{equation*}
	\ln \left(    \frac{2\kappa^2 \specnorm{U}}{ c_3 \sigma_{\min } \left( V_L^T U \right)   }    \right)  \le  \frac{1}{9 \kappa^2}  \ln \left(   \frac{\lambda_1 \left(M\right) }{ 4 \alpha^2 \left( 1+\delta_1 \right) \specnorm{U}^3  } \left( \frac{\twonorm{U_0^T v_1}}{\alpha \min \left\{ r;n \right\}}  \right)  \right).
	\end{equation*}
	By rearranging terms we see that this inequality is equivalent to
	\begin{align*}
	\alpha^2  &\le    \frac{\lambda_1 \left(M\right) }{ 4  \left( 1+\delta_1 \right) \specnorm{U}^3  } \left( \frac{\twonorm{U_0^T v_1}}{\alpha \min \left\{ r;n \right\}}  \right) \left(   \frac{2\kappa^2 \specnorm{U}}{ c_3 \sigma_{\min } \left( V_L^T U \right)   }   \right)^{-9\kappa^2}.
	\end{align*}
	Since by assumption $\delta_1 < 1$ and since by Lemma \ref{lemma:weylconsequence} we have $\lambda_1 \left(M\right) \ge \frac{1}{2} \specnorm{X}^2$, we observe that this inequality is implied by
	\begin{align*}
	\alpha^2  &\le    \frac{ \specnorm{X}^2 }{ 16\specnorm{U}^3  } \left( \frac{\twonorm{U_0^T v_1}}{\alpha \min \left\{ r;n \right\}}  \right) \left(   \frac{2\kappa^2 \specnorm{U}}{ c_3 \sigma_{\min } \left( V_L^T U \right)   }   \right)^{-9\kappa^2}=  \frac{ \specnorm{X}^2 }{ 16\specnorm{U}^3  } \left( \frac{\twonorm{U^T v_1}}{ \min \left\{ r;n \right\}}  \right) \left(   \frac{2\kappa^2 \specnorm{U}}{ c_3 \sigma_{\min } \left( V_L^T U \right)   }   \right)^{-9\kappa^2},
	\end{align*}
	which follows from assumption \eqref{ineq:assumptiononalpha}, which shows $ t_{\star} \le t^{\star} $.
	
	In order to show condition \eqref{ineq:log2}, we recall that by Lemma \ref{lemma:errorboundspectral} (which we can apply since we just showed $ t_{\star} \le t^{\star} $)
	\begin{align*}
	\specnorm{E_{t_{\star}}} \le    \frac{4}{\lambda_1 \left(M\right)} \alpha^3   \min \left\{ r;n \right\}  \left(1+\delta_1 \right)   \left( 1+\mu \lambda_1 \left( M \right)  \right)^{3 t_{\star}  }  \specnorm{U}^3 .
	\end{align*}
	Hence, inequality \eqref{ineq:log2} is implied by the inequality
	\begin{equation*}
	\frac{8}{\lambda_1 \left(M\right)} \alpha^2   \min \left\{ r;n \right\} \left(1+\delta_1 \right)   \left( 1+\mu \lambda_1 \left( M \right)  \right)^{3 t_{\star} } \specnorm{U}^3  \le c_3  \left(  1+\mu \lambda_{r_{\star}} \left(M\right)  \right)^t \sigma_{\min } \left( V_L^T U \right)     \kappa^{-2}.
	\end{equation*}
	This, in turn, is equivalent to
	\begin{align}\label{ineq:bound_alpha}
	\alpha^2   \le    \frac{ c_3 \lambda_1 \left(M\right)  \sigma_{\min } \left(V_L^T U\right)  }{  8   \min \left\{ r;n \right\} \left(1+\delta_1\right) \kappa^2 \specnorm{U}^3} \left[ \frac{ 1+\mu \lambda_{r_{\star}} \left( M \right) }{ \left( 1+\mu \lambda_{1 } \left(M\right)\right)^3}    \right]^{t_{\star}} .
	\end{align}
	In order to proceed, we note that
	\begin{align*}
	\left[ \frac{ 1+\mu \lambda_{r_{\star}} \left( M \right) }{ \left( 1+\mu \lambda_{1 } \left(M\right)\right)^3}    \right]^{t_{\star}}  & \ge \exp \left( -3	t_{\star} \ln  \left(   1+\mu  \lambda_1 \left(M\right)    \right) \right)  \\
	&\ge   \exp \left( - \sigma \frac{ 6\ln  \left(   1+\mu  \lambda_1 \left(M\right)    \right)  }{\ln \left( \frac{1+\mu \lambda_{r_{\star}} \left(M\right)  }{1+ \mu \lambda_{r_{\star} +1} \left(M\right)  }   \right)  } \right).
	\end{align*}
	Hence, using \eqref{ineq:logestimate2}, we have shown that
	\begin{equation*}
	\left[ \frac{ 1+\mu \lambda_{r_{\star}} \left( M \right) }{ \left( 1+\mu \lambda_{1 } \left(M\right)\right)^3}    \right]^{t_{\star}}   \ge \exp \left(  -12\sigma  \kappa^2\right) .
	\end{equation*}
	Inserting this into \eqref{ineq:bound_alpha} and using the definition of $t_{\star}$, we have shown that inequality \eqref{ineq:log2} holds, if
	\begin{equation*}
	\alpha^2 \le \frac{ c_2  \specnorm{X}^2  \sigma_{\min } \left(V_L^T U\right)  }{  32  r   \kappa \specnorm{U}^3}   \left(    \frac{2\kappa^2  \specnorm{U} }{ c_3 \sigma_{\min } \left( V_L^T U \right)   }    \right)^{-12\kappa^2}
	\end{equation*}
	holds, which is precisely our assumption on $\alpha$. In particular, we have shown that $ \gamma \le c_3 \kappa^{-2} $, which allows us to apply Lemma \ref{lemma:Westimates}. We obtain that
	\begin{align*}
	\sigma_{\min} \left( U_{t_{\star}} W_{t_{\star}}  \right)  &\ge    \frac{\alpha}{4}   \sigma_{r_{\star}} \left( Z_{t_{\star}}  \right) \sigma_{\min} \left(  V_L^T U \right),   \\
	\specnorm{ U_t  W_{t_{\star},\perp} }  &\le  \frac{ \kappa^{-2}}{8}     \alpha   \sigma_{r_{\star}} \left( Z_{t_{\star}}  \right) \sigma_{\min} \left(  V_L^T U \right), \\
	\specnorm{ V_{X^\perp}^T  V_{U_{t_{\star}} W_{t_{\star}}}   } &\le 56 \left( \delta + \gamma  \right) \overleq{(a)} c \kappa^{-2},
	\end{align*}
	where inequality $(a)$ follows from setting $c_1 $ and $ c_3 $ small enough. Setting $ \beta:=   \sigma_{r_{\star}} \left( Z_{t_{\star}} \right) \sigma_{\min} \left(  V_L^T U \right) $ shows inequalities \eqref{ineq:specmain2}, \eqref{ineq:specmain3}, and \eqref{ineq:specmain4}. It remains to verify that $\specnorm{U_{t_{\star}}}$, $t_{\star}$, and $\beta$ have the desired properties. We start with $t_{\star}$. Note that 
	\begin{align*}
	\ln \left( \frac{1+\mu \lambda_{r_{\star}} \left(M\right)  }{1+ \mu \lambda_{r_{\star} +1} \left(M\right)  }   \right)    &\le  \ln \left( 1+\mu \lambda_{r_{\star}} \left(M\right)   \right)  \le \mu \lambda_{r_{\star}} \left(M\right)   \le \mu \left(1+\delta\right) \sigma_{r_{\star} } \left(X\right)^2
	\end{align*}
	as well as
	\begin{align*}
	\ln  \left( \frac{1+\mu \lambda_{r_{\star}} \left(M\right)  }{1+ \mu \lambda_{r_{\star} +1} \left(M\right)  }   \right)   &\ge  \frac{\mu   \lambda_{r_{\star}} \left(M\right)  }{1+ \mu \lambda_{r_{\star}} \left(M\right)  } - \mu  \lambda_{r_{\star}+1} \left(M\right)  \ge \frac{1}{2} \mu \sigma_{r_{\star} } \left(X\right)^2.
	\end{align*}
	Here we have used the inequalities $\frac{x}{1+x} \le \ln \left(1+x\right) \le x $, $  \lambda_{r_{\star} } \left(M\right)  \le \delta \sigma_{\min } \left(X\right)^2$,  and $  \left(1-\delta\right) \sigma_{\min } \left(X\right)^2 \le   \lambda_{r_{\star} } \left(M\right) \le   \left( 1+\delta \right)  \sigma_{r_{\star} } \left(X\right)^2$ from Lemma \ref{lemma:weylconsequence}. Hence, these estimates show that $t_{\star}$ has the desired property.
	
Next, we are going to prove the desired bound for $\specnorm{U_{t_{\star}}}$. We obtain that
	\begin{align*}
	\specnorm{U_{t_{\star}}} &\le \alpha  \specnorm{Z_{t_{\star}}} \specnorm{U} + \specnorm{E_{t_{\star}}}\\
	&= \alpha  \left(  1+ \mu \lambda_{1 } \left(  M \right)  \right)^{t_{\star}} \specnorm{U} + \specnorm{E_{t_{\star}}} \\
	&\overleq{(a)} 2 \alpha  \left(  1+ \mu \lambda_{1 } \left(  M \right)  \right)^{t_{\star}} \specnorm{U}  \\
	&\le   2 \alpha  \exp \left( 2 \sigma   \frac{ \ln  \left(   1+\mu  \lambda_1 \left(M\right)    \right)  }{\ln \left( \frac{1+\mu \lambda_{r_{\star}} \left(M\right)  }{1+ \mu \lambda_{r_{\star} +1} \left(M\right)  }   \right) }   \right)    \specnorm{U}  \\
	&\le  2 \alpha  \exp \left(  4 \sigma \kappa^2 \right)    \specnorm{U}     ,
	\end{align*}
	where $(a)$ follows from \eqref{ineq:log1} and in the last line we used inequality \eqref{ineq:logestimate2}. Hence, by inserting the definition of $\sigma$ we have shown that
	\begin{align*}
	\specnorm{U_{t_{\star}}}  &\le 2\alpha  \left(    \frac{2\kappa^2 \specnorm{U}}{ c_3 \sigma_{\min } \left( V_L^T U \right)   }    \right)^{4\kappa^2} \specnorm{U}\\
	&\overleq{(a)}  2  \sqrt{     \frac{c_2 \specnorm{X}^2   \sigma_{\min } \left(  V_L^T U  \right)   }{32 \min \left\{ r;n \right\} \kappa \specnorm{U} }       }          \left(    \frac{2\kappa^2  \specnorm{U}}{ c_3 \sigma_{\min } \left( V_L^T U \right)   }    \right)^{-2\kappa^2} \\
	&\le 3  \specnorm{X}
	\end{align*}
	where in inequality $(a)$ we have used the assumption on $\alpha$. This shows inequality \eqref{ineq:specmain1}. Now let us check that $\beta$ has the desired property. For that, note that
	\begin{align*}
	\beta &=   \left(    1+ \mu \lambda_{r_{\star} } \left(M\right)  \right)^{t_{\star}}  \sigma_{\min }\left(  V_L^T U \right)  = \sigma_{\min }\left(  V_L^T U \right)          \exp \left( t_{\star}  \ln \left(   1+ \mu \lambda_{r_{\star} } \left(M\right)    \right)     \right).
	\end{align*}
	By inserting the definition of $ t_{\star}$ and using inequality \eqref{ineq:logestimate2} we can show the upper bound for $\beta$ in inequality \eqref{ineq:betabound}. The lower bound follows immediately from the definition of $\beta$. This finishes the proof.
\end{proof}

%% file: convergence_phase_proofs.tex
\section{Proofs for the saddle avoidance phase and the refinement phase}\label{appendix:refinementphase}

\subsection{Proof of Lemma \ref{lemma:Vxgrowth}}

\begin{proof}[Proof of Lemma \ref{lemma:Vxgrowth}]
	Let $\WWt$ and $\Wtperp$ be defined as before. We note that
	\begin{align*}
	V_X^T\UUtplus \WWt  =&  V_X^T \left( \Id + \mu \mathcal{A}^* \mathcal{A} \left(XX^T -\UUt \UUt^T \right) \right) \UUt \WWt \\
	=&  V_X^T \left( \Id + \mu \left(XX^T - \UUt \UUt^T \right)  + \mu \left[\left(  \mathcal{A}^* \mathcal{A} - \Id \right) \left(XX^T -\UUt \UUt^T \right)  \right] \right) \UUt \WWt  \\
	=& V_X^T \UUt \WWt +\mu \Sigma_X^2 V_X^T \UUt \WWt - \mu V_X^T \UUt  \UUt^T  \UUt \WWt + \mu V_X^T  \left[\left(  \mathcal{A}^* \mathcal{A} - \Id \right) \left(XX^T -\UUt \UUt^T \right)  \right]  \UUt \WWt \\
	=& \left(\Id + \mu \Sigma_X^2 \right)V_X^T \UUt \WWt - \mu V_X^T \UUt \WWt \WWt^T \UUt^T  \UUt \WWt + \mu V_X^T  \left[\left(  \mathcal{A}^* \mathcal{A} - \Id \right) \left(XX^T -\UUt \UUt^T \right)  \right]  \UUt \WWt \\
	=& \left(\Id + \mu \Sigma_X^2 \right)V_X^T \UUt \WWt - \mu V_X^T \UUt \WWt \WWt^T \UUt^T  V_X V_X^T \UUt \WWt -\mu V_X^T \UUt \WWt \WWt^T \UUt^T  V_{X^{\perp}} V_{X^{\perp}}^T \UUt \WWt  \\
	&+  \mu V_X^T  \left[\left(  \mathcal{A}^* \mathcal{A} - \Id \right) \left(XX^T -\UUt \UUt^T \right)  \right]  \UUt \WWt\\
	=& \left(\Id + \mu \Sigma_X^2 \right)V_X^T \UUt \WWt \left( \Id -\mu  \WWt^T \UUt^T  V_X V_X^T \UUt \WWt\right) - \mu  V_X^T \UUt \WWt \WWt^T \UUt^T  V_{X^{\perp}} V_{X^{\perp}}^T \UUt \WWt  \\
	&+\mu V_X^T  \left[\left(  \mathcal{A}^* \mathcal{A} - \Id \right) \left(XX^T -\UUt \UUt\right)  \right]  \UUt \WWt + \mu^2 \Sigma_X^2 V_X^T \UUt \WWt \WWt^T \UUt^T  V_X V_X^T \UUt \WWt\\
	=& \left(\Id + \mu \Sigma_X^2 \right)V_X^T \UUt \WWt \left( \Id -\mu  \WWt^T \UUt^T  V_X V_X^T \UUt \WWt\right) - \mu  \bracing{ =: A_1 }{ V_X^T \UUt  \UUt^T  V_{X^{\perp}} V_{X^{\perp}}^T \UUt \WWt } \\
	&+\mu \bracing{ =: A_2}{ V_X^T  \left[\left(  \mathcal{A}^* \mathcal{A} - \Id \right) \left(XX^T -\UUt \UUt^T \right)  \right]  \UUt \WWt } + \mu^2  \bracing{ =:A_3 }{ \Sigma_X^2 V_X^T \UUt \UUt^T  V_X V_X^T \UUt \WWt}.
	\end{align*}
	First, we want to bring all $A_i$  into the form $P_i  V_X^T \UUt \WWt    \left( \Id -\mu  \WWt^T \UUt^T  V_X V_X^T \UUt \WWt\right)$ for $ i \in \left\{ 1;2;3 \right\} $.\\
	
	\noindent \textbf{Rewriting $A_1$:} Now let the singular value decomposition of $\UUtplus \WWt \in \mathbb{R}^{n \times r_{\star}} $ be given by $V_{\UUtplus \WWt}\Sigma_{\UUtplus \WWt} W_{\UUtplus \WWt}^T $ with $ V_{\UUtplus \WWt} \in \mathbb{R}^{n \times r_{\star}}$.  This allows us to compute 
	\begin{align*}
	V_{X^{\perp}}^T \UUt\WWt  &=  V_{X^{\perp}}^T \UUt \WWt  \left(  V_X^T \UUt \WWt\right)^{-1} V_X^T \UUt \WWt \\
	&=  V_{X^{\perp}}^T V_{\UUt \WWt} V_{\UUt \WWt}^T  \UUt \WWt  \left(  V_X^T V_{\UUt \WWt} V_{\UUt \WWt}^T \UUt \WWt\right)^{-1} V_X^T U\WWt  \\
	&=  V_{X^{\perp}}^T V_{\UUt \WWt} \left(  V_X^T V_{\UUt \WWt}\right)^{-1} V_X^T \UUt \WWt.
	\end{align*}
	We compute that
	\begin{equation}
	\begin{split}
	&V_X^T \UUt  \UUt^T  V_{X^{\perp}} V_{X^{\perp}}^T \UUt \WWt\\
	=&  V_X^T \UUt  \UUt^T  V_{X^{\perp}}   V_{X^{\perp}}^T V_{\UUt \WWt} \left(  V_X^T V_{\UUt \WWt}\right)^{-1} V_X^T \UUt \WWt\\
	=&V_X^T \UUt  \UUt^T  V_{X^{\perp}}   V_{X^{\perp}}^T V_{\UUt \WWt} \left(  V_X^T V_{\UUt \WWt}\right)^{-1} V_X^T \UUt \WWt  \left( \Id -\mu  \WWt^T \UUt^T  V_X V_X^T \UUt \WWt\right)^{-1}  \left( \Id -\mu  \WWt^T \UUt^T  V_X V_X^T \UUt \WWt\right)\\
	=&  V_X^T \UUt  \UUt^T  V_{X^{\perp}}   V_{X^{\perp}}^T V_{\UUt \WWt} \left(  V_X^T V_{\UUt \WWt}\right)^{-1} \left( \Id -\mu  V_X^TU\WWt \WWt^T \UUt V_X\right)^{-1} V_X^T \UUt\WWt    \left( \Id -\mu  \WWt^T \UUt^T  V_X V_X^T \UUt \WWt\right)\\
	=& \bracing{=:P_1}{  V_X^T \UUt  \UUt^T  V_{X^{\perp}}   V_{X^{\perp}}^T V_{\UUt \WWt} \left(  V_X^T V_{\UUt \WWt}\right)^{-1} \left( \Id -\mu  V_X^T \UUt \UUt^T  V_X\right)^{-1} } V_X^T \UUt \WWt    \left( \Id -\mu  \WWt^T \UUt^T  V_X V_X^T \UUt \WWt\right).
	\end{split}
	\end{equation}
	\textbf{Rewriting $A_2$:} We observe that
	\begin{align*}
	\UUt \WWt  &=   V_{\UUt \WWt} V_{\UUt \WWt}^T    \UUt \WWt \left( V_X^TV_{\UUt \WWt} V_{\UUt \WWt}^T   \UUt \WWt\right)^{-1} V_X^T U_t\WWt\\
	&=  V_{\UUt \WWt} \left( V_X^TV_{\UUt \WWt}\right)^{-1}   V_X^T \UUt \WWt.
	\end{align*} 
	Hence, we can write
	\begin{align*}
	&V_X^T  \left[\left(  \mathcal{A}^* \mathcal{A} - \Id \right) \left(XX^T -\UUt \UUt^T \right)  \right]  \UUt \WWt\\
	=& V_X^T  \left[\left(  \mathcal{A}^* \mathcal{A} - \Id \right) \left(XX^T -\UUt \UUt^T \right)  \right]   V_{\UUt \WWt} \left( V_X^TV_{\UUt \WWt}\right)^{-1}  V_X^T \UUt \WWt\\
	=&  V_X^T  \left[\left(  \mathcal{A}^* \mathcal{A} - \Id \right) \left(XX^T -\UUt \UUt^T \right)  \right]   V_{\UUt \WWt} \left( V_X^TV_{\UUt \WWt}\right)^{-1}  V_X^T \UUt \WWt \left( \Id -\mu  \WWt^T \UUt^T  V_X V_X^T \UUt \WWt\right)^{-1}  \\
	&\cdot   \left( \Id -\mu  \WWt^T \UUt^T  V_X V_X^T \UUt \WWt\right)\\
	=& \bracing{ =: P_2}{ V_X^T  \left[\left(  \mathcal{A}^* \mathcal{A} - \Id \right) \left(XX^T -\UUt \UUt^T \right)  \right]   V_{\UUt \WWt} \left( V_X^TV_{\UUt \WWt}\right)^{-1} \left( \Id -\mu  V_X^T \UUt  \WWt \WWt^T \UUt^T  V_X \right)^{-1} } V_X^T \UUt \WWt   \\
	&\cdot   \left( \Id -\mu  \WWt^T \UUt^T  V_X V_X^T \UUt \WWt\right).
	\end{align*}
	\textbf{Rewriting $A_3$:} Note that
	\begin{align*}
	&\Sigma_X^2 V_X^T \UUt \UUt^T  V_X V_X^T \UUt \WWt\\
	=& \Sigma_X^2 V_X^T \UUt \WWt \WWt^T\UUt^T  V_X V_X^T \UUt \WWt \\
	=& \Sigma_X^2 V_X^T \UUt \WWt \WWt^T\UUt^T  V_X V_X^T \UUt \WWt  \left( \Id -\mu  \WWt^T \UUt^T  V_X V_X^T \UUt \WWt\right)^{-1} \left( \Id -\mu  \WWt^T \UUt^T  V_X V_X^T \UUt \WWt\right) \\
	=&  \bracing{=:P_3}{ \Sigma_X^2 V_X^T \UUt \WWt\left( \Id -\mu  \WWt^T \UUt^T  V_X V_X^T \UUt  \WWt\right)^{-1}   \WWt^T\UUt^T  V_X} V_X^T \UUt \WWt \left( \Id -\mu  \WWt^T \UUt V_X V_X^T \UUt \WWt\right).
	\end{align*}
	Hence, we have computed that
	\begin{equation}\label{ineq:intern456}
	V_X^T\UUtplus \WWt  =   \left(\Id + \mu \Sigma_X^2 - \mu P_1+ \mu P_2 + \mu^2 P_3 \right)V_X^T \UUt \WWt \left( \Id -\mu  \WWt^T \UUt^T  V_X V_X^T U\WWt\right) .
	\end{equation}
	It follows that
	\begin{equation}\label{ineq:intern459}
	\begin{split}
	&\sigma_{\min} \left( V_X^T\UUtplus \WWt   \right)\\
	\ge&  \sigma_{\min} \left( \Id + \mu \Sigma_X^2 - \mu P_1+ \mu P_2 + \mu^2 P_3  \right)  \sigma_{\min} \left(V_X^T \UUt \WWt \left( \Id -\mu  \WWt^T \UUt^T  V_X V_X^T \UUt \WWt\right)  \right)\\
	\overeq{(a)} & \sigma_{\min} \left( \Id + \mu \Sigma_X^2 - \mu P_1+ \mu P_2 + \mu^2 P_3  \right)  \sigma_{\min} \left( V_X^T \UUt \WWt \right) \left( 1 -\mu  \sigma_{\min}^2 \left( V_X^T \UUt \WWt \right) \right) \\
	= & \sigma_{\min} \left( \Id + \mu \Sigma_X^2 - \mu P_1+ \mu P_2 + \mu^2 P_3  \right)  \sigma_{\min} \left( V_X^T \UUt \right) \left( 1 -\mu  \sigma_{\min}^2 \left( V_X^T \UUt \right) \right)\\
	\overgeq{(b)} &  \left( \sigma_{\min} \left( \Id + \mu \Sigma_X^2  \right) - \mu \specnorm{P_1} -  \mu \specnorm{P_2} - \mu^2 \specnorm{P_3}  \right) \sigma_{\min} \left( V_X^T \UUt \right) \left( 1 -\mu  \sigma_{\min}^2 \left( V_X^T \UUt \right) \right)\\
	= &  \left( 1 + \mu \sigma_{\min}^2 \left(X\right) - \mu \specnorm{P_1} -  \mu \specnorm{P_2} - \mu^2 \specnorm{P_3}  \right) \sigma_{\min} \left( V_X^T \UUt \right) \left( 1 -\mu  \sigma_{\min}^2 \left( V_X^T \UUt \right) \right).
	\end{split}
	\end{equation}
	Equality $(a)$ can be obtained by using the singular value decomposition of $V_X^T \UUt \WWt $ and the fact that $\mu \le 1/\left( \sqrt{3 } \specnorm{V_X^T \UUt}^2 \right) $, which follows from our assumption on $\mu$. For inequality $(b)$ we used Weyl's inequality. In order to proceed, we are going to estimate $ \specnorm{P_1} $, $\specnorm{P_2} $, and $\specnorm{P_3} $. First, we note that
	\begin{align*}
	\specnorm{P_1} &\overleq{(a)}  \specnorm{ V_X^T \UUt \WWt \WWt^T  \UUt^T  V_{X^{\perp}}   V_{X^{\perp}}^T V_{\UUt W} } \specnorm{ \left(  V_X^T V_{\UUt \WWt}\right)^{-1} } \specnorm{ \left( \Id -\mu  V_X^T \UUt \UUt^T  V_X\right)^{-1} }\\
	&\le  \specnorm{  \UUt \WWt} \specnorm{ V_{X^{\perp}}^T \UUt \WWt  } \specnorm{   V_{X^{\perp}}^T V_{\UUt \WWt} } \specnorm{ \left(  V_X^T V_{\UUt \WWt}\right)^{-1} } \specnorm{ \left( \Id -\mu  V_X^T \UUt \UUt^T  V_X\right)^{-1} } \\
	& \le \specnorm{  \UUt \WWt}^2 \specnorm{ V_{X^{\perp}}^T V_{\UUt \WWt} }^2 \specnorm{ \left(  V_X^T V_{\UUt \WWt}\right)^{-1} } \specnorm{ \left( \Id -\mu  V_X^T \UUt \UUt^T  V_X\right)^{-1} }\\
	&= \frac{  \specnorm{  \UUt \WWt}^2 \specnorm{ V_{X^{\perp}}^T V_{\UUt \WWt} }^2 }  {\sigma_{\min} \left( V_X^T V_{\UUt \WWt} \right)  \sigma_{\min} \left( \Id -\mu  V_X^T \UUt \UUt^T  V_X \right) } \\
	&=  \frac{  \specnorm{  \UUt \WWt}^2 \specnorm{ V_{X^{\perp}}^T V_{\UUt \WWt} }^2 }  {\sigma_{\min} \left( V_X^T V_{\UUt \WWt} \right)  \left(1 - \mu \specnorm{V_X^T \UUt}^2 \right) } \\
	&\overleq{(b)} 4 \specnorm{  \UUt \WWt}^2 \specnorm{ V_{X^{\perp}}^T V_{\UUt \WWt} }^2  \\
	& \overleq{(c)} 36  \specnorm{  X}^2 \specnorm{ V_{X^{\perp}}^T V_{\UUt \WWt} }^2 \\
	&\overleq{(d)}  \frac{1}{4} \sigma_{\min} \left(X\right)^2.
	\end{align*}
	For inequality $(a)$ we used the submultiplicativity of the spectral norm and the fact that $V_X^T \UUt=  V_X^T \UUt \WWt \WWt^T$. In  $(b)$ we used the assumption $ \specnorm{V_{X^\perp}^T V_{\UUt \WWt} } \le c \kappa^{-1} $ and $ \mu \le c_1 \specnorm{X}^{-2} \kappa^{-2} \le c \specnorm{V_X^T \UUt}^{-2}/9  $. In  inequality $(c)$ we used the assumption $ \specnorm{\UUt} \le 3 \specnorm{X} $. Inequality $(d)$ follows from the assumption $ \specnorm{ V_{X^{\perp}}^T V_{\UUt \WWt} } \le  c \kappa^{-1} $, where the constant $c$ is chosen sufficiently small.
	
	In order to estimate $\specnorm{P_2}$ we note that
	\begin{align*}
	\specnorm{P_2} & \overleq{(a)} \specnorm{ \left[\left(  \mathcal{A}^* \mathcal{A} - \Id \right) \left(XX^T -\UUt \UUt^T \right)  \right] }   \specnorm{ \left( V_X^TV_{\UUt \WWt}\right)^{-1} } \specnorm{ \left( \Id -\mu  V_X^T \UUt  \UUt^T  V_X \right)^{-1} } \\
	&=   \frac{ \specnorm{ \left[\left(  \mathcal{A}^* \mathcal{A} - \Id \right) \left(XX^T -\UUt \UUt^T \right)  \right] } }{\sigma_{\min} \left( V_X^TV_{\UUt \WWt} \right)  \left(1 - \mu \specnorm{V_X^T \UUt}^2\right) } \\
	&\overleq{(b)} 4  \specnorm{ \left[\left(  \mathcal{A}^* \mathcal{A} - \Id \right) \left(XX^T -\UUt \UUt^T \right)  \right] }.
	\end{align*}
	In $(a)$ we used the submultiplicativity of the spectral norm. In inequality $(b)$ we used the assumption $ \specnorm{V_{X^\perp}^T V_{\UUt \WWt} } \le c \kappa^{-1} $ and $ \mu \le c \specnorm{X}^{-2} \kappa^{-2} \le c \specnorm{V_X^T \UUt}^{-2}/9  $. Next, we are going to estimate $ \specnorm{P_3}$ by
	\begin{align*}
	\specnorm{P_3} &\le \specnorm{ \Sigma_X^2 } \specnorm{ V_X^T \UUt \WWt} \specnorm{ \left( \Id -\mu  \WWt^T \UUt^T  V_X V_X^T \UUt \WWt\right)^{-1} }  \specnorm { \WWt^T\UUt^T  V_X }  \\
	&= \frac{  \specnorm{ X }^2 \specnorm{ V_X^T \UUt \WWt}^2  }{1- \mu \specnorm{ V_X^T \UUt \WWt}^2}\\
	&\le 2 \specnorm{ X }^2 \specnorm{ V_X^T \UUt \WWt}^2 \\
	&\le 2 \specnorm{ X }^2 \specnorm{ \UUt \WWt}^2 \\
	& \le 18 \specnorm{ X }^4.
	\end{align*}
	In the last line we used the assumption $ \specnorm{\UUt} \le 3 \specnorm{X} $. Inserting our estimates for $\specnorm{P_1} $, $\specnorm{P_2} $, and $\specnorm{P_3} $ into \eqref{ineq:intern459} we obtain that
	\begin{align*}
	\sigma_{\min} \left( V_X^T\UUtplus \WWt   \right) \ge & \left( 1 + \frac{3}{4} \mu \sigma_{\min} \left(X\right)^2 -  4 \mu \specnorm{ \left[\left(  \mathcal{A}^* \mathcal{A} - \Id \right) \left(XX^T -\UUt \UUt^T \right)  \right] } - 18 \mu^2 \specnorm{ X }^4 \right)  \\ 
	&\sigma_{\min} \left( V_X^T \UUt \right) \left( 1 -\mu  \sigma_{\min}^2 \left( V_X^T \UUt \right) \right)\\
	\overgeq{(a)} & \left( 1 + \frac{1}{2} \mu \sigma_{\min}^2 \left(X\right)  \right)  \sigma_{\min} \left( V_X^T \UUt \right) \left( 1 -\mu  \sigma_{\min}^2 \left( V_X^T \UUt \right) \right)\\
	=&  \sigma_{\min} \left( V_X^T \UUt \right)  \left( 1 +  \frac{1}{2} \mu \sigma_{\min} \left(X\right)^2  \left(1 - \mu \sigma_{\min}^2 \left( V_X^T \UUt \right) \right)  -  \mu  \sigma_{\min}^2 \left( V_X^T \UUt \right)  \right) \\
	\overgeq{(b)}&  \sigma_{\min} \left( V_X^T \UUt \right)  \left( 1 +  \frac{1}{4} \mu \sigma_{\min}^2 \left(X\right)   -  \mu  \sigma_{\min}^2 \left( V_X^T \UUt \right)  \right).
	\end{align*}
	Inequality $(a)$ follows from assumption \eqref{ineq:intern414} and the assumption $ \mu \le c \kappa^{-2} \specnorm{X}^{-2} $. Inequality $(b)$ is a consequence of our assumption on the step size $\mu$ and the assumption $ \specnorm{\UUt} \le 3 \specnorm{X} $.
	The final claim follows from the observation that $ \sigma_{\min} \left( V_X^T\UUtplus  \right)\ge \sigma_{\min} \left( V_X^T\UUtplus \WWt \right)    $.
\end{proof}

\subsection{Proof of Lemma \ref{lemma:specnormperp}}
Before we can prove Lemma \ref{lemma:specnormperp}, we first need the following technical lemma.
\begin{lemma}\label{lemma:auxiliary2}
	Suppose that the assumptions of Lemma \ref{lemma:specnormperp} are fulfilled with a small enough constant $c>0$. Then we have that
	\begin{equation}\label{ineq:intern5678}
	\specnorm{  V_{X^{\perp}}^T  V_{\UUtplus \WWt}} \le 2	\specnorm{  V_{X^{\perp}}^T  V_{\UUt \WWt}} + 2 \mu \specnorm{\left( \mathcal{ A}^* \mathcal{ A}  \left(XX^T - \UUt \UUt^T\right)\right) }.
	\end{equation}
	In particular, it holds that $\specnorm{  V_{X^{\perp}}^T  V_{\UUtplus \WWt}}  \le 1/50$.
\end{lemma}

\begin{proof}
	We note that
	\begin{align*}
	\UUtplus \WWt = \left( \Id +  \mu \mathcal{ A}^* \mathcal{ A}  \left(XX^T - \UUt \UUt^T\right)\right)   \UUt \WWt.
	\end{align*}
	Let $V_{\UUt \WWt} \Sigma_{\UUt \WWt} W_{\UUt W}^T = \UUt \WWt $ be the singular value decomposition of $\UUt \WWt$.  Set
	\begin{equation*}
	Z:=  \left( \Id + \mu \mathcal{ A}^* \mathcal{ A}  \left(XX^T - \UUt \UUt^T\right)\right)  V_{\UUt \WWt}.
	\end{equation*}
	Since  $\Sigma_{\UUt \WWt} W_{\UUt W}^T$ has full rank by assumption, the matrix $Z=V_Z \Sigma_Z W_Z^T$ has the same column space as the matrix $\UUtplus \WWt $. In particular, it follows that
	\begin{align*}
	\specnorm{  V_{X^{\perp}}^T  V_{\UUtplus \WWt}}  &=  	\specnorm{  V_{X^{\perp}}^T  V_Z} \\
	&\le \specnorm{  V_{X^{\perp}}^T  V_Z  \Sigma_Z W_Z^T} \specnorm{ \left(  \Sigma_Z W_Z^T \right)^{-1} }\\
	&= \specnorm{  V_{X^{\perp}}^T  Z} \specnorm{  Z^{-1} } \\
	&= \frac{\specnorm{  V_{X^{\perp}}^T  Z} }{\sigma_{\min} \left(Z\right)}.
	\end{align*}
	By Weyl's inequality it holds that
	\begin{align*}
	\sigma_{\min} \left(Z\right) &\ge \sigma_{\min} \left(V_{\UUt \WWt}\right) - \mu \specnorm{\left( \mathcal{ A}^* \mathcal{ A}  \left(XX^T - \UUt \UUt^T\right)\right)  V_{\UUt \WWt}}\\
	&=1-  \mu  \specnorm{\left( \mathcal{ A}^* \mathcal{ A}  \left(XX^T - \UUt \UUt^T\right)\right)  V_{\UUt \WWt}}\\
	&=1-  \mu  \specnorm{\left( \mathcal{ A}^* \mathcal{ A}  \left(XX^T - \UUt \UUt^T\right)\right) }\\
	&\ge 1/2,
	\end{align*}
	where in the last inequality we used the assumption on the step size $\mu$. Moreover, note that
	\begin{align*}
	\specnorm{  V_{X^{\perp}}^T  Z}   &\le  	\specnorm{  V_{X^{\perp}}^T  V_{\UUt \WWt}} +  \mu \specnorm{\left( \mathcal{ A}^* \mathcal{ A}  \left(XX^T - \UUt \UUt^T\right)\right) }.
	\end{align*}
	This implies inequality \eqref{ineq:intern5678}. Using the assumptions on $ 	\specnorm{  V_{X^{\perp}}^T  V_{\UUt \WWt}} $ and $\mu$, where the constant $c$ is chosen small enough, it follows that $ \specnorm{  V_{X^{\perp}}^T  V_{\UUtplus \WWt}} \le 1/50$, which finishes the proof.
\end{proof}
With all ingredients in place, we can give a proof of Lemma \ref{lemma:specnormperp}.
\begin{proof}[Proof of Lemma \ref{lemma:specnormperp}]
	As a first step we are going to establish a formula for $W_t^T  \Wtplusperp $. Recall that $V_X^T  \UUtplus  \Wtplusperp =0$ due to the definition of $\Wtplusperp $. Since $\WWt \WWt^T + \Wtperp \Wtperp^T =\Id $ we obtain that
	\begin{equation*}
	V_X^T  \UUtplus \WWt \WWt^T  \Wtplusperp = -V_X^T  \UUtplus \Wtperp \Wtperp^T  \Wtplusperp,
	\end{equation*}
	or, equivalently,
	\begin{equation}\label{eq:intern316}
	\WWt^T  \Wtplusperp = - \left(V_X^T  \UUtplus \WWt \right)^{-1} V_X^T  \UUtplus \Wtperp \Wtperp^T  \Wtplusperp.
	\end{equation}
	\noindent Now recall that we want to bound $  \specnorm{ \UUtplus \Wtplusperp }$ from above. Note that using $V_{X}^T \UUtplus \Wtplusperp = 0$ we have
	\begin{align*}
	\UUtplus \Wtplusperp=V_{X}V_{X}^T\UUtplus \Wtplusperp+V_{X^{\perp}}V_{X^{\perp}}^T\UUtplus \Wtplusperp=V_{X^{\perp}}V_{X^{\perp}}^T\UUtplus \Wtplusperp,
	\end{align*}
	which implies that $ \specnorm{ \UUtplus \Wtplusperp} = \specnorm{ V_{X^{\perp}}^T \UUtplus \Wtplusperp} $. Due to $\WWt \WWt^T + \Wtperp\Wtperp^T = \Id $ we have that
	\begin{equation}\label{equation:intern312}
	V_{X^{\perp}}^T \UUtplus \Wtplusperp = \bracing{=(a)}{  V_{X^{\perp}}^T \UUtplus \WWt \WWt^T \Wtplusperp} + \bracing{=(b)}{  V_{X^{\perp}}^T \UUtplus \Wtperp \Wtperp^T \Wtplusperp}.
	\end{equation}
	We are going to consider the two summands individually.\\
	
	\noindent \textbf{Summand $(a)$:} 
	We note that from \eqref{eq:intern316} it follows that
	\begin{equation*}
	V_{X^{\perp}}^T \UUtplus \WWt \WWt^T \Wtplusperp=-V_{X^{\perp}}^T \UUtplus \WWt   \left(V_X^T  \UUtplus \WWt\right)^{-1} V_X^T  \UUtplus \Wtperp \Wtperp^T  \Wtplusperp.
	\end{equation*}
	Let the singular value decomposition of $\UUtplus \WWt \in \mathbb{R}^{n \times r_{\star}} $ be given by $V_{\UUtplus \WWt}\Sigma_{\UUtplus W_t} W_{\UUtplus \WWt}^T $ with $ V_{\UUtplus \WWt} \in \mathbb{R}^{n \times r_{\star}}$. By assumption we have that $ V_X^T	 \UUtplus \WWt $ is invertible, which also implies that $ \UUtplus \WWt $ has full-rank. Hence, we can compute that
	\begin{align*}
	V_{X^{\perp}}^T  \UUtplus \WWt  \left(V_{X}^T  \UUtplus \WWt\right)^{-1} 
	&=   V_{X^{\perp}}^T  V_{\UUtplus \WWt}  V_{\UUtplus \WWt}^T \UUtplus \WWt \left(  V_X^T  V_{\UUtplus \WWt} V_{\UUtplus \WWt}^T  \UUtplus \WWt \right)^{-1} \\
	&=   V_{X^{\perp}}^T  V_{\UUtplus \WWt}  V_{\UUtplus \WWt}^T \UUtplus \WWt \left(   V_{\UUtplus \WWt}^T  \UUtplus \WWt \right)^{-1} \left( V_X^T  V_{\UUtplus \WWt}  \right)^{-1} \\
	&= V_{X^{\perp}}^T  V_{\UUtplus \WWt}   \left( V_{X}^T  V_{\UUtplus \WWt}  \right)^{-1},
	\end{align*}
	which shows that
	\begin{equation*}
	V_{X^{\perp}}^T \UUtplus \WWt \WWt^T \Wtplusperp= -  V_{X^{\perp}}^T  V_{\UUtplus \WWt}   \left( V_{X}^T  V_{\UUtplus \WWt}  \right)^{-1} V_X^T  \UUtplus \Wtperp \Wtperp^T  \Wtplusperp.
	\end{equation*}
	Moreover, we note that
	\begin{align*}
	V_X^T \UUtplus \Wtperp & =V_X^T \UUt \Wtperp +\mu V_X^T  \left[ \mathcal{A}^* \mathcal{A}  \left(XX^T -\UUt \UUt^T\right) \right] \UUt \Wtperp  \\
	&= V_X^T \UUt \Wtperp +\mu V_X^T \left(XX^T -\UUt \UUt^T\right) \UUt \Wtperp  +\mu  V_X^T \left[ \left(  \mathcal{A}^* \mathcal{A} -\Id \right) \left(XX^T -\UUt \UUt^T\right) \right] \UUt \Wtperp  \\
	&\overeq{(a)} -\mu V_X^T \UUt \UUt^T \UUt \Wtperp + \mu V_X^T  \left[ \left(  \mathcal{A}^* \mathcal{A} -\Id \right) \left(XX^T -\UUt \UUt^T\right) \right] \UUt \Wtperp  \\
	&= \mu V_X^T  \left[ - \UUt \UUt^T   +   \left[  \left(   \mathcal{A}^* \mathcal{A} -\Id \right) \left(XX^T -\UUt \UUt^T\right) \right]  \right] \UUt \Wtperp.
	\end{align*}
	In equality $(a)$ we used that $V_X^T \UUt \Wtperp=0 $ and $ X^T U\Wtperp=0$, which follows from the definition of $\Wtperp$.  Hence, we have shown that
	\begin{align*}
	&V_{X^{\perp}}^T \UUtplus \WWt \WWt^T \Wtplusperp \\
	=& \mu  V_{X^{\perp}}^T  V_{\UUtplus \WWt}   \left( V_{X}^T  V_{\UUtplus \WWt}  \right)^{-1} V_X^T \left[  \UUt   \UUt^T   -  \left[   \left(   \mathcal{A}^* \mathcal{A} -\Id \right) \left(XX^T -\UUt \UUt^T\right) \right]  \right] \UUt \Wtperp \Wtperp^T  \Wtplusperp   \\
	\overeq{(b)}& \mu  V_{X^{\perp}}^T  V_{\UUtplus \WWt}   \left( V_{X}^T  V_{\UUtplus \WWt}  \right)^{-1} \bracing{=:M_1}{  V_X^T \left[   \UUt \UUt^T V_{X^\perp}     -  \left[  \left(   \mathcal{A}^* \mathcal{A} -\Id \right) \left(XX^T -\UUt \UUt^T\right) \right]  V_{X^\perp}  \right] }V^T_{X^{\perp}} \UUt \Wtperp \Wtperp^T  \Wtplusperp \\
	=& \mu  V_{X^{\perp}}^T  V_{\UUtplus \WWt}   \left( V_{X}^T  V_{\UUtplus \WWt}  \right)^{-1} M_1 V^T_{X^{\perp}} \UUt \Wtperp \Wtperp^T  \Wtplusperp
	\end{align*}
	In equality $(b)$ above we used that $ V_{X} V_{X}^T \UUt \Wtperp=0   $, which is a consequence of $V_X^T \UUt \Wtperp=0  $. It follows that
	\begin{align*}
	\specnorm{ 	V_{X^{\perp}}^T \UUtplus \WWt \WWt^T \Wtplusperp } &\le \mu \specnorm{  V_{X^{\perp}}^T  V_{\UUtplus \WWt}} \specnorm{ \left( V_{X}^T  V_{\UUtplus \WWt}  \right)^{-1}   } \specnorm{ M_1  }  \specnorm{  V_{X^{\perp}}^T \UUt \Wtperp } \specnorm{\Wtperp^T  \Wtplusperp}\\
	&\le \mu \specnorm{  V_{X^{\perp}}^T  V_{\UUtplus \WWt}} \specnorm{ \left( V_{X}^T  V_{\UUtplus \WWt}  \right)^{-1}   } \specnorm{ M_1  }  \specnorm{  V_{X^{\perp}}^T \UUt \Wtperp } \\
	&= \mu \frac{\specnorm{  V_{X^{\perp}}^T  V_{\UUtplus \WWt}}  \specnorm{M_1} \specnorm{  V_{X^{\perp}}^T \UUt \Wtperp }}{ \sigma_{\min}   \left( V_{X}^T  V_{\UUtplus \WWt}   \right)}.
	\end{align*}
	In order to proceed we note that by Lemma \ref{lemma:auxiliary2} it holds that $ \specnorm{ V_{X^{\perp}}^T  V_{\UUtplus W} } \le 1/50  $. This implies that
	\begin{align*}
	\sigma_{\min}   \left( V_{X}^T  V_{\UUtplus \WWt}   \right) &= \sqrt{ \sigma_{\min}   \left( V_{\UUtplus \WWt}^T  V_X V_{X}^T  V_{\UUtplus \WWt}   \right)} \\
	&= \sqrt{ \sigma_{\min}   \left(  V_{\UUtplus \WWt}^T \left( \Id -  V_{X^{\perp}} V_{X^{\perp}}^T  \right)   V_{\UUtplus \WWt}   \right) } \\
	&= \sqrt{ 1 - \specnorm{ V_{\UUtplus \WWt}^T  V_{X^{\perp}}  V_{X^{\perp}}^T  V_{\UUtplus \WWt}  }  }\\
	&=  \sqrt{ 1 - \specnorm{ V_{X^{\perp}}^T  V_{\UUtplus \WWt}  }^2  }\\
	&\ge 1/2.
	\end{align*}
	Hence, we have shown that
	\begin{align*}
	\specnorm{ 	V_{X^{\perp}}^T \UUtplus \WWt \WWt^T \Wtplusperp}  &\le   2 \mu \specnorm{  V_{X^{\perp}}^T  V_{\UUtplus \WWt}}  \specnorm{M_1} \specnorm{  V_{X^{\perp}}^T \UUt \Wtperp }.
	\end{align*}
	We can estimate $\specnorm{M_1}$ by
	\begin{align*}
	\specnorm{M_1} &\overleq{(a)} \specnorm{ V_X^T \UUt \UUt^T V_{X^\perp}    } + \specnorm{\left[  \left(   \mathcal{A}^* \mathcal{A} -\Id \right) \left(XX^T -\UUt \UUt^T\right) \right]  V_{X^\perp}}\\
	&\overeq{(b)} \specnorm{ V_X^T \UUt \WWt \WWt^T\UUt^T V_{X^\perp}    } + \specnorm{\left[  \left(   \mathcal{A}^* \mathcal{A} -\Id \right) \left(XX^T -\UUt \UUt^T\right) \right]  V_{X^\perp}}\\
	&\le  \specnorm{ V_X^T \UUt \WWt} \specnorm{ V^T_{X^\perp} \UUt \WWt } + \specnorm{\left[  \left(   \mathcal{A}^* \mathcal{A} -\Id \right) \left(XX^T -\UUt \UUt^T\right) \right]  V_{X^\perp}}\\
	&\le  \specnorm{ V_X^T \UUt \WWt } \specnorm{ V^T_{X^\perp} \UUt \WWt   } + \specnorm{ \left(   \mathcal{A}^* \mathcal{A} -\Id \right) \left(XX^T -\UUt \UUt^T\right)}\\
	&\le  \specnorm{ V_X^T \UUt \WWt } \specnorm{V^T_{X^\perp}  V_{\UUt \WWt}} \specnorm{ \UUt  \WWt   } + \specnorm{ \left(   \mathcal{A}^* \mathcal{A} -\Id \right) \left(XX^T -\UUt \UUt^T\right)}\\
	&\le \specnorm{V^T_{X^\perp}  V_{\UUt \WWt}}  \specnorm{\UUt \WWt }^2 + \specnorm{ \left(   \mathcal{A}^* \mathcal{A} -\Id \right) \left(XX^T -\UUt \UUt^T\right)}.
	\end{align*}
	In inequality $(a)$ we used the triangle inequality and in equality $(b)$ we used that $V_X^T \UUt = V_X^T \UUt \WWt \WWt^T $. Hence, we have shown that
	\begin{align*}
	&\specnorm{V_{X^{\perp}}^T \UUtplus \WWt \WWt^T \Wtplusperp}\\
	 \le&  2 \mu \left(  \specnorm{V^T_{X^\perp}  V_{\UUt \WWt}}  \specnorm{\UUt \WWt}^2 + \specnorm{ \left(   \mathcal{A}^* \mathcal{A} -\Id \right) \left(XX^T -\UUt \UUt^T\right)} \right) \specnorm{  V_{X^{\perp}}^T  V_{\UUtplus \WWt }}  \specnorm{  V_{X^{\perp}}^T \UUt \Wtperp }\\
	\le &  2 \mu \left(  \specnorm{V^T_{X^\perp}  V_{\UUt \WWt}}  \specnorm{\UUt \WWt}^2 + \specnorm{ \left(   \mathcal{A}^* \mathcal{A} -\Id \right) \left(XX^T -\UUt \UUt^T\right)} \right) \specnorm{  V_{X^{\perp}}^T  V_{\UUtplus \WWt}}  \specnorm{ \UUt \Wtperp } \\
	\le&  2 \mu \left(  9 \specnorm{V^T_{X^\perp}  V_{\UUt \WWt}}  \specnorm{X}^2 + \specnorm{ \left(   \mathcal{A}^* \mathcal{A} -\Id \right) \left(XX^T -\UUt \UUt^T\right)} \right) \specnorm{  V_{X^{\perp}}^T  V_{\UUtplus \WWt}}  \specnorm{ \UUt \Wtperp } \\
	\overleq{(a)} & 4 \mu \left(  9 \specnorm{V^T_{X^\perp}  V_{\UUt \WWt}}  \specnorm{X}^2 + \specnorm{ \left(   \mathcal{A}^* \mathcal{A} -\Id \right) \left(XX^T -\UUt \UUt^T\right)}  \right) \\
	 &\cdot \left( \specnorm{ V_{X^{\perp}}^T  V_{\UUt \WWt}  } + \mu \specnorm{ \left(   \mathcal{A}^* \mathcal{A}  \right) \left(XX^T -\UUt  \UUt^T\right) }   \right)  \specnorm{ \UUt \Wtperp }\\ 
	 \le & 4 \mu \left(  9 \specnorm{V^T_{X^\perp}  V_{\UUt \WWt}}  \specnorm{X}^2 + \specnorm{ \left(   \mathcal{A}^* \mathcal{A} -\Id \right) \left(XX^T -\UUt \UUt^T\right)}  \right) \\
	 &\cdot \left( \specnorm{ V_{X^{\perp}}^T  V_{\UUt \WWt}  } + \mu \specnorm{ \left(   \mathcal{A}^* \mathcal{A} -\Id \right) \left(XX^T -\UUt  \UUt^T\right) }  + \mu \specnorm{XX^T -\UUt  \UUt^T}  \right)  \specnorm{ \UUt \Wtperp }\\ 
	 	 	\overleq{(b)} & 4 \mu \left(  9 \specnorm{V^T_{X^\perp}  V_{\UUt \WWt}}  \specnorm{X}^2 + \specnorm{ \left(   \mathcal{A}^* \mathcal{A} -\Id \right) \left(XX^T -\UUt \UUt^T\right)}  \right) \\
	 &\cdot \left( \specnorm{ V_{X^{\perp}}^T  V_{\UUt \WWt}  } + \mu \specnorm{ \left(   \mathcal{A}^* \mathcal{A} -\Id \right) \left(XX^T -\UUt  \UUt^T\right) }  + 10 \mu \specnorm{X}^2  \right)  \specnorm{ \UUt \Wtperp }\\ 
	\overleq{(c)} &  \mu \left(  9 \specnorm{V^T_{X^\perp}  V_{\UUt \WWt}}  \specnorm{X}^2 +   \specnorm{ \left(   \mathcal{A}^* \mathcal{A} -\Id \right) \left(XX^T -\UUt  \UUt^T\right)}  \right) \specnorm{ \UUt \Wtperp },
	\end{align*}
	where in inequality $(a)$ we used Lemma \ref{lemma:auxiliary2}. For inequality $(b)$ we used the assumption $\specnorm{\Ut} \le 3 \specnorm{X}$ and for inequality $(c)$ we used assumption $\specnorm{  V_{X^{\perp}}^T  V_{\UUt \WWt}} \le  c \kappa^{-1}$ and the assumption on the step size $\mu$ with a small enough constant $c>0$.\\

	
	\noindent\textbf{Summand $(b)$:} First, we compute that
	\begin{align*}
	&V_{X^{\perp}}^T \UUtplus \Wtperp \\
	=&  V_{X^{\perp}}^T \UUt \Wtperp + \mu V_{X^{\perp}}^T \left(XX^T - \UUt \UUt^T\right) \UUt \Wtperp+  \mu   V_{X^{\perp}}^T  \left[  \left(   \mathcal{A}^* \mathcal{A} -\Id \right) \left(XX^T -\UUt \UUt^T\right) \right]  \UUt \Wtperp\\
	=& V_{X^{\perp}}^T \UUt \Wtperp  -  \mu V_{X^{\perp}}^T \UUt \UUt^T \UUt \Wtperp + \mu  V_{X^{\perp}}^T  \left[  \left(   \mathcal{A}^* \mathcal{A} -\Id \right) \left(XX^T -\UUt \UUt^T\right) \right]  \UUt \Wtperp \\
	=& V_{X^{\perp}}^T \UUt \Wtperp  -  \mu V_{X^{\perp}}^T \UUt \UUt^T V_{X^{\perp}} V_{X^{\perp}}^T   \UUt \Wtperp +  \mu    V_{X^{\perp}}^T  \left[  \left(   \mathcal{A}^* \mathcal{A} -\Id \right) \left(XX^T -\UUt \UUt^T\right) \right]  \UUt \Wtperp\\
	=& \left( \Id -  \mu     V_{X^{\perp}}^T \UUt \UUt^T V_{X^{\perp}} - \mu  V_{X^{\perp}}^T  \left[  \left(   \mathcal{A}^* \mathcal{A} -\Id \right) \left(XX^T -\UUt \UUt^T\right) \right]  V_{X^\perp}\right)   V_{X^{\perp}}^T \UUt \Wtperp\\
	=& V_{X^{\perp}}^T \UUt \Wtperp  -  \mu V_{X^{\perp}}^T \UUt \Wtperp \Wtperp^T \UUt^T V_{X^{\perp}} V_{X^{\perp}}^T   \UUt \Wtperp -  \mu V_{X^{\perp}}^T \UUt \WWt \WWt^T \UUt^T V_{X^\perp} V_{X^\perp}^T    \UUt \Wtperp\\
	&+  \mu    V_{X^{\perp}}^T \left[  \left(   \mathcal{A}^* \mathcal{A} -\Id \right) \left(XX^T -\UUt \UUt^T\right) \right]  V_{X^\perp} V_{X^\perp}^T     \UUt \Wtperp\\
	=& \left( \Id - \mu   V_{X^{\perp}}^T \UUt \WWt \WWt^T \UUt^T V_{X^{\perp}} +  \mu   V_{X^{\perp}}^T \left[  \left(   \mathcal{A}^* \mathcal{A} -\Id \right) \left(XX^T -\UUt \UUt^T\right) V_{X^{\perp}} \right]   \right) V_{X^{\perp}}^T \UUt \Wtperp  \left( \Id - \mu  \Wtperp^T \UUt^T    \UUt \Wtperp   \right)\\
	& - \mu^2 \left(  V_{X^{\perp}}^T \UUt \WWt \WWt^T \UUt^T V_{X^{\perp}} -   V_{X^{\perp}}^T \left[  \left(   \mathcal{A}^* \mathcal{A} -\Id \right) \left(XX^T -\UUt \UUt^T\right) V_{X^\perp}  \right] \right)  V_{X^{\perp}}^T \UUt \Wtperp \Wtperp^T \UUt^T   \UUt \Wtperp .
	\end{align*}
	Set for brevity of notation $M_2:=  V_{X^{\perp}}^T \UUt \WWt \WWt^T \UUt^T V_{X^{\perp}} $ and $M_3:=   V_{X^{\perp}}^T \left(   \mathcal{A}^* \mathcal{A} -\Id \right) \left(XX^T -\UUt \UUt^T\right)  V_{X^{\perp}}$. Hence, we have computed that
	\begin{align*}
	&V_{X^{\perp}}^T \UUtplus \Wtperp \\
	=&  \left( \Id - \mu M_2 + \mu M_3\right) V_{X^{\perp}}^T \UUt \Wtperp  \left( \Id - \mu  \Wtperp^T \UUt^T   \UUt \Wtperp   \right) - \mu^2  \left(M_2 - M_3\right) V_{X^{\perp}}^T \UUt \Wtperp\Wtperp^T \UUt^T   \UUt \Wtperp  \\
	=&  \left( \Id - \mu M_2 + \mu M_3\right) V_{X^{\perp}}^T \UUt \Wtperp  \left( \Id - \mu  \Wtperp^T \UUt^T V_{X^{\perp}} V_{X^{\perp}}^T   \UUt \Wtperp   \right) - \mu^2  \left(M_2 - M_3\right) V_{X^{\perp}}^T \UUt  \Wtperp \Wtperp^T \UUt^T   \UUt \Wtperp  .
	\end{align*}
	Hence, we obtain that
	\begin{align*}
	&\specnorm{ V_{X^{\perp}}^T \UUtplus \Wtperp \Wtperp^T  \Wtplusperp  }\\
	\le& \specnorm{   \left( \Id - \mu M_2 +\mu M_3 \right) V_{X^{\perp}}^T \UUt \Wtperp  \left( \Id - \mu  \Wtperp^T \UUt^T V_{X^{\perp}} V_{X^{\perp}}^T   \UUt \Wtperp   \right)\\
		&-\mu^2   \left(  M_2 - M_3 \right)    V_{X^{\perp}}^T \UUt  \Wtperp\Wtperp^T \UUt^T   \UUt \Wtperp     } \specnorm{ \Wtperp^T  \Wtplusperp }\\
	\le& \specnorm{   \left( \Id - \mu M_2 +\mu M_3 \right) V_{X^{\perp}}^T \UUt \Wtperp  \left( \Id - \mu  \Wtperp^T \UUt^T V_{X^{\perp}} V_{X^{\perp}}^T   \UUt \Wtperp   \right)\\
		& -\mu^2   \left(  M_2 - M_3 \right)    V_{X^{\perp}}^T \UUt \Wtperp \Wtperp^T \UUt^T   \UUt \Wtperp } \\
	\le & \specnorm{\left( \Id - \mu M_2 + \mu M_3 \right) V_{X^{\perp}}^T \UUt \Wtperp  \left( \Id - \mu  \Wtperp^T \UUt^T V_{X^{\perp}} V_{X^{\perp}}^T   \UUt  \Wtperp   \right)}\\
	&+\mu^2 \specnorm{ \left( M_2 - M_3\right)  V_{X^{\perp}}^T \UUt  \Wtperp\Wtperp^T \UUt^T  \UUt \Wtperp    }.
	\end{align*}
	In order to proceed, we compute that
	\begin{align*}
	&\specnorm{\left( \Id - \mu M_2 + \mu M_3  \right) V_{X^{\perp}}^T \UUt \Wtperp  \left( \Id - \mu  \Wtperp^T \UUt^T  V_{X^{\perp}} V_{X^{\perp}}^T   \UUt \Wtperp   \right)}\\
	\overleq{(a)}& \specnorm{\Id  - \mu M_2  + \mu M_3 } \specnorm{V_{X^{\perp}}^T \UUt \Wtperp  \left( \Id - \mu  \Wtperp^T \UUt^T  V_{X^{\perp}} V_{X^{\perp}}^T   \UUt \Wtperp   \right) }\\
	\overeq{(b)} &   \specnorm{\Id - \mu M_2 + \mu M_3  } \specnorm{V_{X^{\perp}}^T \UUt \Wtperp} \left(1-\mu \specnorm{V_{X^{\perp}}^T \UUt \Wtperp}^2 \right) \\
	\overleq{(c)} & \left(  \specnorm{ \Id - \mu M_2 } + \mu \specnorm{ M_3 }    \right)   \specnorm{V_{X^{\perp}}^T \UUt \Wtperp} \left(1-\mu \specnorm{V_{X^{\perp}}^T \UUt \Wtperp}^2 \right)\\
	\overleq{(d)} & \left( 1 + \mu \specnorm{ M_3 }    \right)   \specnorm{V_{X^{\perp}}^T \UUt \Wtperp} \left(1-\mu \specnorm{V_{X^{\perp}}^T \UUt \Wtperp}^2 \right)\\
	= &  \left( 1 + \mu \specnorm{ M_3 }    \right)    \specnorm{\UUt \Wtperp} \left(1-\mu \specnorm{\UUt \Wtperp}^2 \right) \\
	\le& \specnorm{\UUt \Wtperp} \left( 1 -\mu \specnorm{U \Wtperp}^2 + \mu \specnorm{\left(   \mathcal{A}^* \mathcal{A} -\Id \right) \left(XX^T -\UUt \UUt^T \right)  V_{X^{\perp}}} -\mu^2 \specnorm{M_3}  \specnorm{\UUt \Wtperp}^2 \right)\\
	\le & \specnorm{\UUt \Wtperp} \left( 1 -\mu \specnorm{\UUt \Wtperp}^2 + \mu \specnorm{\left(   \mathcal{A}^* \mathcal{A} -\Id \right) \left(XX^T -\UUt \UUt^T \right)  V_{X^{\perp}}}  \right)\\
	\le & \specnorm{\UUt \Wtperp} \left( 1 -\mu \specnorm{\UUt \Wtperp}^2 + \mu \specnorm{\left(   \mathcal{A}^* \mathcal{A} -\Id \right) \left(XX^T -\UUt \UUt^T \right)  }  \right).
	\end{align*}
	The inequality $(a)$ follows from the submultiplicativity of the spectral norm. Equality $(b)$ can be seen be using the singular value decomposition of  $ V_{X^{\perp}}^T \UUt \Wtperp $ and the assumption $ \mu \le c_1 \specnorm{X}^{-2}  \le 1/(\sqrt{3} \specnorm{V_{X^{\perp}}^T \UUt \Wtperp}^2) $. Inequality $(c)$ follows from the triangle inequality. For inequality $(d)$ we used that $ 0 \preceq  \Id - \mu   V_{X^{\perp}}^T \UUt \WWt \WWt^T \UUt^T  V_{X^{\perp}}   \preceq \Id $, which again is a consequence of our assumptions on $\mu$ and $\specnorm{\UUt} \le 3 \specnorm{X} $.
	For the $ O \left( \mu^2 \right)$-term we note that
	\begin{align*}
	&\specnorm{  \left(  M_2 - M_3 \right) V_{X^{\perp}}^T \UUt  \Wtperp\Wtperp^T \UUt^T    \UUt \Wtperp    } \\
	= &  	\specnorm{  \left(  M_2 - M_3 \right) V_{X^{\perp}}^T \UUt  \Wtperp\Wtperp^T \UUt^T  V_{X^\perp}   V_{X^\perp}^T \UUt    \Wtperp    }    \\
	\le & \specnorm{M_2-M_3} \specnorm{V_{X^{\perp}}^T   \UUt \Wtperp }^3\\
	\le &  \left(  \specnorm{M_2}  + \specnorm{M_3} \right) \specnorm{V_{X^{\perp}}^T   \UUt \Wtperp }^3\\
	\le & \left(  \specnorm{ V_{X^{\perp}}^T \UUt \WWt}^2  + \specnorm{M_3} \right) \specnorm{V_{X^{\perp}}^T   \UUt \Wtperp }^3 \\
	\le & \left(  \specnorm{ V_{X^{\perp}}^T V_{\UUt \WWt} }^2  \specnorm{ \UUt \WWt }^2  + \specnorm{M_3} \right) \specnorm{V_{X^{\perp}}^T   \UUt \Wtperp }^3 \\
	=& \left(  \specnorm{ V_{X^{\perp}}^T V_{\UUt \WWt} }^2  \specnorm{ \UUt \WWt }^2  + \specnorm{M_3} \right) \specnorm{  \UUt \Wtperp }^3 \\
	=& \left(  \specnorm{ V_{X^{\perp}}^T V_{\UUt \WWt} }^2  \specnorm{ \UUt \WWt }^2  + \specnorm{\left(   \mathcal{A}^* \mathcal{A} -\Id \right) \left(XX^T -\UUt \UUt^T \right)  V_{X^{\perp}} } \right) \specnorm{  \UUt \Wtperp }^3 \\
	\le&  \left(  \specnorm{ V_{X^{\perp}}^T V_{\UUt \WWt} }^2  \specnorm{ \UUt \WWt }^2  + \specnorm{\left(   \mathcal{A}^* \mathcal{A} -\Id \right) \left(XX^T - \UUt \UUt^T \right)   } \right) \specnorm{  \UUt  \Wtperp }^3.
	\end{align*}
	It follows that
	\begin{align*}
	&\mu^2 	\specnorm{  \left(  M_2 - M_3 \right) V_{X^{\perp}}^T \UUt  \Wtperp^T \UUt^T    \UUt \Wtperp    }\\
	\le&  \mu^2  \left(  \specnorm{ V_{X^{\perp}}^T V_{\UUt \WWt} }^2  \specnorm{ \UUt \WWt }^2  + \specnorm{\left(   \mathcal{A}^* \mathcal{A} -\Id \right) \left(XX^T -\UUt \UUt^T \right)   } \right) \specnorm{  \UUt \Wtperp }^3  \\
	\overleq{(a)} &   \mu^2  \left(  9 \specnorm{ X }^2  + \specnorm{\left(   \mathcal{A}^* \mathcal{A} -\Id \right) \left(XX^T -\UUt  \UUt^T \right)   } \right) \specnorm{  \UUt \Wtperp }^3     \\
	\overleq{(b)} & \frac{\mu}{2} \specnorm{  \UUt \Wtperp }^3,
	\end{align*}
	In $(a)$ we used that $\specnorm{\Ut} \le 3 \specnorm{X} $. In $(b)$ we used our assumption on the step size $\mu$. This implies that
	\begin{align*}
	\specnorm{(b)} 	\le & \specnorm{\UUt \Wtperp} \left( 1 -\frac{\mu}{2} \specnorm{\UUt \Wtperp}^2 + \mu \specnorm{\left(   \mathcal{A}^* \mathcal{A} -\Id \right) \left(XX^T -\UUt \UUt^T \right)  }  \right).
	\end{align*}
	\textbf{Conclusion: }Putting things together it follows
	\begin{align*}
	&\specnorm{ V_{X^{\perp}}^T \UUtplus \Wtplusperp  }\\
	\le & \specnorm{(a)} + \specnorm{(b)} \\
	\le &  \left(   1- \frac{\mu}{2} \specnorm{\UUt \Wtperp}^2   +    9 \mu \specnorm{V^T_{X^\perp}  V_{\UUt \WWt}}  \specnorm{X}^2 + 2 \mu \specnorm{\left(   \mathcal{A}^* \mathcal{A} -\Id \right) \left(XX^T -\UUt \UUt^T \right)  }  \right)  \specnorm{\UUt \Wtperp}.
	\end{align*}
	This finishes the proof.
\end{proof}

\subsection{Proof of Lemma \ref{lemma:anglecontrol}}
We define the inverse of the square root of a symmetric, positiv definite matrix $A= V_A \Sigma_A V_A^T$ by $A^{-1/2}:= V_A \Sigma^{-1/2}_A V_A^T $, where $ \left( \Sigma^{-1/2}_A   \right)_{ii} :=  \frac{1}{\sqrt{ A_{ii}   }}$. We will need the following technical lemma, which gives a bound on the first order Taylor-approximation of the matrix inverse square root.
\begin{lemma}\label{lemma:taylor1}
	Let $A$ be a symmetric matrix such that $ \specnorm{A} \le 1/2 $. Then it holds that
	\begin{equation*}
	\specnorm{\left(\Id +A\right)^{-1/2} - \Id + \frac{1}{2} A } \le  3 \specnorm{A}^2.
	\end{equation*}
\end{lemma}

\begin{proof}
	Since $A$ is symmetric, this can be readily deduced from the (one-dimensional) Taylor's theorem. Indeed, we have that for $\vert x \vert \le 1/2 $ that
	\begin{align*}
	\big\vert \frac{1}{\sqrt{1+x}} - 1 +\frac{x}{2} \big\vert &\le \underset{\vert z\vert \le 1/2}{\sup} \big\vert  \frac{3}{8} \cdot \left( 1+z  \right)^{-5/2} \cdot x^2 \big\vert \\
	&\le 3 x^2.
	\end{align*}
\end{proof}
The next technical lemma shows that the orthogonal matrices $W_t$ and $W_{t+1}$ span approximately the same column space.
\begin{lemma}\label{lemma:auxiliary}
	Assume that the assumptions of Lemma  \ref{lemma:anglecontrol}  are fulfilled. Then it holds that
	\begin{equation}\label{ineq:intern334}
	\begin{split}
	\specnorm{\Wtperp^T \WWtplus }  \le &  \mu  \left( \frac{1}{6400}  \sigma_{\min} \left(X\right)^2  +   \specnorm{\UUt \WWt} \specnorm{\UUt \Wtperp} \right) \specnorm{ V_{X^\perp}^T V_{\UUt \WWt }  }  + 4\mu \specnorm{ \left[    \left( \Id -\mathcal{ A}^*  \mathcal{A}   \right) \left( XX^T - \UUt \UUt^T \right)  \right] }
	\end{split}
	\end{equation}
	and
	\begin{equation*}
	\sigma_{\min} \left(\WWt^T \WWtplus\right)    \ge 1/2.
	\end{equation*}
\end{lemma}
\begin{proof}
	Due to $V_X^T \UUtplus = V_X^T \UUtplus \WWtplus\WWtplus^T $ we observe that
	\begin{align*}
	\specnorm{\Wtperp^T \WWtplus } &= \specnorm{\Wtperp^T  \UUtplus^T V_X \left(  V_X^T \UUtplus \UUtplus^T V_X  \right)^{-1/2} }.
	\end{align*}
	We note that
	\begin{align*}
	&V_{X}^T\UUtplus \Wtperp\\
	 =&  V_{X}^T \left( \Id + \mu \left(XX^T -\UUt \UUt^T\right) \right) \UUt \Wtperp - \mu V_X^T \left[ \left( \Id -\mathcal{ A}^*  \mathcal{A}   \right) \left( XX^T - \UUt \UUt^T \right) \right] \UUt \Wtperp \\
	=& - \mu V_X^T \UUt \UUt^T \UUt \Wtperp   -\mu  V_X^T \left[ \left( \Id -\mathcal{ A}^*  \mathcal{A}   \right) \left( XX^T - \UUt \UUt^T \right) \right] \UUt \Wtperp \\
	=& - \mu V_X^T \UUt \WWt \WWt^T \UUt^T \UUt \Wtperp  -\mu V_X^T \left[    \left( \Id -\mathcal{ A}^*  \mathcal{A}   \right) \left( XX^T - \UUt \UUt^T \right)  \right] \UUt \Wtperp\\
	=& - \mu V_X^T \UUt \WWt \WWt^T \UUt^T V_{X^\perp} V_{X^\perp}^T \UUt \Wtperp  - \mu V_X^T \left[    \left( \Id -\mathcal{ A}^*  \mathcal{A}   \right) \left( XX^T - \UUt \UUt^T \right)  \right] \UUt \Wtperp \\
	=& - \mu V_X^T \UUt \WWt \WWt^T \UUt^T V_{\UUt \WWt} V_{\UUt \WWt}^T V_{X^\perp} V_{X^\perp}^T \UUt \Wtperp  - \mu  V_X^T \left[    \left( \Id -\mathcal{ A}^*  \mathcal{A}   \right) \left( XX^T - \UUt \UUt^T \right)  \right] \UUt \Wtperp.
	\end{align*}
	It follows that
	\begin{align*}
	\specnorm{\Wtperp^T \WWtplus } \le& \mu \specnorm{ \left(  V_X^T \UUtplus \UUtplus^T V_X  \right)^{-1/2} V_X^T \UUt \WWt \WWt^T \UUt^T V_{\UUt \WWt} V_{\UUt \WWt}^T V_{X^\perp} V_{X^\perp}^T \UUt \Wtperp}\\
	&  +  \mu \specnorm{ \left[    \left( \Id -\mathcal{ A}^*  \mathcal{A}   \right) \left( XX^T - \UUt \UUt^T \right)  \right] } \specnorm{\UUt \Wtperp} \specnorm{  \left(  V_X^T \UUtplus \UUtplus^T V_X  \right)^{-1/2} }\\
	=& \mu  \specnorm{ \left(  V_X^T \UUtplus \UUtplus^T V_X  \right)^{-1/2} V_X^T \UUt \WWt \WWt^T \UUt^T V_{\UUt \WWt} V_{\UUt \WWt}^T V_{X^\perp} V_{X^\perp}^T \UUt \Wtperp}\\
	&  + \mu  \frac{ \specnorm{ \left[    \left( \Id -\mathcal{ A}^*  \mathcal{A}   \right) \left( XX^T - \UUt \UUt^T \right)  \right] } \specnorm{\UUt \Wtperp}}{ \sigma_{\min} \left( V_X^T \UUtplus \right)   } \\
	\le& \mu  \specnorm{ \left(  V_X^T \UUtplus \UUtplus^T V_X  \right)^{-1/2} V_X^T \UUt \WWt }  \specnorm{ \UUt \WWt }  \specnorm{\UUt \Wtperp} \specnorm{ V_{X^\perp}^T V_{\UUt \WWt}  }\\
	&  + \mu \frac{ \specnorm{ \left[    \left( \Id -\mathcal{ A}^*  \mathcal{A}   \right) \left( XX^T - \UUt \UUt^T \right)  \right] } \specnorm{\UUt \Wtperp}}{ \sigma_{\min} \left( V_X^T \UUtplus \right)  } .
	\end{align*}
	We note that
	\begin{align*}
	\left(  V_X^T \UUtplus \UUtplus^T V_X  \right)^{-1/2} V_X^T \UUt \WWt=& \left(  V_X^T \UUtplus \UUtplus^T V_X  \right)^{-1/2} V_X^T \UUtplus \WWt\\
	&-  \mu \left(  V_X^T \UUtplus \UUtplus^T V_X  \right)^{-1/2} V_X^T \mathcal{ A}^* \mathcal{ A}  \left(XX^T - \UUt \UUt^T\right) \UUt \WWt.
	\end{align*}
	It follows that
	\begin{align*}
	&\specnorm{ \left(  V_X^T \UUtplus \UUtplus^T V_X  \right)^{-1/2} V_X^T \UUt\WWt   }\\
	\le& \specnorm{\left(  V_X^T \UUtplus \UUtplus^T V_X  \right)^{-1/2} V_X^T \UUtplus \WWt} + \mu  \specnorm{\left(  V_X^T \UUtplus \UUtplus^T V_X  \right)^{-1/2} V_X^T \mathcal{ A}^* \mathcal{ A}  \left(XX^T - \UUt \UUt^T\right)\UUt \WWt} \\
	\le& \specnorm{\left(  V_X^T \UUtplus \UUtplus^T V_X  \right)^{-1/2} V_X^T \UUtplus } + \mu \specnorm{\left(  V_X^T \UUtplus \UUtplus^T V_X  \right)^{-1/2} V_X^T \mathcal{ A}^* \mathcal{ A}  \left(XX^T - \UUt \UUt^T\right)\UUt \WWt} \\
	=&1 +\mu  \specnorm{\left(  V_X^T \UUtplus \UUtplus^T V_X  \right)^{-1/2} V_X^T \mathcal{ A}^* \mathcal{ A}  \left(XX^T - \UUt \UUt^T\right)\UUt \WWt} \\
	\le & 1 + \mu \frac{\specnorm{ \mathcal{ A}^* \mathcal{ A}  \left(XX^T - \UUt \UUt^T\right)} \specnorm{\UUt \WWt}}{\sigma_{\min} \left(  V_X^T \UUtplus   \right)}.
	\end{align*}
	Hence, we obtain that
	\begin{equation}\label{ineq:intern445}
	\begin{split}
	\specnorm{\Wtperp^T \WWtplus }  \le&     \mu  \left(    1 + \mu \frac{\specnorm{ \mathcal{ A}^* \mathcal{ A}  \left(XX^T - \UUt \UUt^T\right)} \specnorm{\UUt \WWt}}{\sigma_{\min} \left(  V_X^T \UUtplus   \right)}     \right)  \specnorm{ \UUt \WWt }  \specnorm{\UUt \Wtperp} \specnorm{ V_{X^\perp}^T V_{\UUt \WWt}  }\\
	&  + \mu  \frac{ \specnorm{     \left( \Id -\mathcal{ A}^*  \mathcal{A}   \right) \left( XX^T - \UUt \UUt^T \right)} \specnorm{\UUt \Wtperp}}{ \sigma_{\min} \left( V_X^T \UUtplus \right)  }.
	\end{split}
	\end{equation}
	Next, we are going to show $ \sigma_{\min} \left(  V_X^T \UUtplus \right) \ge \frac{\sigma_{\min} \left(  \UUt \WWt\right)}{2}$. We note that
	\begin{align*}
	\sigma_{\min} \left(  V_X^T \UUtplus \right)&\ge   \sigma_{\min} \left(  V_X^T \UUtplus \WWt \right)  \\
	&= \sigma_{\min} \left(   V_X^T\left( \Id + \mu \left[    \left( \mathcal{ A}^*  \mathcal{A}   \right) \left( XX^T - \UUt \UUt^T \right)  \right] \right) \UUt \WWt \right)\\
	&= \sigma_{\min} \left(   V_X^T\left( \Id + \mu \left[    \left( \mathcal{ A}^*  \mathcal{A}   \right) \left( XX^T - \UUt \UUt^T \right)  \right] \right) V_{\UUt \WWt} V_{\UUt \WWt}^T \UUt \WWt \right)\\
	&\ge \sigma_{\min} \left(   V_X^T\left( \Id + \mu \left[    \left( \mathcal{ A}^*  \mathcal{A}   \right) \left( XX^T - \UUt \UUt^T\right)  \right] \right) V_{\UUt \WWt} \right) \sigma_{\min} \left(V_{\UUt \WWt}^T  \UUt \WWt \right)\\
	&= \sigma_{\min} \left(   V_X^T V_{\UUt \WWt}+ \mu V_X^T  \left[    \left( \mathcal{ A}^*  \mathcal{A}   \right) \left( XX^T - \UUt \UUt^T \right)  \right] V_{\UUt \WWt} \right) \sigma_{\min} \left(\UUt  \WWt \right)\\
	&\ge \left( \sigma_{\min} \left(   V_X^T V_{\UUt \WWt}\right) - \mu  \specnorm{ V_X^T  \left[    \left( \mathcal{ A}^*  \mathcal{A}   \right) \left( XX^T - \UUt \UUt^T \right)  \right] V_{\UUt \WWt} }   \right) \sigma_{\min} \left( \UUt \WWt \right).
	\end{align*}
	We observe that due to our assumption on $\specnorm{V_{X^\perp}^T V_{\UUt \WWt}}$ we have that
	\begin{align*}
	\sigma_{\min} \left(   V_X^T V_{\UUt \WWt}\right) = \sqrt{1- \specnorm{V_{X^\perp}^T V_{\UUt \WWt}}^2} \ge \frac{3}{4}.
	\end{align*}
	Next, we note that due to assumptions \eqref{intern:eqref1}, \eqref{intern:eqref3} and $\specnorm{\UUt} \le 3 \specnorm{X}$ we have that
	\begin{align*}
	\mu \specnorm{ V_X^T \left[    \left( \mathcal{ A}^*  \mathcal{A}   \right) \left( XX^T - \UUt \UUt^T \right)  \right] V_{\UUt \WWt} } &\le  \mu\specnorm{     \left( \mathcal{ A}^*  \mathcal{A}   \right) \left( XX^T - \UUt \UUt^T \right)}\\
	&\le 10 \mu \specnorm{X}^2  +  \mu\specnorm{     \left( \Id - \mathcal{ A}^*  \mathcal{A}   \right) \left( XX^T - \UUt \UUt^T \right)} \le \frac{1}{4}.
	\end{align*}
	Hence, we have shown that $\sigma_{\min} \left(  V_X^T \UUtplus \right) \ge \frac{\sigma_{\min} \left(\UUt \WWt\right)}{2} $. This implies due to inequality \eqref{ineq:intern445} that
	\begin{align*}
	\specnorm{\Wtperp^T \WWtplus }  \le&   \mu    \left(    1 +  2\mu \frac{\specnorm{ \mathcal{ A}^* \mathcal{ A}  \left(XX^T - \UUt \UUt^T\right)} \specnorm{\UUt \WWt}}{\sigma_{\min} \left(  \UUt \WWt \right)}     \right)  \specnorm{ \UUt \WWt }  \specnorm{\UUt \Wtperp} \specnorm{ V_{X^\perp}^T V_{\UUt \WWt}  }\\
	&  + 2 \mu \frac{ \specnorm{     \left( \Id -\mathcal{ A}^*  \mathcal{A}   \right) \left( XX^T - \UUt \UUt^T \right)   } \specnorm{\UUt \Wtperp}}{ \sigma_{\min} \left( \UUt \WWt \right)  }\\
	\overleq{(a)}&  \mu \specnorm{ V_{X^\perp}^T V_{\UUt \WWt}  } \specnorm{\UUt \WWt} \specnorm{\UUt \Wtperp} + 4 \mu \specnorm{    \left( \Id -\mathcal{ A}^*  \mathcal{A}   \right) \left( XX^T - \UUt \UUt^T \right)   } \\
	& + 4\mu^2 \specnorm{     \left( \mathcal{ A}^*  \mathcal{A}   \right) \left( XX^T - \UUt \UUt^T \right)   }  \specnorm{\UUt \WWt}^2 \specnorm{ V_{X^\perp}^T V_{\UUt \WWt}  }.
	\end{align*}
	In inequality $(a)$ we have used the assumption that $  \specnorm{\UUt \Wtperp}  \le 2 \sigma_{\min} \left(  \UUt \WWt\right)$. In order to proceed, we note that
	\begin{align*}
	\specnorm{     \left( \mathcal{ A}^*  \mathcal{A}   \right) \left( XX^T - \UUt \UUt^T \right)   }   &\le \specnorm{XX^T - \UUt \UUt^T }  + \specnorm{     \left(\Id - \mathcal{ A}^*  \mathcal{A}   \right) \left( XX^T - \UUt \UUt^T \right)   } \\
	&\le 11 \specnorm{X}^2,
	\end{align*}
	where we used the assumption $\specnorm{\UUt \WWt} \le 3 \specnorm{X} $ and $\specnorm{ \left[    \left( \Id -\mathcal{ A}^*  \mathcal{A}   \right) \left( XX^T - \UUt \UUt^T \right)  \right] } \le c \sigma_{\min} \left(X\right)^2 $. 
	Hence, we obtain that
	\begin{align*}
	\specnorm{\Wtperp^T \WWtplus }  
	\le &  \mu\left( \frac{1}{6400}  \sigma_{\min} \left(X\right)^2  +   \specnorm{\UUt \WWt} \specnorm{\UUt \Wtperp} \right) \specnorm{ V_{X^\perp}^T V_{\UUt \WWt}  }  + 4  \mu \specnorm{    \left( \Id -\mathcal{ A}^*  \mathcal{A}   \right) \left( XX^T - \UUt \UUt^T \right)   },
	\end{align*}
	where  we also used $ \mu  \le c \kappa^{-2} \specnorm{X}^{-2} $. Hence, we have shown inequality \eqref{ineq:intern334}.\\
	
	\noindent In order to finish the proof we note that
	\begin{align*}
	\specnorm{\Wtperp^T \WWtplus }  &\overleq{(a)} \mu \left( \frac{1}{6400}  \sigma_{\min} \left(X\right)^2  +   \specnorm{\UUt \WWt} \specnorm{\UUt \Wtperp} \right) \specnorm{ V_{X^\perp}^T V_{\UUt \WWt}  }  + 4\mu c  \sigma_{\min} \left(X\right)^2\\
	&\overleq{(b)}  \mu  \left( \frac{1}{6400}  \sigma_{\min} \left(X\right)^2  +   9 \specnorm{X}^2 \right) \specnorm{ V_{X^\perp}^T V_{\UUt \WWt}  }  + 4c\mu  \sigma_{\min} \left(X\right)^2\\
	& \overset{(c)}{\lesssim} c.
	\end{align*}
	In inequality $(a)$ we used the assumption $\specnorm{     \left( \Id -\mathcal{ A}^*  \mathcal{A}   \right) \left( XX^T - \UUt \UUt^T \right)   } \le c \sigma_{\min} \left(X\right)^2 $. In inequality $(b)$ we used $ \specnorm{\UUt \Wtperp} \le 3 \specnorm{X} $ and $ \specnorm{\UUt \WWt} \le 3 \specnorm{X} $. To obtain inequality $(c)$ we used the assumption $ \mu  \le c \specnorm{X}^{-2} $. By choosing $ c>0$ small enough we obtain that $\specnorm{\Wtperp^T \WWtplus }  \le 1/2$. Note that this implies that
	\begin{equation*}
	\sigma_{\min} \left(\WWt^T \WWtplus\right) = \sqrt{1- \specnorm{ \Wtperp^T \WWtplus }^2 }    \ge 1/2,
	\end{equation*}
	which finishes the proof.
\end{proof}
Now we have provided all the technical preliminaries to prove Lemma   \ref{lemma:anglecontrol}.
\begin{proof}[Proof of Lemma \ref{lemma:anglecontrol}]
	In order to simplify notation, we define $ M:= \mathcal{ A^* \mathcal{A}} \left(XX^T - \UUt \UUt^T  \right)$. Hence, we may write
	\begin{equation*}
	\UUtplus = \left(\Id + \mu M\right) \UUt.
	\end{equation*}
	Now we note that
	\begin{equation}\label{equ:interncorrected1}
	\begin{split}
	\UUtplus \WWtplus &= \left( \Id + \mu M  \right)  \UUt \WWtplus  \\
	&= \left( \Id + \mu M  \right)  \UUt \WWt \WWt^T \WWtplus  + \left( \Id + \mu M  \right)  \UUt \Wtperp\Wtperp^T \WWtplus \\
	&= \left( \Id + \mu M  \right) V_{\UUt W} V_{\UUt W}^T  \UUt \WWt \WWt^T \WWtplus  + \left( \Id + \mu M  \right)  \UUt \Wtperp\Wtperp^T \WWtplus .
	\end{split}
	\end{equation}
	Note that $V_{\UUt \WWt}^T \UUt \WWt \WWt^T \WWtplus   $ is invertible, since $ V_{\UUt \WWt }^T \UUt \WWt $ is invertible by assumption \eqref{intern:eqref2}  and $ \WWt^T \WWtplus$ is invertible by Lemma \ref{lemma:auxiliary}. Hence, we see that
	\begin{align*}
	&\left( \Id + \mu M  \right)  \UUt \Wtperp\Wtperp^T \WWtplus \\
	=&   \left( \Id + \mu M  \right)   \UUt \Wtperp\Wtperp^T \WWtplus \left(  V_{\UUt \WWt}^T \UUt \WWt \WWt^T \WWtplus  \right)^{-1}  V_{\UUt \WWt}^T  \UUt  \WWt \WWt^T \WWtplus   \\
	=& \left( \Id + \mu M  \right)  \bracing{=P}{ \UUt \Wtperp\Wtperp^T \WWtplus \left(  V_{\UUt \WWt}^T \UUt \WWt \WWt^T \WWtplus  \right)^{-1} V_{\UUt \WWt}^T } V_{\UUt \WWt} V_{\UUt \WWt}^T  \UUt \WWt \WWt^T \WWtplus \\
	=& \left( \Id + \mu M\right) P V_{\UUt \WWt} V_{\UUt \WWt}^T  \UUt \WWt \WWt^T \WWtplus.
	\end{align*}
	Hence, by inserting the last equation into equation \eqref{equ:interncorrected1} we obtain that
	\begin{align*}
	\UUtplus \WWtplus  &= \left(\Id + \mu M \right) \left(  \Id +  P \right) V_{\UUt \WWt} V_{\UUt \WWt}^T \UUt \WWt \WWt^T \WWtplus
	\end{align*}
	Recall that $V_{\UUt \WWt}^T \UUt \WWt \WWt^T \WWtplus$ is an invertible matrix. This implies that the span of the left-singular vectors of
	\begin{align*}
	Z:=  \left(\Id + \mu M \right) \left(  \Id +  P \right) V_{\UUt \WWt}
	\end{align*}
	is the same as the span of the left-singular vectors of  $\UUtplus \WWtplus$. Let $V_Z \Sigma_Z W_Z^T $ be the singular value decomposition of $Z$. From these considerations it follows that
	\begin{equation*}
	\specnorm{ V^T_{X^\perp} V_{\UUtplus \WWtplus} } = \specnorm{V^T_{X^\perp} V_Z } = \specnorm{V^T_{X^\perp} V_Z W_Z^T }.
	\end{equation*}
	Next, we note that
	\begin{align*}
	V_Z W_Z^T  &= Z \left( Z^T Z \right)^{-1/2} \\
	&= \left(\Id + \mu M \right) \left(  \Id +  P \right) V_{\UUt \WWt} \left(    V_{\UUt \WWt}^T  \left(  \Id +  P^T \right)  \left(\Id + \mu M \right)^2 \left(  \Id +  P \right) V_{\UUt \WWt}  \right)^{-1/2}.
	\end{align*}
	We note that
	\begin{align*}
	\left(\Id + \mu M \right) \left(  \Id +  P \right)  =& \Id + \bracing{ =:B }{\mu M + P + \mu MP}\\
	=&\Id + \bracing{ =:B_1 }{\mu \left(XX^T - \UUt \UUt^T \right)} + \bracing{ =:B_2 }{ \mu \left(  \mathcal{ A^* \mathcal{A}  -\Id   }   \right)\left( XX^T - \UUt \UUt^T  \right)} \\
	+& \bracing{ =:B_3 }{\UUt \Wtperp\Wtperp^T \WWtplus \left(  V_{\UUt \WWt}^T \UUt \WWt \WWt^T \WWtplus  \right)^{-1} V_{\UUt \WWt}^T } + \bracing{=:B_4}{\mu MP}.
	\end{align*}
	Hence, we have that
	\begin{align*}
	&Z \left(Z^T Z\right)^{-1/2} \\
	 =& \left( \Id + B  \right) V_{\UUt \WWt} \left(   V_{U\WWt}^T  \left( \Id + B + B^T + B^TB \right) V_{\UUt \WWt}    \right)^{-1/2}\\
	=& \left( \Id + B  \right) V_{\UUt \WWt} \left(   \Id +  V_{\UUt \WWt}^T B V_{\UUt \WWt}  +V_{\UUt \WWt}^T B^T V_{\UUt \WWt}+ V_{\UUt \WWt}^TB^TB  V_{\UUt \WWt}  \right)^{-1/2}.
	\end{align*}
	\noindent It follows  from Lemma \ref{lemma:taylor1} that
	\begin{align*}
	&\left(   \Id +  V_{\UUt \WWt}^T B V_{\UUt \WWt}  +V_{\UUt \WWt}^T B^T V_{\UUt \WWt}+ V_{\UUt \WWt}^TB^TB  V_{\UUt \WWt}  \right)^{-1/2}\\
	=&  \Id - \frac{1}{2} \left( V_{\UUt \WWt}^T B V_{\UUt \WWt}  +V_{\UUt \WWt}^T B^T V_{\UUt \WWt}+ V_{\UUt \WWt}^TB^TB  V_{\UUt \WWt} \right) + C,
	\end{align*}
	where $C$ is matrix, which satisfies
	\begin{equation}\label{ineq:Cboundintern}
	\specnorm{C} \le 3 \specnorm{V_{\UUt \WWt}^T B V_{\UUt \WWt}  +V_{\UUt \WWt}^T B^T V_{\UUt \WWt}+ V_{\UUt \WWt}^TB^TB  V_{\UUt \WWt}   }^2.
	\end{equation}
	It follows that
	\begin{align*}
	& Z \left(Z^T Z\right)^{-1/2} \\
	=&\left( \Id + B  \right) V_{\UUt \WWt} \left(   \Id - \frac{1}{2} \left( V_{\UUt \WWt}^T B V_{\UUt \WWt}  +V_{\UUt \WWt}^T B^T V_{\UUt \WWt}+ V_{\UUt \WWt}^TB^TB  V_{\UUt \WWt} \right) + C  \right)\\
	=& V_{\UUt\WWt}  + B V_{\UUt \WWt} -\frac{1}{2} \left( \Id + B  \right)V_{\UUt \WWt} V^T_{\UUt \WWt} \left( B +B^T \right)    V_{\UUt\WWt} - D,
	\end{align*}
	where we have set 
	\begin{equation}\label{equ:Ddefinition}
	D=   \left(\Id + B\right)V_{\UUt \WWt} \left(  \frac{1}{2}   V_{\UUt \WWt}^T B^T B V_{\UUt \WWt}   -C\right).
	\end{equation}
	Hence,
	\begin{align*}
	&V_{X^\perp}^T  Z \left(Z^T Z\right)^{-1/2}\\
	=& V_{X^\perp}^T  \left(  \Id +  B - \frac{1}{2} V_{\UUt \WWt} V^T_{\UUt \WWt} \left(   B +B^T \right)   \right) V_{\UUt \WWt} - \frac{1}{2} V_{X^\perp}^T  B   V_{\UUt \WWt} V^T_{\UUt \WWt} \left( B +B^T \right)    V_{\UUt \WWt}  - V_{X^\perp}^T D \\
	=&  \bracing{=:(I)}     { V_{X^\perp}^T   \left(  \Id +  B_1 - \frac{1}{2} V_{\UUt \WWt} V^T_{\UUt \WWt} \left(   B_1 +B_1^T \right)   \right) V_{\UUt \WWt} } \\
	&+ \bracing{=:(II)}{ V_{X^\perp}^T   \left(    B_2 - \frac{1}{2} V_{\UUt \WWt} V^T_{\UUt \WWt} \left(   B_2 +B_2^T \right)   \right) V_{\UUt \WWt} } \\
	&+ \bracing{=:(III)} { V_{X^\perp}^T   \left(  B_3 - \frac{1}{2} V_{\UUt \WWt} V^T_{\UUt \WWt} \left(   B_3 +B_3^T \right)   \right) V_{\UUt \WWt}} \\
	&+ \bracing{=:(IV) } {V_{X^\perp}^T  \left(  B_4 - \frac{1}{2} V_{\UUt \WWt} V^T_{\UUt \WWt} \left(   B_4 +B_4^T \right)   \right) V_{\UUt \WWt} }- \bracing{=:(V)}{\frac{1}{2}   V_{X^\perp}^T  B   V_{\UUt \WWt} V^T_{\UUt \WWt} \left( B +B^T \right)    V_{\UUt \WWt}  }  - \bracing{=:(VI)} { V_{X^\perp}^T  D}.
	\end{align*}
	\textbf{Estimating $(I)$:} We observe that
	\begin{align*}
	& V_{X^\perp}^T   \left(  \Id +  B_1 - \frac{1}{2} V_{\UUt \WWt} V^T_{\UUt \WWt} \left(   B_1 +B_1^T \right)   \right) V_{\UUt \WWt}\\
	=&V_{X^\perp}^T \left(  \Id + \mu \left( \Id - V_{\UUt \WWt} V^T_{\UUt \WWt} \right) \left(   XX^T - \UUt \UUt^T  \right)   \right) V_{\UUt \WWt}\\
	=&V_{X^\perp}^T  V_{\UUt \WWt} + \mu  V_{X^\perp}^T  \left( \Id - V_{\UUt \WWt} V^T_{\UUt \WWt} \right) \left(   XX^T - \UUt \UUt^T  \right) V_{\UUt \WWt}\\
	=&V_{X^\perp}^T  V_{\UUt \WWt} + \mu  V_{X^\perp}^T  \left( \Id - V_{\UUt \WWt} V^T_{\UUt \WWt} \right)  XX^TV_{\UUt \WWt} - \mu  V_{X^\perp}^T  \left( \Id - V_{\UUt \WWt} V^T_{\UUt \WWt} \right)  U\UUt^T  V_{\UUt \WWt} \\
	=&V_{X^\perp}^T  V_{\UUt \WWt} -\mu  V_{X^\perp}^T  V_{\UUt \WWt} V^T_{\UUt \WWt} XX^TV_{\UUt \WWt} - \mu  V_{X^\perp}^T  \left( \Id - V_{\UUt \WWt} V^T_{\UUt \WWt} \right)  \UUt \UUt^T  V_{\UUt \WWt} \\
	=&V_{X^\perp}^T  V_{\UUt \WWt} -\mu  V_{X^\perp}^T  V_{\UUt \WWt} V^T_{\UUt \WWt} XX^TV_{\UUt \WWt} - \mu  V_{X^\perp}^T  \left( \Id - V_{\UUt \WWt} V^T_{\UUt \WWt} \right)  \UUt \Wtperp \Wtperp^T \UUt^T  V_{\UUt \WWt} \\
	=&V_{X^\perp}^T  V_{\UUt \WWt} \left( \Id - \mu V^T_{\UUt \WWt} XX^TV_{\UUt \WWt} \right)  - \mu  V_{X^\perp}^T  \left( \Id - V_{\UUt \WWt} V^T_{\UUt \WWt} \right)  \UUt \Wtperp \Wtperp^T \UUt^T   V_{\UUt \WWt}\\
	=&V_{X^\perp}^T  V_{\UUt \WWt} \left( \Id - \mu V^T_{\UUt \WWt} XX^TV_{\UUt \WWt} \right)  - \mu  V_{X^\perp}^T  \left( \Id - V_{\UUt \WWt} V^T_{\UUt \WWt} \right)  \UUt \Wtperp \Wtperp^T \UUt^T  V_{X^\perp}V_{X^\perp}^T V_{\UUt \WWt}.
	\end{align*}
	It follows that
	\begin{align*}
	\specnorm{(I)}	\le& \specnorm{V_{X^\perp}^T  V_{\UUt \WWt} } \left( 1- \mu \sigma_{\min} \left( V^T_{\UUt \WWt} XX^TV_{\UUt \WWt}\right)   \right) + \mu \specnorm{V_{X^\perp}^T  V_{\UUt \WWt} }  \specnorm{\UUt \Wtperp}^2\\
	\le& \specnorm{V_{X^\perp}^T  V_{\UUt \WWt} } \left( 1- \frac{\mu}{2} \sigma_{\min} \left( X \right)^2   \right) + \mu \specnorm{V_{X^\perp}^T  V_{\UUt \WWt} }  \specnorm{\UUt \Wtperp}^2\\
	=& \specnorm{V_{X^\perp}^T  V_{\UUt \WWt} } \left( 1- \frac{\mu}{2} \sigma_{\min} \left( X \right)^2 + \mu  \specnorm{\UUt \Wtperp}^2   \right)\\
	\le &  \specnorm{V_{X^\perp}^T  V_{\UUt \WWt} } \left( 1- \frac{\mu}{3} \sigma_{\min} \left( X \right)^2   \right).
	\end{align*}
	
	\noindent \textbf{Bounding $(II)$:} We observe that
	\begin{align*}
	&V_{X^\perp}^T   \left(    B_2 - \frac{1}{2} V_{\UUt \WWt} V^T_{\UUt \WWt} \left(   B_2 +B_2^T \right)   \right) V_{\UUt \WWt} \\
	&= \mu V_{X^\perp}^T   \left(   \left( \Id - V_{\UUt \WWt} V_{\UUt \WWt}^T \right) \left(   \mathcal{ A^*} \mathcal{A}  - \Id   \right) \left(XX^T - \UUt \UUt^T \right)  \right) V_{\UUt \WWt} .
	\end{align*}
	It follows that
	\begin{align*}
	\specnorm{V_{X^\perp}^T   \left(    B_2 - \frac{1}{2} V_{\UUt \WWt} V^T_{\UUt \WWt} \left(   B_2 +B_2^T \right)   \right) V_{\UUt \WWt} } \le \mu \specnorm{ \left(  \Id - \mathcal{ A^*} \mathcal{A}    \right) \left(XX^T - \UUt \UUt^T \right)  }.
	\end{align*}
	\noindent \textbf{Estimating $(III)$:} First, we recall that
	\begin{align*}
	B_3 &= \UUt \Wtperp\Wtperp^T \WWtplus \left(  V_{\UUt \WWt}^T \UUt \WWt \WWt^T \WWtplus  \right)^{-1} V_{\UUt \WWt}^T\\ 
	&=  \UUt \Wtperp\Wtperp^T \WWtplus \left(  \WWt^T \WWtplus  \right)^{-1}  \left(  V_{\UUt \WWt}^T \UUt \WWt  \right)^{-1}   V_{\UUt \WWt}^T .
	\end{align*}
	Before we proceed further, we need to understand $\Wtperp^T \WWtplus $ and $ \WWt^T \WWtplus  $.  
	By Lemma \ref{lemma:auxiliary} it holds that
	\begin{align*}
	\specnorm{\Wtperp^T \WWtplus }  \le & \mu \left( \frac{1}{800}  \sigma_{\min} ^2\left(X\right)  +   \specnorm{\UUt \WWt} \specnorm{\UUt \Wtperp} \right) \specnorm{ V_{X^\perp}^T V_{\UUt \WWt}  }  + 4\mu \specnorm{   \left( \Id -\mathcal{ A}^*  \mathcal{A}   \right) \left( XX^T - \UUt \UUt^T  \right)  }
	\end{align*}	
	and
	\begin{equation*}
	\sigma_{\min} \left(\WWt^T \WWtplus\right) \ge 1/2.
	\end{equation*}
	\noindent  It follows that
	\begin{align*}
	\specnorm{B_3} &\le \specnorm{ \Wtperp^T \WWtplus  } \specnorm{\UUt \Wtperp} \specnorm{\left(  \WWt^T \WWtplus  \right)^{-1}   } \specnorm{\left(  V_{\UUt \WWt}^T \UUt \WWt  \right)^{-1}} \\
	&= \frac{\specnorm{ \Wtperp^T \WWtplus  } \specnorm{\UUt \Wtperp}}{\sigma_{\min} \left(  \WWt^T \WWtplus \right)   \sigma_{\min} \left(  V_{\UUt \WWt}^T \UUt \WWt   \right)}\\
	&= \frac{\specnorm{ \Wtperp^T \WWtplus } \specnorm{\UUt \Wtperp}}{\sigma_{\min} \left(  \WWt^T \WWtplus \right)   \sigma_{\min} \left(   \UUt \WWt   \right)}\\
	&\le 4 \specnorm{ \Wtperp^T \WWtplus }.
	\end{align*}
	Hence, we can conclude that
	\begin{align*}
	&\specnorm{V_{X^\perp}^T   \left(  B_3 - \frac{1}{2} V_{\UUt \WWt} V^T_{\UUt \WWt} \left(   B_3 +B_3^T \right)   \right) V_{\UUt \WWt} }\\
	\le & 2 \specnorm{B_3} \\
	\le & 8 \specnorm{\Wtperp^T \WWtplus }\\
	\overleq{(a)} & \mu \left( \frac{1}{800}  \sigma_{\min}^2 \left(X\right)  +  8  \specnorm{\UUt \WWt} \specnorm{\UUt \Wtperp} \right) \specnorm{ V_{X^\perp}^T V_{\UUt \WWt}  }  + 32\mu \specnorm{    \left( \Id -\mathcal{ A}^*  \mathcal{A}   \right) \left( XX^T - \UUt \UUt^T  \right)   }\\
	\overleq{(b)} &  \frac{ 1}{400} \mu  \cdot \sigma_{\min}^2 \left(X\right) \specnorm{ V_{X^\perp}^T V_{\UUt \WWt}  } + 32\mu \specnorm{  \left( \Id -\mathcal{ A}^*  \mathcal{A}   \right) \left( XX^T - \UUt \UUt^T  \right) }.
	\end{align*}
	Inequality $(a)$ follows from Lemma  \ref{lemma:auxiliary}. In $(b)$ we used the assumption $	\specnorm{\UUt \Wtperp} \le  c \kappa^{-2} \sigma_{\min} \left(X\right) $ and $ \specnorm{\UUt \WWt} \le 3 \specnorm{X} $.\\
	
	\noindent \textbf{Bounding $(IV)$:} We start by noticing that
	\begin{align*}
	\mu \specnorm{ \mathcal{A}^* \mathcal{ A} \left(XX^T - \UUt \UUt^T \right) } &\le \mu \left(  \specnorm{XX^T -\UUt \UUt^T } + \specnorm{\left( \Id -\mathcal{ A}^* \mathcal{ A} \right) \left(XX^T -\UUt \UUt^T \right)} \right)\\
	&\le 11\mu \specnorm{X}^2\\
	&\le 11 c  \kappa^{-2}, 
	\end{align*}	
	where we have used the assumption $\specnorm{U} \le 3 \specnorm{X} $, \eqref{intern:eqref1} and \eqref{intern:eqref3}. Hence, we obtain that
	\begin{align*}
	&\specnorm{V_{X^\perp}^T  \left(  B_4 - \frac{1}{2} V_{\UUt \WWt} V^T_{\UUt \WWt} \left(   B_4 +B_4^T \right)   \right) V_{\UUt \WWt} }\\
	\le &2 \specnorm{B_4}\\
	=& 2\mu \specnorm{MP}\\
	\le & 2 \mu \specnorm{\mathcal{A}^* \mathcal{ A} \left(XX^T - \UUt \UUt^T \right) } \specnorm{ B_3 }\\
	\le &22  c \kappa^{-2} \specnorm{B_3}\\
	\overleq{(a)}&   \frac{\mu}{50} \cdot  \kappa^{-2} \sigma_{\min}^2 \left(X\right) \specnorm{ V_{X^\perp}^T V_{\UUt \WWt}  }   + 352 c\mu \specnorm{ \left[    \left( \Id -\mathcal{ A}^*  \mathcal{A}   \right) \left( XX^T - \UUt \UUt^T  \right)  \right] }.
	\end{align*}
	Inequality $(a)$ follows from similar arguments as when we were bounding $(III)$ and by choosing the constant $c>0$ small enough.\\
	
	\noindent \textbf{Bounding $(V)$:} First, we want to estimate $\specnorm{B}$. We note that 
	\begin{align*}
	\specnorm{B} &\overleq{(a)}  \mu \specnorm{M } + \specnorm{P} + \mu \specnorm{MP} \\
	&\overleq{(b)}  \mu \specnorm{M } + \specnorm{B_3} + \mu \specnorm{M} \specnorm{B_3} \\
	&\overeq{(c)} \mu \specnorm{\mathcal{ A^* \mathcal{A}} \left(XX^T - \UUt \UUt^T \right) } + \left( 1  + \mu \specnorm{\mathcal{ A^* \mathcal{A}} \left(XX^T - \UUt \UUt^T \right)} \right)  \specnorm{B_3} \\
	&\overleq{(d)}   \mu \specnorm{\mathcal{ A^* \mathcal{A}} \left(XX^T - \UUt \UUt^T \right) } + 2 \specnorm{B_3} \\
	&\overleq{(e)}   \mu \specnorm{\mathcal{ A^* \mathcal{A}} \left(XX^T - \UUt \UUt^T \right) } + \frac{ 1}{400} \mu  \cdot \sigma_{\min}^2 \left(X\right)\specnorm{ V_{X^\perp}^T V_{\UUt \WWt}  } + 32 \mu \specnorm{ \left[    \left( \Id -\mathcal{ A}^*  \mathcal{A}   \right) \left( XX^T - \UUt \UUt^T  \right)  \right] } \\
	& \le   \mu \specnorm{ XX^T - \UUt \UUt^T  } + \frac{ 1}{400} \mu  \cdot \sigma_{\min}^2 \left(X\right) \specnorm{ V_{X^\perp}^T V_{\UUt \WWt}  } + 33 \mu \specnorm{ \left[    \left( \Id -\mathcal{ A}^*  \mathcal{A}   \right) \left( XX^T - \UUt \UUt^T  \right)  \right] }. 
	\end{align*}
	In $(a)$ we used the triangle inequality and in $(b)$ we used $B_3=P$ and the submultiplicativity of the spectral norm. To obtain equality $(c)$ we inserted the definition of $M$. For $(d)$ we used that $ \mu \specnorm{\mathcal{ A^* \mathcal{A}} \left(XX^T - \UUt \UUt^T \right)} \le 2  $, which follows from assumption \eqref{intern:eqref1} and \eqref{intern:eqref3}. For inequality $(e)$ we used our bound for $\specnorm{B_3}$, which we have derived when bounding $(III)$. Hence, we have shown that
	\begin{equation}\label{ineq;internBbound}
	\specnorm{B} \le    \mu \specnorm{XX^T -\UUt \UUt^T  } + \frac{ 1}{400} \mu  \cdot \sigma_{\min}^2 \left(X\right) \specnorm{ V_{X^\perp}^T V_{\UUt \WWt}  } + 33  \mu \specnorm{    \left( \Id -\mathcal{ A}^*  \mathcal{A}   \right) \left( XX^T - \UUt \UUt^T  \right)  } . 
	\end{equation}
	We obtain that
	\begin{equation}\label{ineq:intern932}
	\begin{split}
	& \frac{1}{2}  \specnorm{V_{X^\perp}^T  B   V_{\UUt \WWt} V^T_{\UUt \WWt} \left( B +B^T \right)    V_{\UUt \WWt} } \\
	\le & \frac{1}{2}   \specnorm{B} \specnorm{B + B^T} \\
	\le&   \specnorm{B}^2\\
	\overleq{(a)}  & 3 \mu^2 \left(    \specnorm{XX^T -\UUt \UUt^T  }^2 + \frac{1}{400^2}  \sigma_{\min}^4 \left(X\right) \specnorm{ V_{X^\perp}^T V_{\UUt \WWt}  }^2  + 33^2   \specnorm{    \left( \Id -\mathcal{ A}^*  \mathcal{A}   \right) \left( XX^T - \UUt \UUt^T  \right)  }^2    \right)\\
	\overleq{(b)} & 3 \mu^2   \specnorm{XX^T -\UUt \UUt^T  }^2  + \frac{3}{2}c \frac{\mu \kappa^{-4}}{400^2}  \sigma_{\min}^2 \left(X\right) \specnorm{ V_{X^\perp}^T V_{\UUt \WWt}  }^2  + \frac{3}{2} \mu^2 \sigma_{\min }^2\left(X\right)  \specnorm{    \left( \Id -\mathcal{ A}^*  \mathcal{A}   \right) \left( XX^T - \UUt \UUt^T  \right)  }\\
	\overleq{(c)} & 3 \mu^2   \specnorm{XX^T -\UUt \UUt^T  }^2  + 3c \frac{\mu \kappa^{-4}}{2 \cdot 400^2}  \sigma_{\min}^2 \left(X\right) \specnorm{ V_{X^\perp}^T V_{\UUt \WWt}  }^2  + \frac{3}{2} c \mu \kappa^{-4}   \specnorm{    \left( \Id -\mathcal{ A}^*  \mathcal{A}   \right) \left( XX^T - \UUt \UUt^T  \right)  },
	\end{split}
	\end{equation} 
	where in $(a)$ we used inequality \eqref{ineq;internBbound} combined with Jensen's inequality. For inequality $(b)$ we used \eqref{intern:eqref1} and \eqref{intern:eqref3}. Inequality $(c)$ follows again from \eqref{intern:eqref3}.\\
	
	\noindent \textbf{Bounding $(VI)$:} We are first going to show that  $\specnorm{B} \le 1 $. Indeed, we have that
	\begin{align*}
	\specnorm{B}  & \overleq{(a)}    10 \mu  \specnorm{X}^2 + \frac{ 1}{400} \mu  \cdot \sigma_{\min}^2 \left(X\right) \specnorm{ V_{X^\perp}^T V_{\UUt \WWt}  } + 33  \mu \specnorm{    \left( \Id -\mathcal{ A}^*  \mathcal{A}   \right) \left( XX^T - \UUt \UUt^T  \right)  }\\
	&\overleq{(b)}  10 \mu  \specnorm{X}^2 + \frac{ 1}{400} \mu  \cdot \sigma_{\min} ^2\left(X\right)  + 33  c \mu\sigma_{\min }^2 \left(X\right) \\
	&\overleq{(c)}  1,
	\end{align*}
	where in $(a)$ we used inequality \eqref{ineq;internBbound} and the assumption $ \specnorm{U_t} \le 3\specnorm{X} $. For inequality $(b)$ we used \eqref{intern:eqref1} and for inequality $(c)$ we used assumption \eqref{intern:eqref3}. 
	We note that from inequality \eqref{equ:Ddefinition} it follows that
	\begin{align}\label{ineq:intern931}
	\specnorm{D} \le \left( 1+ \specnorm{B}  \right) \left( \frac{1}{2}   \specnorm{B}^2 + \specnorm{C}   \right) \le 2\left(   \specnorm{B}^2 + \specnorm{C}   \right).
	\end{align}
	In the first inequality we used the triangle inequality and in the second inequality we used $ \specnorm{B} \le  1$.	 Note that from \eqref{ineq:Cboundintern} and again $ \specnorm{B} \le 1  $ it follows that
	\begin{align}\label{ineq:intern930}
	\specnorm{C} &\le 3 \left(2 \specnorm{B} + \specnorm{B}^2\right)^2  \le 27 \specnorm{B}^2.
	\end{align}
	Hence, we obtain that
	\begin{align*}
	&\specnorm{D} \\
	\overleq{(a)} & 56 \specnorm{B}^2\\
	\overleq{(b)} &   56 \left(  3 \mu^2   \specnorm{XX^T -\UUt \UUt^T  }^2  + 3c \frac{\mu \kappa^{-4}}{2\cdot 400^2}  \sigma_{\min}^2 \left(X\right) \specnorm{ V_{X^\perp}^T V_{\UUt \WWt}  }^2  + \frac{3}{2} c \mu \kappa^{-4}   \specnorm{    \left( \Id -\mathcal{ A}^*  \mathcal{A}   \right) \left( XX^T - \UUt \UUt^T  \right)  }  \right).
	\end{align*}
	Inequality $(a)$ is due to the inequalities \eqref{ineq:intern931} and \eqref{ineq:intern930} and inequality $(b)$ follows from \eqref{ineq:intern932}.\\
	
	\noindent \textbf{Combining the estimates:} By combining our results we obtain that for small enough $c>0$ we have that
	\begin{align*}
	&\specnorm{ V_{X^\perp}^T V_{\UUtplus \WWtplus} }\\  
	\le& \specnorm{(I)} + \specnorm{(II)}  + \specnorm{(III)} +\specnorm{(IV)} +\specnorm{(V)}  + \specnorm{(VI)}  \\
	\le &  \left(   1-    \frac{\mu}{4} \sigma_{\min}^2 \left( X \right)   \right) \specnorm{ V_{X^\perp}^T V_{\UUt \WWt} }  + 100   \mu  \specnorm{     \left( \Id -\mathcal{ A}^*  \mathcal{A}   \right) \left( XX^T - \UUt \UUt^T  \right)  }  +500  \mu^2  \specnorm{XX^T -\UUt \UUt^T  }^2 .
	\end{align*}
	This finishes the proof.
\end{proof}

\subsection{Proof of Lemma \ref{lemma:Ucontrol}}

\begin{proof}[Proof of Lemma \ref{lemma:Ucontrol}]
	We observe that
	\begin{align*}
	\UUtplus &= \UUt + \mu \left(XX^T - \UUt \UUt^T\right) \UUt + \mu \left[ \left( \mathcal{ A}^* \mathcal{A} - \Id  \right) \left(  XX^T - \UUt \UUt^T \right)  \right] \UUt\\
	&= \left( \Id - \mu \UUt \UUt^T   \right) \UUt  + \mu XX^T \UUt + \mu \left[ \left( \mathcal{ A}^* \mathcal{A} - \Id  \right) \left(  XX^T - \UUt \UUt^T \right) \right] \UUt.
	\end{align*}
	Note that $ \specnorm{ \left( \Id -  \mu \UUt \UUt^T   \right) \UUt  } = \left( 1- \mu \specnorm{ \UUt }^2 \right) \specnorm{\UUt} $ due to $ \mu \le  \frac{1}{27 \specnorm{X}^{2}  } \le \frac{1}{3 \specnorm{\UUt}^2 } $. Hence, by the triangle inequality and submultiplicativity of the spectral norm we obtain that
	\begin{align*}
	\specnorm{\UUtplus } \le \left( 1- \mu \specnorm{\UUt}^2 +\mu \specnorm{X}^2 + \mu \specnorm{ \left( \mathcal{ A}^* \mathcal{A} - \Id  \right) \left(  XX^T - \UUt \UUt^T \right) }      \right) \specnorm{\UUt}.
	\end{align*}
	Hence, by our assumption on $\specnorm{    \left( \Id - \mathcal{ A}^* \mathcal{ A} \right)  \left(XX^T - \UUt \UUt^T\right)   } $ we obtain that
	\begin{equation}\label{ineq:intern3456}
	\specnorm{\UUtplus } \le \left( 1-   \mu \specnorm{\UUt}^2 + 2 \mu \specnorm{X}^2  \right) \specnorm{\UUt}.
	\end{equation}
	Now assume that $  2 \specnorm{X}  \le \specnorm{\UUt} \le 3  \specnorm{X} $. Then it follows from the last inequality that $\specnorm{\UUtplus} \le \specnorm{\UUt}$, which due to the assumption $ \specnorm{\UUt} \le 3 \specnorm{X}  $ implies the claim  $\specnorm{\UUtplus} \le  3 \specnorm{X}$. However, if $ \specnorm{\UUt} \le  2 \specnorm{X} $ holds,  then by combining inequality \eqref{ineq:intern3456} with the assumption $ \mu \le   \frac{1}{27 \specnorm{X}^2 }$ we obtain that $ \specnorm{\UUtplus} \le 3 \specnorm{X} $ as well, which finishes the proof.
\end{proof}

\subsection{Proof of Lemma \ref{lemma:localconvergence}}

First, we prove the following auxiliary lemma.

\begin{lemma}\label{lemma:technicallemma8}
	Under the assumptions of Lemma \ref{lemma:localconvergence} it holds that
	\begin{equation*}
	\Vnorm{ V_{X^\perp}^T \UUt \UUt^T  } \le   3 \Vnorm{ V_X^T \left(  XX^T -\UUt \UUt^T  \right)} + \Vnorm{  \UUt \Wtperp\Wtperp^T \UUt^T  }
	\end{equation*}
	as well as 
	\begin{equation*}
	\Vnorm{XX^T -\UUt \UUt^T} \le 4 \Vnorm{ V_X^T \left(  XX^T -\UUt \UUt^T  \right)} + \Vnorm{  \UUt \Wtperp\Wtperp^T \UUt^T  } .
	\end{equation*}
\end{lemma}

\begin{proof}
	We notice that by the triangle inequality and submultiplicativity it holds that
	\begin{align*}
	\Vnorm{ V_{X^\perp}^T \UUt \UUt^T  }  &\le \Vnorm{ V_{X^\perp}^T \UUt \UUt^T V_X  }  + \Vnorm{ V_{X^\perp}^T \UUt \UUt^T V_{X^\perp} }  \\
	&=   \Vnorm{ V_{X^\perp}^T  \left(XX^T - \UUt \UUt^T\right)  V_X  }  + \Vnorm{ V_{X^\perp}^T \UUt \UUt^T V_{X^\perp} }  \\
	&\le  \Vnorm{ V_{X}^T  \left(XX^T - \UUt \UUt^T\right)    }  + \Vnorm{ V_{X^\perp}^T \UUt \UUt^T V_{X^\perp} }.
	\end{align*}
	In order to bound the second term we compute that
	\begin{align*}
	\Vnorm{V_{X^\perp}^T  \UUt \UUt^T V_{X^\perp} } &\le \Vnorm{  V_{X^\perp}^T  \UUt \WWt \WWt^T \UUt^T V_{X^\perp}  }    + \Vnorm{  V_{X^\perp}^T  \UUt \Wtperp\Wtperp^T\UUt^T V_{X^\perp}  } \\
	&=  \Vnorm{  V_{X^\perp}^T  \UUt \WWt \WWt^T \UUt^T V_{X^\perp}  }    + \Vnorm{  \UUt \Wtperp\Wtperp^T \UUt^T  }.
	\end{align*}
	In order to bound the first term we note that
	\begin{align*}
	\Vnorm{  V_{X^\perp}^T  \UUt \WWt \WWt^T \UUt^T V_{X^\perp}  } &=  	\Vnorm{  V_{X^\perp}^T V_{\UUt \WWt} V_{\UUt \WWt}^T  \UUt \WWt \WWt^T \UUt^T V_{X^\perp}  } \\
	&= \Vnorm{V_{X^\perp}^T  V_{\UUt \WWt} \left( V_X^T V_{\UUt \WWt} \right)^{-1}  V_X^T V_{\UUt \WWt}  V_{\UUt \WWt}^T \UUt \WWt \WWt^T \UUt^T V_{X^\perp}   }\\
	&\le \specnorm{V_{X^\perp}^T  V_{\UUt \WWt}} \specnorm{ \left( V_X^T V_{\UUt \WWt} \right)^{-1} }  \Vnorm{ V_X^T V_{\UUt \WWt} V_{\UUt \WWt}^T \UUt \WWt \WWt^T \UUt^T V_{X^\perp} }\\
	&= \frac{ \specnorm{V_{X^\perp}^T  V_{\UUt \WWt}}}{\sigma_{\min} \left(   V_X^T V_{\UUt \WWt}\right)}  \Vnorm{ V_X^T \UUt \UUt^T V_{X^\perp} } \\
	&= \frac{ \specnorm{V_{X^\perp}^T  V_{\UUt \WWt}}}{\sigma_{\min} \left(   V_X^T V_{\UUt \WWt}\right)}  \Vnorm{ V_X^T \left(XX^T -\UUt \UUt^T\right) V_{X^\perp} }\\
	&\le  \frac{ \specnorm{V_{X^\perp}^T  V_{\UUt \WWt}}}{\sigma_{\min} \left(   V_X^T V_{\UUt \WWt}\right)}  \Vnorm{ V_X^T \left(XX^T -\UUt \UUt^T\right)  }\\
	&\le  2  \Vnorm{ V_X^T \left(XX^T -\UUt \UUt^T\right)  }.
	\end{align*}
	Hence we can conclude that
	\begin{align*}
	\Vnorm{ V_{X^\perp}^T \UUt \UUt^T  }  &\le 3  \Vnorm{ V_X^T \left(XX^T -\UUt \UUt^T\right)  }    + \Vnorm{\UUt \Wtperp\Wtperp^T \UUt^T  },
	\end{align*}
	which shows the first inequality  in the lemma. In order to prove the second inequality, we note that by the triangle inequality and submultiplicativity it holds that
	\begin{align*}
	\Vnorm{XX^T -\UUt \UUt^T} &\le \Vnorm{ V_X^T \left(  XX^T -\UUt \UUt^T  \right)} + \Vnorm{ V_{X^\perp}^T \UUt \UUt^T   }\\
	&\le 4 \Vnorm{ V_X^T \left(  XX^T -\UUt \UUt^T  \right)}  +\Vnorm{  \UUt \Wtperp\Wtperp^T \UUt^T  },
	\end{align*}
	where in the last line we used the previous inequality. This finishes the proof.
\end{proof}
After having provided the necessary ingredients, we are in a position to prove Lemma \ref{lemma:localconvergence}.
\begin{proof}[Proof of Lemma \ref{lemma:localconvergence}]
	Recall that
	\begin{equation*}
	\UUtplus= \UUt +   \mu \left[  \left( \mathcal{A}^* \mathcal{A}  \right)        \left(XX^T -\UUt \UUt^T \right)  \right] \UUt.
	\end{equation*}
	Next, we compute that
	\begin{align*}
	XX^T-\UUtplus\UUtplus^T=  &XX^T - \UUt \UUt^T  - \mu \left[  \left( \mathcal{A}^* \mathcal{A}  \right)        \left(XX^T -\UUt \UUt^T \right)  \right]\UUt \UUt^T  - \mu  \UUt \UUt^T  \left[  \left( \mathcal{A}^* \mathcal{A}  \right)        \left(XX^T -\UUt \UUt^T \right)  \right]  \\
	& -  \mu^2  \left[  \left( \mathcal{A}^* \mathcal{A}  \right)        \left(XX^T -\UUt \UUt^T \right)  \right] \UUt \UUt^T  \left[  \left( \mathcal{A}^* \mathcal{A}  \right)        \left(XX^T -\UUt \UUt^T \right)  \right] \\
	=  &XX^T - \UUt \UUt^T  -     \mu  \left(XX^T -\UUt \UUt^T \right)  \UUt \UUt^T  - \mu  \UUt \UUt^T       \left(XX^T -\UUt \UUt\right) \\
	&+ \mu  \left[   \left( \Id- \mathcal{A}^* \mathcal{A} \right) \left(XX^T -\UUt \UUt^T \right) \right] \UUt \UUt^T    + \mu  \UUt \UUt^T  \left[ \left( \Id- \mathcal{A}^* \mathcal{A} \right) \left(XX^T -\UUt \UUt^T \right) \right]    \\
	& -  \mu^2  \left[  \left( \mathcal{A}^* \mathcal{A}  \right)        \left(XX^T -U\UUt^T \right)  \right] \UUt \UUt^T  \left[  \left( \mathcal{A}^* \mathcal{A}  \right)        \left(XX^T -\UUt \UUt^T \right)  \right] \\
	=&  \bracing{=(I)}{\left( \Id - \mu  \UUt \UUt^T  \right) \left( XX^T -\UUt \UUt^T  \right) \left( \Id -   \mu \UUt \UUt^T   \right)}  + \mu  \bracing{=(II)} {\left[   \left( \Id- \mathcal{A}^* \mathcal{A} \right) \left(XX^T -\UUt\UUt\right) \right] \UUt \UUt^T } \\
	& +  \mu  \bracing{=(III)} { \UUt \UUt^T \left[ \left( \Id- \mathcal{A}^* \mathcal{A} \right) \left(XX^T -\UUt \UUt^T \right) \right] } - \mu^2\bracing{=(IV)} {   \UUt \UUt \left(XX^T -\UUt \UUt^T \right)  \UUt \UUt^T   }\\
	&  - \mu^2  \bracing{ =(V) }{  \left[  \left( \mathcal{A}^* \mathcal{A}  \right)        \left(XX^T -\UUt \UUt^T \right)  \right]   \UUt \UUt^T   \left[  \left( \mathcal{A}^* \mathcal{A}  \right)        \left(XX^T -\UUt \UUt^T \right)  \right]   }.
	\end{align*}
	We are going to deal with each summand individually.\\
	
	\noindent \textbf{Estimation of $(I)$:} We note that
	\begin{align*}
	&V_{X}^T   \left( \Id - \mu \UUt \UUt^T  \right) \left( XX^T -\UUt\UUt^T  \right) \left( \Id -\mu  \UUt \UUt  \right)\\
	=& V_{X}^T   \left( \Id -  \mu \UUt \UUt^T  \right) V_{X} V_{X}^T    \left( XX^T -\UUt \UUt^T  \right) \left( \Id - \mu  \UUt \UUt^T   \right)\\
	& +  V_{X}^T   \left( \Id -  \mu \UUt \UUt^T  \right) V_{X^\perp} V_{X^\perp}^T    \left( XX^T -\UUt \UUt^T  \right) \left( \Id - \mu \UUt \UUt^T   \right) \\
	=& V_{X}^T   \left( \Id -  \mu \UUt \UUt^T  \right) V_{X} V_{X}^T    \left( XX^T -\UUt \UUt \right) \left( \Id - \mu \UUt \UUt^T   \right) + \mu   V_{X}^T  \UUt \UUt^T  V_{X^\perp} V_{X^\perp}^T    \UUt \UUt^T \left( \Id - \mu \UUt \UUt  \right) \\
	=& \left( \Id - \mu  V_{X}^T   \UUt \UUt V_{X}  \right) V_{X}^T    \left( XX^T -\UUt \UUt^T  \right) \left( \Id -  \mu \UUt \UUt^T   \right) + \mu  V_{X}^T  \UUt \UUt^T  V_{X^\perp} V_{X^\perp}^T    \UUt \UUt^T \left( \Id - \mu \UUt \UUt^T   \right) .
	\end{align*}
	Hence, we obtain that
	\begin{align*}
	&\Vnorm{  \left( \Id -  \mu V_{X}^T   \UUt \UUt^T  V_{X}  \right) V_{X}^T    \left( XX^T -\UUt \UUt^T  \right) \left( \Id -  \mu \UUt \UUt^T   \right)    }\\
	\le & \specnorm{  \left( \Id - \mu V_{X}^T   \UUt \UUt^T  V_{X}  \right) } \Vnorm{  V_{X}^T    \left( XX^T -\UUt \UUt^T  \right) } \specnorm{\left( \Id -\mu  \UUt \UUt^T   \right)    } \\
	\le & \specnorm{  \left( \Id - \mu V_{X}^T   \UUt \UUt^T  V_{X}  \right) } \Vnorm{  V_{X}^T    \left( XX^T -\UUt \UUt^T  \right) } \\
	=&  \left(  1 - \mu \sigma_{\min} \left(   V_{X}^T   \UUt \UUt^T  V_{X}   \right)   \right) \Vnorm{  V_{X}^T    \left( XX^T -\UUt \UUt^T  \right) }  \\
	\le & \left(  1 -  \mu \sigma_{\min}^2 \left(  V_X^T \UUt \WWt\right) \right)\Vnorm{  V_{X}^T    \left( XX^T -\UUt \UUt^T  \right) }.
	\end{align*}
	Next, we note that 
	\begin{align*}
	\sigma_{\min}^2 \left(  V_X^T \UUt \WWt\right)  &=   \sigma_{\min}^2 \left(  V_X^T V_{\UUt \WWt} V_{\UUt \WWt}^T \UUt \WWt\right) \\
	&\ge \sigma_{\min}^2  \left(  V_X^T V_{\UUt \WWt} \right)  \sigma_{\min}^2 \left( \UUt \WWt\right)\\
	&\ge \frac{1}{2} \sigma_{\min}^2 \left( \UUt \WWt\right) \\
	& \ge   \frac{1}{20} \sigma_{\min}^2 \left( X\right),
	\end{align*}
	where in the last line we used the assumption   $\sigma_{\min}^2 \left(\UUt \WWt\right) \ge \frac{1}{10} \sigma_{\min}^2 \left(X\right)$. Hence, we have shown that
	\begin{equation*}
	\Vnorm{  \left( \Id - \mu  V_{X}^T   \UUt \UUt^T  V_{X}  \right) V_{X}^T    \left( XX^T -\UUt \UUt^T  \right) \left( \Id -  \mu \UUt \UUt^T   \right)    } \le   \left(  1 - \frac{\mu}{20}  \sigma_{\min}^2 \left(  X\right) \right)\Vnorm{  V_{X}^T    \left( XX^T -\UUt \UUt^T  \right) }.
	\end{equation*}
	Next, we note that
	\begin{align*}
	&\Vnorm{  V_{X}^T  \UUt \UUt^T  V_{X^\perp} V_{X^\perp}^T    \UUt \UUt^T \left( \Id - \mu \UUt \UUt  \right) }\\ \le &  \Vnorm{  V_{X}^T  \UUt \UUt^T  V_{X^\perp} V_{X^\perp}^T    \UUt \UUt }  \specnorm{ \Id - \mu \UUt \UUt^T     } \\
	\le&  \Vnorm{  V_{X}^T  \UUt \UUt^T V_{X^\perp} V_{X^\perp}^T    \UUt \UUt^T  }  \\
	\le&  \Vnorm{  V_{X}^T  \UUt \WWt \WWt^T \UUt V_{X^\perp} V_{X^\perp}^T    \UUt \UUt^T  }  \\
	\le&  \specnorm{ V_{X}^T  \UUt \WWt } \specnorm{ V_{X^\perp}^T \UUt \WWt } \Vnorm{  V_{X^\perp}^T    \UUt \UUt^T   } \\
	\le&  \specnorm{   \UUt \WWt }^2 \specnorm{ V_{X^\perp}^T V_{\UUt \WWt}  }      \Vnorm{  V_{X^\perp}^T    \UUt \UUt^T   }\\
	\le &  9 \specnorm{X}^2 \specnorm{ V_{X^\perp}^T V_{\UUt \WWt}  }      \Vnorm{  V_{X^\perp}^T    \UUt \UUt^T   }\\
	\le &   9   \specnorm{X}^2   \specnorm{ V_{X^\perp}^T V_{\UUt \WWt}  }  \left(   3 \Vnorm{  V_{X}^T     \left(XX^T - \UUt \UUt^T \right)  } + \Vnorm{  \UUt \Wtperp\Wtperp^T \UUt  }\right) ,
	\end{align*}
	where in the last line we used Lemma \ref{lemma:technicallemma8}. Then, using the assumption $  \specnorm{V_{X^\perp}^T V_{\UUt \WWt}} \le c \kappa^{-2} $ it follows that
	\begin{align*}
	&\Vnorm{  V_{X}^T  \UUt \UUt^T  V_{X^\perp} V_{X^\perp}^T    U\UUt^T \left( \Id -  \UUt \UUt^T   \right) } \\
	\le &  \frac{1}{100}  \sigma_{\min}^2\left(X\right)   \Vnorm{  V_{X}^T     \left(XX^T - \UUt \UUt^T \right)  } +  \frac{\sigma_{\min}^2 \left(X\right) }{400}  \Vnorm{  \UUt \Wtperp\Wtperp^T \UUt  } .
	\end{align*}
	Hence, we have shown that 
	\begin{align*}
	&\Vnorm{V_{X}^T   \left( \Id - \mu \UUt \UUt^T  \right) \left( XX^T -\UUt \UUt^T  \right) \left( \Id -  \mu U\UUt^T   \right) } \\
	\le  &  \left(  1 - \frac{\mu}{40}  \sigma_{\min}^2 \left(  X\right) \right)\Vnorm{  V_{X}^T    \left( XX^T -\UUt \UUt^T  \right) } + \mu \frac{\sigma_{\min}^2 \left(X\right) }{400} \Vnorm{  \UUt \Wtperp\Wtperp^T \UUt  }.
	\end{align*}
	\noindent \textbf{Estimation of $(II)$:} We note that
	\begin{align*}
	\Vnorm{ V_X^T \left[   \left( \Id- \mathcal{A}^* \mathcal{A} \right) \left(XX^T -\UUt \UUt^T \right) \right] \UUt \UUt^T   } & \le \Vnorm{  V_X^T \left[   \left( \Id- \mathcal{A}^* \mathcal{A} \right) \left(XX^T -\UUt \UUt^T \right) \right]   } \specnorm{\UUt}^2\\
	&\lesssim \Vnorm{   V_X^T \left[   \left( \Id- \mathcal{A}^* \mathcal{A} \right) \left(XX^T -\UUt \UUt^T \right) \right]  } \specnorm{X}^2\\
	&\lesssim  c \sigma_{\min}^2 \left(X\right) \Vnorm{XX^T -\UUt \UUt^T }\\
	&\lesssim   c \sigma_{\min}^2 \left(X\right)  \left(    \Vnorm{ V_X^T \left(  XX^T -\UUt \UUt^T   \right)} + \Vnorm{  \UUt \Wtperp\Wtperp^T \UUt  }     \right).
	\end{align*}
	In the second inequality we used the assumption $ \specnorm{U} \le 3 \specnorm{X} $ and in the third inequality we used assumption \eqref{ineq:intern773}. In the fourth inequality we applied Lemma \ref{lemma:technicallemma8}. Hence, by choosing the constant $c>0$ small enough, we obtain that
	\begin{align*}
	\Vnorm{ V_X^T \left[   \left( \Id- \mathcal{A}^* \mathcal{A} \right) \left(XX^T -U\UUt^T \right) \right] \UUt \UUt^T   } \le \frac{1}{1000} \sigma_{\min}^2 \left(X\right) \left(    \Vnorm{ V_X^T \left(  XX^T -U\UUt^T   \right)} + \Vnorm{  \UUt \Wtperp\Wtperp^T \UUt^T   }      \right).
	\end{align*}
	
	\noindent \textbf{Estimation of $(III)$:} We derive that
	\begin{align*}
	&\Vnorm{ V_X^T \UUt \UUt^T   \left[   \left( \Id- \mathcal{A}^* \mathcal{A} \right) \left(XX^T -\UUt \UUt^T \right) \right]  } \\
	=& \Vnorm{ V_X^T \UUt \WWt \WWt^T \UUt^T   \left[   \left( \Id- \mathcal{A}^* \mathcal{A} \right) \left(XX^T -\UUt \UUt^T \right) \right]  } \\
	=& \Vnorm{ V_X^T \UUt \WWt \WWt^T \UUt^T V_{\Ut \WWt} V_{\Ut \WWt}^T  \left[   \left( \Id- \mathcal{A}^* \mathcal{A} \right) \left(XX^T -\UUt \UUt^T \right) \right]  } \\
	=& \specnorm{ V_X^T \UUt \WWt \WWt^T \UUt^T V_{\Ut \WWt} } \Vnorm{V_{\Ut \WWt}^T  \left[   \left( \Id- \mathcal{A}^* \mathcal{A} \right) \left(XX^T -\UUt \UUt^T \right) \right]  } \\
	\overleq{(a)}&  9 \specnorm{X}^2 \Vnorm{V_{\Ut \WWt}^T  \left[   \left( \Id- \mathcal{A}^* \mathcal{A} \right) \left(XX^T -\UUt \UUt^T \right) \right]  } \\
	\overleq{(b)}&  9 c  \sigma_{\min} \left( X \right)^2  \\
	\le& \frac{1}{1000} \sigma_{\min}^2 \left(X\right)  \left(    \Vnorm{ V_X^T \left(  XX^T -\UUt \UUt^T   \right)} + \Vnorm{  \UUt \Wtperp\Wtperp^T \UUt  }     \right).
	\end{align*}
	In inequality $(a)$ we used the assumption $\specnorm{\Ut} \le 3 \specnorm{X}$ and in inequality $(b)$ we used assumption \eqref{ineq:intern773}. 
	The last line follows since the absolute constant $c>0$ has been chosen small enough.\\ 
	
	\noindent \textbf{Estimation of $(IV)$:} We note that it follows from submultiplicativity of the spectral norm that
	\begin{align*}
	\Vnorm{ V_X^T  \UUt \UUt^T  \left(XX^T -\UUt \UUt^T \right)  \UUt \UUt^T   } &\le  \specnorm{U}^4 \Vnorm{XX^T -\UUt \UUt} \\
	&\lesssim    \specnorm{X}^4 \Vnorm{XX^T -\UUt \UUt^T } \\
	&\lesssim   \specnorm{X}^4 \Vnorm{V_X^T \left(  XX^T -\UUt \UUt^T  \right)}  + \specnorm{X}^4   \specnorm{\UUt \Wtperp}^2 ,
	\end{align*}
	where in the second line we used the assumption $ \specnorm{\UUt} \le 3 \specnorm{X} $. In the third line we used Lemma \ref{lemma:technicallemma8}. Then using the assumption $ \mu \le c \kappa^{-2} \specnorm{X}^{-2} $ it follows that
	\begin{equation*}
	\mu^2 \Vnorm{ V_X^T  \UUt \UUt^T  \left(XX^T -\UUt \UUt^T \right)  \UUt \UUt^T   } \le \frac{\mu}{200} \sigma_{\min}^2 \left(  X  \right)  \Vnorm{ V_X^T \left( XX^T - \UUt \UUt^T  \right)} +  \mu  \frac{\sigma_{\min}^2 \left(X\right) }{1000}  \Vnorm{  \UUt \Wtperp\Wtperp^T \UUt^T  } .
	\end{equation*}
	
	\noindent \textbf{Estimation of $(V)$:} We first note that 
	\begin{align*}
	\Vnorm{ V_X^T  \left( \mathcal{A}^* \mathcal{A}  \right)        \left(XX^T -\UUt \UUt^T \right)    }  &\le  \Vnorm{XX^T -\UUt \UUt^T }    + \Vnorm{V_X^T \left[  \left( \Id - \mathcal{A}^* \mathcal{A}  \right)        \left(XX^T -\UUt \UUt^T \right)  \right]    }  \\
	&\le  \left(1+c  \kappa^{-2}\right) \Vnorm{XX^T -\UUt \UUt} \\
	&\le 2  \Vnorm{XX^T -\UUt \UUt^T },
	\end{align*}
	where we have used assumption \eqref{ineq:intern773}.
	In a similar manner, again using assumption \eqref{ineq:intern773}, we can show that 
	\begin{equation*}
	\specnorm{  \left( \mathcal{A}^* \mathcal{A}  \right)        \left(XX^T -\UUt \UUt^T \right)     } \le 2 \Vnorm{XX^T - \UUt \UUt^T } \le 2 \left(  \specnorm{X}^2 + \specnorm{\UUt}^2 \right).
	\end{equation*}
	Hence, it follows that
	\begin{align*}
	&\Vnorm{ V_X^T  \left[  \left( \mathcal{A}^* \mathcal{A}  \right)        \left(XX^T -\UUt \UUt^T \right)  \right]   \UUt \UUt^T   \left[  \left( \mathcal{A}^* \mathcal{A}  \right)        \left(XX^T -\UUt \UUt^T \right)  \right] }\\
	\le  &  \Vnorm{  V_X^T \left[  \left( \mathcal{A}^* \mathcal{A}  \right)        \left(XX^T -\UUt \UUt^T \right)  \right]    } \specnorm{\UUt}^2  \specnorm{  \left( \mathcal{A}^* \mathcal{A}  \right)        \left(XX^T -\UUt \UUt^T \right)      }  \\
	\lesssim  & \Vnorm{XX^T - \UUt \UUt^T } \specnorm{\UUt}^2  \specnorm{  \left( \mathcal{A}^* \mathcal{A}  \right)        \left(XX^T -\UUt \UUt^T \right)      }  \\
	\lesssim &   \Vnorm{XX^T - \UUt \UUt^T } \specnorm{\UUt}^2 \left( \specnorm{X}^2 + \specnorm{\UUt}^2  \right)  \\
	\lesssim &   \Vnorm{XX^T - \UUt\UUt^T } \specnorm{X}^4  \\
	\lesssim &  \left(  \Vnorm{ V_X^T \left( XX^T - \UUt \UUt^T \right)}    + \Vnorm{  \UUt \Wtperp\Wtperp^T \UUt^T  }\right)  \specnorm{X}^4,
	\end{align*}
	where in the third and fourth line we used the estimates from above.  The fifth line is due to assumption $\specnorm{\UUt } \le 3\specnorm{X} $. In the last line we used Lemma \ref{lemma:technicallemma8}. Hence, it follows that
	\begin{align*}
	& \mu^2 	\Vnorm{ V_X^T  \left[  \left( \mathcal{A}^* \mathcal{A}  \right)        \left(XX^T -\UUt \UUt^T\right)  \right]   \UUt \UUt^T  \left[  \left( \mathcal{A}^* \mathcal{A}  \right)        \left(XX^T -\UUt \UUt^T\right)  \right] }	\\
	\le  &  \frac{ \mu}{1000} \sigma_{\min}^2 \left(  X  \right)  \Vnorm{ V_X^T \left( XX^T - \UUt \UUt^T \right)} +  \frac{\mu }{400} \sigma_{\min}^2 \left(  X  \right)  \Vnorm{  \UUt \Wtperp\Wtperp^T \UUt^T  },
	\end{align*}
	where the last inequality is due to  due to the assumption $ \mu \le c \kappa^{-2} \specnorm{X}^{-2} $ for a sufficiently small constant $c>0$.\\
	
	\noindent \textbf{Combining the estimates: } By combining the estimates, it follows that
	\begin{align*}
	\Vnorm{V_{X}^T  \left( XX^T -\UUtplus\UUtplus^T \right) }  &\le  \left(  1 - \frac{\mu}{200}  \sigma_{\min}^2 \left(  X\right) \right)\Vnorm{  V_{X}^T    \left( XX^T -\UUt \UUt^T \right) }+ \mu  \frac{\sigma_{\min}^2 \left(X\right) }{100}  \Vnorm{  \UUt \Wtperp\Wtperp^T \UUt^T  },
	\end{align*}
	which finishes the proof.
\end{proof}